\documentclass[letterpaper]{article} 
\usepackage{aaai2026}  
\usepackage{times}  
\usepackage{helvet}  
\usepackage{courier}  
\usepackage[hyphens]{url}  
\usepackage{graphicx} 
\usepackage{natbib}  
\usepackage{caption} 
\usepackage{algorithm}
\usepackage{algorithmic}

\usepackage{newfloat}
\usepackage{listings}
\DeclareCaptionStyle{ruled}{labelfont=normalfont,labelsep=colon,strut=off} 
\floatstyle{ruled}
\newfloat{listing}{tb}{lst}{}
\floatname{listing}{Listing}
\usepackage{graphicx}
\usepackage{amsmath}
\usepackage{amsthm}
\usepackage{amsfonts, amssymb}

\usepackage{booktabs}
\usepackage{nicematrix} 

\usepackage{bm}
\usepackage{dsfont}
\usepackage{subfigure}
\usepackage{multirow}
\usepackage{enumitem}
\usepackage{caption}

\newcommand{\real}{I\!\!R} 

\def\beqa{\begin{eqnarray*}}
\def\eeqa{\end{eqnarray*}}
\def\be{\begin{eqnarray}}
\def\ee{\end{eqnarray}}

\def\Xcal{\mathcal{X}}
\def\Ycal{\mathcal{Y}}

\def\Hcal{\mathcal{H}}

\def\Esf{\mathsf{E}}
\def\Psf{\mathsf{P}}

\def\bb{\bm{b}}
\def\bn{\bm{n}}
\def\bx{\bm{x}}

\def\bmu{\boldsymbol{\mu}}

\def\ub{\bm{b}}
\def\ux{\bm{x}}

\def\nFP{n_{FP}}
\def\nFN{n_{FN}}
\def\mFP{m_{FP}}
\def\mFN{m_{FN}}
\def\qFP{q_{FP}}
\def\qFN{q_{FN}}

\newtheorem{definition}{Definition}
\newtheorem{theorem}{Theorem}

\newtheorem{lemma}{Lemma}

\newtheorem{example}{Example}

\title{Impacts of Individual Fairness on Group Fairness \\from the Perspective of Generalized Entropy}
\author {
    Youngmi Jin\textsuperscript{\rm 1},
    Jio Gim\textsuperscript{\rm 2},
    Tae-Jin Lee\textsuperscript{\rm 3},
    Young-Joo Suh\textsuperscript{\rm 2}
}
\affiliations {
    \textsuperscript{\rm 1}KAIST\\
    \textsuperscript{\rm 2}POSTECH\\
    \textsuperscript{\rm 3}Sungkyunkwan University\\
    youngmijin@kaist.ac.kr, \{jio.gim, yjsuh\}@postech.ac.kr, tjlee@skku.edu 
}

\begin{document}

\maketitle

\begin{abstract}
This paper investigates  how the degree of group fairness changes 
when   the degree of individual fairness is actively controlled.  
As a metric quantifying individual fairness, 
we consider generalized entropy (GE) recently introduced into machine learning community. 
To control the degree of  individual fairness, 
we design a classification algorithm satisfying a given degree of individual fairness 
through  an  empirical risk minimization (ERM)  
with a fairness constraint specified in terms of GE. 
We show the PAC learnability of the fair ERM problem by proving
that the true fairness degree does not deviate much from an empirical one with high probability 
for finite VC dimension if the sample size is big enough. 
Our experiments show that strengthening individual fairness degree does not always 
lead to enhancement of group fairness. 
\end{abstract}


\section{Introduction}
As the use of  machine learning algorithms grows in diverse areas such as criminal justice, lending and  hiring,
the issue of  algorithmic fairness gets  big attention. 
In response, a variety of work on algorithmic fairness has been proposed 
such as many fairness definitions 
\cite{Kusner17, Lowy22}, 
examination of relationship between different fairness concepts \cite{ Kleinberg17},
fairness metrics  quantifying the degree of fairness  \cite{Heidari18},
and finding a fair empirical risk minimizer  \cite{Donini18}. 

Diverse group and  individual fairness definitions   are 
proposed such as equal opportunity and equalized odds  in \cite{Hardt16},  
disparate treatment/impact  in \cite{Zafar17,Feldman15}, 
for group fairness concepts,  
and average individual fairness in \cite{Kearns19} for individual fairness, 
after Dwork et al. introduces the concepts of individual fairness 
and group fairness \cite{Dwork12}, 
where individual fairness requires that similar individuals should be similarly classified 
and group fairness requires some approximated statistical parity over the partitioned groups  
based on some attributes such as race or gender.

The work of \cite{Speicher18}  proposes to use generalized entropy (GE) indices 
which originally   evaluates  income inequality in economics 
as a metric measuring algorithmic fairness.
GE is an individual fairness as  a group fairness in that 
each individual is treated as one single group and 
has the  nice property that  GE can be decomposed 
into  two terms,  within-group term and between-group term where 
between-group term is a kind of group fairness at coarser level than individual.
Using the GE's nice property, 
they  demonstrate  the well-known insight provided in \cite{Dwork12}  
that reducing the degree of  group unfairness may increase that of individual unfairness,

Motivated by the work  of \cite{Speicher18}, 
we  investigate the under-explored question 
"How does the degree of group fairness (i.e., the between-group term) changes  
if we regulate  the degree of individual fairness (i.e., GE)?". 
The question demands us to  design a classifier 
achieving a given  degree of individual fairness specified in terms of GE 
in order to control the degree of of individual fairness.     
The task to design a classifier satisfying a given fairness constraint 
directly related to the  roles of GE as fairness metric. 
In general, there are two roles for an  algorithmic fairness metric. The first role 
is to measure the degree of fairness of a classifier. 
The second one is to  specify   fairness requirements  of a  classification problem 
so that the requirements can be realized with small deviation by an algorithm. 
We focus on the second role of a fairness metric for GE, while Speicher et al. are  interested in  the first role. 

Before the  design of a classifier satisfying fairness requirements, 
we analyze the deviation of the empirical fairness degree measured 
by GE on some sample data set from the true fairness degree on the original space 
where the sample data set has been drawn.
To design a classifier, we consider a fair empirical risk minimization (ERM) 
where the fairness constraint  is specified 
in terms of GE and find an approximate optimal empirical classifier 
satisfying the given fairness constraint  based on Hedge algorithm \citep{Freund96, Freund97}. 
With  the approximate optimal classifier achieving 
given degree of GE (individual  fairness), 
we experimentally study the impact of controlling  GE (individual  fairness) 
on between-group term (on group  fairness).

The most related works are the papers of \cite{Speicher18}, \cite{Kearns18}, \cite{Agarwal18}, and \cite{Cousins21}. 
The major differences  of  the paper of \cite{Speicher18} and ours 
are as follows. 
The work \cite{Speicher18} focuses on how to evaluate the degree of fairness of a given 
algorithm and examines GE as a metric to quantify fairness degree. 
They study interesting properties of GE as an algorithmic fairness metric. 
One of them is  that strengthening group fairness may result in degradation of individual fairness. 
Unlike the paper, this paper  is interested in  
whether the degree of group fairness (between-group  term) gets improved  
when we impose a strong  requirement on the degree of  individual fairness (GE), 
which drives us to design and analyze a fair algorithm 
whose fairness requirement is  given by GE.  
In  this paper, the role of generalized entropy is not limited 
to a  metric quantifying fairness degree 
but extends to a design parameter in pursuit of a fair algorithm.

Our algorithm-designing philosophy is similar to that 
of \cite{Kearns18} and \cite{Agarwal18}  in  that all of them  seek a randomized 
algorithm based on Hedge algorithm and analyze the performance of the  randomized algorithm 
with the help of minmax game theory.  
However, the objectives of the papers are very different from ours. 
The objective of    \cite{Agarwal18}  is  
to provide the unified reduction approach for fair classification 
where the fairness concepts are represented by a linear function 
of numbers of false positive and false negative labels, 
including   demographic parity  and equalized  odds  fairness concepts. 
Our work is not related with such reduction approach.  
The objective of \cite{Kearns18}    is how to 
prevent fairness gerrymandering, 
the situation where a classifier 
satisfies some  fairness constraint on small number of pre-defined groups 
but it severely violates the fairness constraint on groups. 
Our work does not focus on  preventing fairness gerrymandering. 
 
Our paper  and the paper of \cite{Cousins21}  study  PAC learnable fair learning algorithms. 
The fairness measure in \cite{Cousins21} is based on malfare (opposite to welfare), 
disadvantage or loss caused by wrong prediction,  and our fairness measure 
is based on generalized entropy. 

We  summarize our contributions: 
(i) 
It is well known that improvement of group fairness may deteriorate individual fairness. 
However, it is little known how group fairness changes as individual fairness enhances.
Our experiments show that strengthening individual fairness degree 
does not always lead to enhancement of group fairness,
which has not been discussed in existing literature as far as we know. 
(ii) 
We formulate and analyze  theoretically and experimentally  a (randomized)  ERM  
with a fairness constraint given by GE. 
(iii) We show the PAC learnability of the fair ERM  problem 
by proving that the true fairness degree does not deviate much 
from an empirical one with high probability  
if the sample size is big enough. 

\section{Problem Formulation} \label{sec:formulation}
The most popular group fairness definitions are demographic parity (DP), equal opportunity, and 
equalized odds (EO)  
\cite{Feldman15,Hardt16}.  
The main paradigm of such fairness definitions is 
the (conditional) independence between prediction results of a hypothesis 
and  sensitive attributes (given ground truth values). 
We believe that  the traditional fairness definitions such as DP and EO do not 
consider the  impact of wrong predictions on individuals. 
Consider the following example. 
\begin{example} Imagine a bank's lending system where positive prediction means acceptance of a loan application. 
A hypothesis  $h$ makes   decisions for a population consisting of  females and  males. 
Each of  the female and male groups has four individuals with  
one true positive (TP), one false negative (FN), 
one true negative (TN), and one false positive (FP)  prediction results. 
Note that $h$ satisfies DP and EO, since the statistical properties of female and male groups 
are identical. However, the individuals (especially  the persons with FN)  do  not think 
the hypothesis makes fair decisions, 
simply because the prediction results bring different benefits to individuals; 
they are unfavorable to the individuals with FN but  
advantageous to the ones with FP.  
\end{example}
The toy example shows that it is necessary to consider the benefits 
brought by prediction results to individuals.  
In his paper, we   considers a new fairness metric, generalized entropy introduced to machine learning community by \cite{Speicher18},  
that focuses on inequality of individual's benefits resulting from prediction results 
in order to remedy  what existing fairness definitions do not pay much attention to.

In this paper,  vectors are denoted in boldface font and scalars  in  normal font; 
$\ub$ is a vector while  $b$ is a real value.

\subsection{Definition of Generalized Entropy Index}
Generalized entropy index  is originally developed 
to  measure income inequality over finite population by \cite{Shorrocks80}. 
For a given population of $n$ individuals with income vector $\bb=(b_1,\cdots, b_n)$ 
with $b_i \geq 0$ for all $i$, 
Shorrocks has considered an  inequality measure, $I(\bb;n)$,  
that satisfies the  following axioms.
\begin{itemize}
\item{Axiom 1:} $I(\bb;n)$ is continuous and symmetric in $\bb$, i.e.,
$I(\bb;n) = I(\bb';n)$ where $\bb'$ is a permutation of $\bb$. 
\item{Axiom 2 :} $I(\bb;n) \geq 0$ and  $I(\bb;n)=0$ if and only if $b_i = c$ for a constant $c$ and for all $i$.
\item{Axiom 3:} 
 $\frac{\partial^2 I(\bb;n)}{\partial b_i \partial b_j}$ is continuous for all $i,j$.
\item{Axiom 4 (Additive decomposability):} For $n\geq 2$ and 
any partition\footnote{For a given set $X$, a class of subsets $\{X^i\}_{i=1}^G$ with $X^i \subseteq X$ 
is called a partition of $X$ 
if and only if $\cup_i X^i  = X$ and $X^i \cap X^j =\emptyset$ if $i \neq j.$}  
$\bb^1, \bb^2, ...\bb^G$ of $\bb$ with $\bb^g =(b_1^g, ..., b_{n_g}^g)$, 
  there exists a set of $w_g^G(\bmu, \bn)$ such that 
$I(\bb^1, ..., \bb^G;n)= \sum_g w_g^G(\bmu, \bn)I(\bb^g;n_g) + V$  
where  $\bmu = (\mu_1,..., \mu_G)$ with  $\mu_g =\frac{1}{n_g} \sum_{i=1}^{n_g} b_i^g$, 
$~\bn = (n_1, \cdots, n_G)$ with $n_g$ the cardinality of $\bb^g$,  and
$V=  I(\mu_1   \mathbf{1}_{n_1},\cdots,\mu_G \mathbf{1}_G;n)$
with  $\mathbf{1}_{m} = (1,\ldots, 1) \in \mathbb{R}^m$.   

\item{Axiom 5: }$I(\bb, \bb,...,\bb;kn) = I(\bb;n)$ for any positive integer $k$.

\item{Axiom 6 (Pigou-Dalton principle of transfers): } 
If a transfer $\delta >0$ is  made from an individual $j$ to another $i$ with $b_j > b_i$ 
such that $ b_j -\delta > b_i + \delta$, then the inequality index decreases after the transfer. 

\item{Axiom 7: } $I(r\bb;n) = I(\bb;n)$ for any  $r >0$. 
\end{itemize} 
\begin{definition}  \label{def:gen_entropy}
For any $\alpha \in [0, \infty)$ and  a given   vector $\bb=(b_1,\ldots,b_n)$ with $b_i \geq 0$ for all $i$, 
the generalized entropy (GE) index of $\bb$, denoted by $I_{\alpha}(\bb;n)$,  
 is  defined as   
$ I_{\alpha}(\ub;n)= \frac{1}{n} \sum_{i=1}^n f_{\alpha}\Big(\frac{b_i}{\mu}\Big)$ 
where   $\mu  = \frac{1}{n} \sum_{i=1}^n b_i$ and    
\be
 f_\alpha(x) = \left \{ \begin{array}{ll} 
  - \ln x    & \mbox{if } \alpha =0,\\    
      x \ln x  & \mbox{if } \alpha =1,\\
        \frac{1}{\alpha (\alpha-1)}  (x^{\alpha} -1) & \mbox{otherwise.} 
      \end{array} \right.  \label{eqn:f_alpha}
\ee
\end{definition}
An income vector $\ub$ is less unfair (i.e. more fair) than $\tilde{\ub}$ if $I_{\alpha}(\ub;n) < I_{\alpha}(\tilde{\ub};n)$.
Note that  
if  $b_i \neq b_j$ for some $i \neq j$, 
then  $I_{\alpha}(\bb;n) \neq 0$  by Axiom 2 and that  
 $I_{\alpha}(\bb;n) = 0$ only when all $b_i$s are equal. 
 
\subsection{Applying GE  to  Algorithmic Fairness}
Consider a supervised machine learning problem. 
Each individual is represented by $(\bx, y) \in \Xcal \times \Ycal$ 
where $\bx$ is a feature vector 
and $y \in \Ycal$ is a (ground truth) label of $\bx$. 
We assume that $\Ycal = \{0,1\}$
and there is an unknown distribution $P$ over $\Xcal \times  \Ycal$. 
The label value 1 corresponds to the desirable case for an individual   
and the value 0 to the undesirable one. 
For a credit lending example, acceptance of  a loan application corresponds to  1 and rejection of it to 0. 
The marginal distribution of $P$ over $\Xcal$ is denoted by $P_x$.   
A sample data set (or training data) with size $n$,   $S = \{(\bx_i,y_i)\}_{i=1}^n$, consists 
of elements $(\ux_i, y_i)$    independently and identically drawn (i.i.d) according to 
the unknown distribution $P$  over $\Xcal \times \Ycal$.  
A hypothesis, called also a learning algorithm, is a function
 $h: \Xcal \rightarrow \Ycal$  that outputs a predicted label $h(\bx)$, 
 either correct or incorrect,  for $\bx \in \Xcal$. 
The empirical risk (or error) of hypothesis $h$ is the error
 that $h$ incurs on the sample data $X =\{\bx_i\}_{i=1}^{n}$:  
$R_S(h)=   \frac{1}{n} \sum_{i=1}^n   \mathds{1}\{ h(\bx_i) \neq y_i \}$ 
where  $\mathds{1}\{C\}$ is the indicator function that returns 1 if condition $C$ is satisfied and returns 0 otherwise.
The \emph{true  error} of a hypothesis $h$ is the error 
that $h$ generates over the whole domain $\Xcal$:  
$~R_{\Xcal}(h) = \Psf_{(\bx,y) \sim P } [h(\bx) \neq y] = P(\{(\ux, y) | h(\bx) \neq y \}).$

We assume that $\Xcal$ and a  class of hypothesis $\mathcal{H}$ are given.
The objective of   learning  is to find a 
hypothesis $h \in \Hcal$ for a given hypothesis class  $\Hcal$ 
that  predicts well the label of a new instance $\bx$ 
(i.e., yields small $R_{\Xcal}(h)$) 
with the help of a sample data set, 
which is usually  done by  ERM: we find $h^*  = \arg \min_{h' \in \Hcal} R_S(h')$
and  expect that  $R_{\Xcal}(h^*)$ is small too. 
It is very well-known that  the true error of $h$ is close to the empirical error 
with high probability if  the sample data size is large enough and VC dimension is finite, 
which is stated in   Theorem \ref{thm:gen-error} (whose proof can be found in  \cite{Mohri2018}.)  
\begin{theorem} ({Standard VC Dimension Bound}) \label{thm:gen-error}
For any distribution $P$ over $\Xcal \times \Ycal$, let $S =\{(\bx_i, y_i)\}_{i=1}^n$ 
be a  sample data set  i.i.d  according to $P$. 
For any   $0 < \delta < \frac{1}{2}$ and  any  $h \in \mathcal{H}$,   with probability at least $1-\delta$, 
it holds that  
\beqa
 R_{\Xcal}(h) ~ \leq ~  R_S(h)+
    \sqrt{ \frac{   8d_{\Hcal}\ln \big(\frac{2en}{d_{\Hcal}}\big) +8\ln \frac{4}{\delta}}{n}} 
\eeqa
where $d_{\Hcal}$ is the VC dimension of $\Hcal$.   
\end{theorem}

The work of \cite{Speicher18}  proposes to use $I_{\alpha}(\bb;n)$ for some $\alpha >0, ~\alpha  \neq 1$ 
as a metric assessing  the fairness degree  of $h$  
by converting prediction results of $h$ into benefit values $b_i$ as follows,   
$b_i = h(\bx_i) - y_i + 1.$   
The last term, adding one, makes $b_i$  non-negative 
so that $I_{\alpha}(\bb;n)$ can be defined 
for finite $n$ and $\alpha >0, \alpha \neq 1$. 
The  philosophy behind the definition of $b_i$ can be understood 
by  an  example of  a bank's lending system  
where the label of 1 
corresponds to the acceptance of a loan application  
and the label 0 the rejection of it. 
Each loan applicant is either creditworthy and can pay back the loan, denoted by the label 1, 
or   not creditworthy and will default, denoted by the  label 0. 
For an applicant with true label $y_i =1$, she would think the decision is unfair if $h(\ux_i)=0$.  
For an  applicant with her true label  $y_j =0$, she would get  more  benefit than she deserves 
if $h(\ux_j) = 1$; others would think it  unfair.  

Similarly  as  \cite{Speicher18} does, we define $b_h(\bx)$   as   
\be 
b_h(\bx) =  a(h(\bx)- y) + c \label{eqn:ours}
\ee     
with $c > a >0$ and $c-a \geq 1$, for $h \in \Hcal$.  
By the   definition of $b_h(\bx)$ in (\ref{eqn:ours}), 
the benefit of an individual is $c$ for correct prediction, 
$ c + a $ for false positive (FP) prediction, and $c-a$ for false negative (FN) prediction. 
Note that FP  labeling is a favorable error  to an individual
and FN labeling a harmful error  to an individual. 
We will drop the subscript $h$ if $h$ is clear in context.

For a given  hypothesis $h \in \Hcal$ and $\alpha \in [0, \infty)$,  
we can measure its \emph{empirical algorithmic  unfairness}  
from the sample data set $\{\ux_i, y_i\}_{i=1}^n$ 
by  $I_{\alpha}(\ub_h;n)$  as in Definition \ref{def:gen_entropy} 
where $\ub_h = (b_h(\ux_1), \cdots, b_h(\ux_n))$ 
(Appendix \ref{appendix:add-decom} 
provides a  computation example of $I_{\alpha}$ for a given $\{b_h(\ux_i)\}_{i=1}^n$.)
\footnote{All appendices can be found in the supplementary document.} 
 
 \subsection{GE as an Individual Fairness}

We assume that the whole population is partitioned to $G$ groups. 
Partitioning of a population into several groups  is made by using features; for example, 
using the gender feature, we can partition the whole population into two groups,  
a group of females  and a group of males, 
if gender has only two components, male and female.

One of the most prominent properties of GE is additive decomposability of Axiom 4; 
\be
 I(\bb^1, ..., \bb^G;n)= \sum_{g=1}^G w_g^G(\bmu, \bn)I(\bb^g;n_g) + V \label{eqn:add_decomp}
\ee
where  $V = I(\mu_1 \mathbf{1}_{n_1},\cdots,\mu_G \mathbf{1}_{n_G};n) 
= \sum_{g=1}^G \frac{n_g}{n} f_{\alpha}\Big( \frac{\mu_g}{\mu}\Big)$ 
and   $w_g^G(\bmu, \bn) = \frac{n_g }{n} \Big(\frac{\mu_g}{\mu}\Big)^{\alpha}$. 
The first term in (\ref{eqn:add_decomp}),  $\sum_g w_g^G(\bmu, \bn)I(\ub^g;n_g)$, 
called  by    \emph{within-group term},  
is the weighted sum of inequality over the groups  $\bb^g$s. 
The second term, $V$, called  by \emph{between-group term},  
is the inequality of the population with size $n$  
where each group consists of $n_g$ members and every member of group $g$ 
has  equal benefit $\mu_g$, 
which implies that in each group, perfect equality is achieved 
(refer to Appendix \ref{appendix:add-decom} for the computation of additive decomposability.) 

Many group fairness definitions  partition 
the whole population into several groups and compare 
statistical measures over the groups. 
Typical examples are DP, equal opportunity, 
and EO \citep{Feldman15, Hardt16}.  
Such group fairness definitions implicitly assume 
that the individuals in a same group are treated equally.    
From this perspective, between-group term  $V$ in (\ref{eqn:add_decomp}) can be 
regarded as a metric quantifying the degree of   group fairness.

Recall that  $I_{\alpha}(\bb;n) = 0$ only when all $b_i$s are equal. 
Non-zero value of generalized entropy ensures 
the existence of (at least two) individuals whose classification result 
(predicted value - ground truth value) 
is different from the other ones.

Consider a special case where each group, $g$,  consists of only one single  element (or individual).   
Hence $I(\bb^g;n_g)=I(b_i;1)=0$, which results in  
$\sum_g w_g^G(\bmu, \bn)I(\bb^g;n_g) = 0$, that is,  
within-group term becomes $0$. 
Therefore, we have  $ I(\bb^1, ..., \bb^G;n)= V$, 
i.e., GE gets equal to between-group term $V$, 
which implies that GE is an extreme case of  between-group term (group  
fairness) when each individual is a group, the finest level of groups. 
From this perspective, we regard GE as individual fairness as   group fairness 
at the finest level.  
It  is a  natural view point that   individual  fairness is 
a special extreme case of group fairness. 
Kearns et al.  propose in \cite{Kearns19} the notion of 
``average" individual fairness 
that seeks  equal  averaged  error rates 
over individuals, when there are sufficiently many classification 
tasks so that each individual takes an averaged  error rate over multiple classification tasks. 
The average individual fairness of \cite{Kearns19} can be also  
regarded as an extreme case of  group fairness 
if  we regard each individual as a group. 
In the papers of \cite{Kearns18, Kearns18a}, they  propose a notion of 
rich group fairness to bridge the gap between group fairness 
and individual fairness by considering very large number of 
groups through the combination of setting feature values 
rather than a small number of groups. 
In the context  of this rich group fairness, the finest level is 
the special case that each individual is a group with size one. 

The view point (that individual  fairness is a special extreme case of group 
fairness) is different from the   notion of individual fairness  of  \cite{Dwork12}   
that similar individuals should be treated similarly which is theoretically attractive but 
requires  in practice a daunting task to find a metric quantifying similarity 
between individuals in the feature space. 
When adopting the additive decomposability of GE and the view point 
that individual  fairness is a special extreme case of group fairness, 
then using GE, we  can investigate how group fairness,$V$,  
is affected by the control of individual fairness, $I_{\alpha}$.

\subsection{Problem Formulation}
We consider   an ERM with  a fairness constraint  specified by  GE 
for $\alpha \in [0, \infty)$, 
which we call a fair empirical risk minimization with GE (FERM-GE); 
\be \label{eqn:FERM}
 \mbox{FERM-GE:} ~~ \min_{h \in \Hcal}  R_S(h)  ~~\mbox{s. t. } I_{\alpha}(\ub_h;n) \leq \gamma. \ee 
Let $h^*$ be the optimal solution of FERM-GE in (\ref{eqn:FERM}). 
We investigate if the   true error $R_{\Xcal}$ and  the fairness degree  of $h^*$  
over   $\Xcal$ (where $\ux_i$ belongs to) do not deviate much from  
the empirical ones over a sample data set with high probability 
when the training set is sufficiently large. 
We  find  $h^*$ of FERM-GE  for a given $\gamma$ 
and   examine the value of between-term $V$ of $h^*$ on the sample space.   
By changing the values of $\gamma$, we  control the degree of individual fairness degree, GE, 
and find a set of classifiers satisfying the various degree of GE. 
By examining the corresponding values of  between-term $V$ of the set of classifiers, 
we investigate   how  between term $V$ changes when we control  $\gamma$ constraint on GE . 
 
\section{Deviation Bounds of Empirical Fairness}  \label{sec:problem_form}
This section considers the PAC learnability of our fair ERM. 
That is we need a   similar result on the degree of fairness to Theorem \ref{thm:gen-error}:  
 with high probability, the degree of fairness of a hypothesis on a given sample 
data set does not deviate much from that on the original space 
from which the sample data set is drawn, 
%

For this, we extend the original definition of GE defined on a finite population $\{\ux_i\}_{i=1}^n$ so that
GE  can work on  $\Xcal$, an arbitrary  space. 
The parametric family of GE in Definition  \ref{def:gen_entropy}   
is originally defined over a finite population 
under the premise that 
$n$ individuals are separately identified 
and each has the same weight $\frac{1}{n}$. 
Hence Definition \ref{def:gen_entropy} is easily applied to  the sample data set 
$S= \{\bx_i, y_i \}_{i=1}^n$ with finite size   
but  does not work on some space $\Xcal$,   like $ \mathbb{R}^m$, 
that has uncountably many elements. 
We extend  the generalized entropy $I_{\alpha}$ so that the extended one can work even  
on a space with uncountably many elements while still satisfying all of the axioms, 
especially  additive decomposability property, after the extension. 
Let $\real^+$ be the set of non-negative real numbers. 
 
\begin{definition}(Extension of GE) 
\label{def:ext_gen_entropy} 
Let $b: \Xcal \rightarrow \real^+$ and  $P_x$ be a probability distribution on $\Xcal$. 
For a constant $\alpha \in [0, \infty)$, 
 GE  of $b(\Xcal)$
with respect to $P_x$ 
is  defined by
\be \label{eqn:ext_gen_entropy}
  I_{\alpha}(b,\Xcal, P_x) = {\displaystyle  \int_{\Xcal}    f_{\alpha}\Big(\frac{b(\ux)}{\Esf[b(\ux)]}\Big) dP_x}  
\ee
where $dP_x$ is the probability density function of $P_x$
and $\Esf[b(\ux)] =    \int_{\Xcal} b(\bx) dP_x $. 
\end{definition} 

After slight modification of Axioms 2, 4, and 6, 
generalized entropy can be  extended   with a mild condition on $b(\ux)$ 
as in Theorem \ref{thm:prop_ext_gen_entropy}. 
(The 
proof of Theorem \ref{thm:prop_ext_gen_entropy} 
can be found in Appendix \ref{append:ext_gen_entropy}. All of 
the proofs in the paper can be found in Appendices.)
 \begin{theorem} \label{thm:prop_ext_gen_entropy}
If $b(\ux)$ is bounded on $\Xcal$, then
the extension of generalized entropy satisfies all (modified) Axioms 1-7.
\end{theorem}  
Obviously $b_h(\ux)$  is bounded over $\Xcal$, 
since    $b_h(\ux) = a(h(\ux)-y)+c$ with $c > a > 0$  and $c-a \geq 1$. 
Hence the extended generalized entropy $I_{\alpha}(b_h, \Xcal, P_x)$ 
meets all of the Axioms 1-7, including the property of additive decomposability,  
for any $\alpha \in [0, \infty)$.  

Recall  that for any distribution $P$ defined on $\Xcal \times \Ycal$, 
$P_x$ denotes the marginal distribution of $P$ over $\Xcal$.  
To emphasize a hypothesis $h$ and a probability distribution $P$ on $\Xcal \times \Ycal$, 
we will use $I_{\alpha}(h, P)$ instead of $I_{\alpha}(b_h, \Xcal, P_x)$ from now on,  
even though the generalized entropy definition needs $P_x$ not $P$.   
For the sample data set $S$, we still use $I_{\alpha}(\ub_h;n)$. 
Hence for a given $h$, $I_{\alpha}(\ub_h;n)$ denotes the empirical fairness of $h$ over $S$ 
and  $I_{\alpha}(h,\Xcal)$ the true fairness of $h$ over \emph{ the whole domain} $\Xcal$.
Theorem \ref{thm:gen-fair-constraint} provides the deviation bounds of the empirical fairness 
from the true fairness (the proof of Theorem \ref{thm:gen-fair-constraint}  can be found in \ref{appendix:proof-gen-fair-constraint}.)

\begin{theorem}\label{thm:gen-fair-constraint} 
For any distribution $P$ over $\Xcal \times \Ycal$, let $S =\{(\bx_i, y_i)\}_{i=1}^n$ 
be a sample data set  i.i.d according to $P$. 
Let $\alpha \geq 0$ and   $r =\frac{c}{a}$. 
For any $0 < \delta < 1$ ,
with probability  at least $1-\delta$, for each $h \in \mathcal{H}$ and  $ \alpha  \in [0, \infty)$,  
it holds that 
\beqa 
\Big |I_{\alpha}(h, P) - I_{\alpha}(\ub_h;n)  \Big|  & \leq &  \psi_{\alpha}(a,r) \sqrt{\frac{1}{2n}\ln \frac{4}{\delta}}
\eeqa
where $\psi_{\alpha}(a,r)$ is defined as follows 
\beqa
\psi_{\alpha}(a,r) = \left\{ \begin{array}{ll }
       \frac{2}{r-1} + \ln\big(\frac{r+1}{r-1}\big) &\mbox{if }    \alpha =0,\\
       \frac{2+4\ln(ar+a)}{r-1}+\ln \big(\frac{r+1}{r-1}\big)  &\mbox{if }    \alpha =1,\\
       \frac{\big(1+\frac{2\alpha}{r-1}\big)\big(\frac{r+1}{r-1}\big)^{\alpha}-1}{|\alpha(\alpha-1)|}&\mbox{otherwise.}
        \end{array} \right.
\eeqa
\end{theorem}

For fixed $\alpha$, note that 
$\psi_{\alpha}(a, r)\sqrt{\frac{1}{2n}\ln \frac{4}{\delta}}$ 
gets close to 0 as $n$ goes $\infty$.
From Theorem \ref{thm:gen-fair-constraint},    
large $\alpha$ may have large deviation of   empirical fairness degree from  true one, 
which implies that  small $\alpha$ is preferred in practice.  
From the fact that $\psi_{\alpha}(a, r)$ is  decreasing with $r$ for fixed $a$ 
(which can be easily checked), 
we learn that   large  $r$    is preferred 
for small deviation of empirical fairness from true one. 
However, if  $r$ is too big, then it yields a very small value of $I_{\alpha}$, 
which may cause difficulty in discerning the existence of unfairness.
 
The deviation bound  $\psi_{\alpha}$  of Theorem \ref{thm:gen-fair-constraint}
is independent of classifier's accuracy. 
In most cases,  we are not interested   in inaccurate hypotheses but  in accurate ones. 
Since $I(\ub_h;n)=0$, if $h$ has no error,
accurate hypotheses may have smaller deviation than inaccurate ones. 
Indeed, we can find   deviation bound  $\tilde{\psi}_{\alpha} $ 
that   depends on empirical error $R_S(h)$ 
such that hypothesis $h$ with small   $R_S(h)$ has 
small deviation bound of empirical fairness  from true one.   
Theorem \ref{thm:gen-fair-constraint-tight} in Appendix  \ref{appendix:proof-simple} tells us 
that hypothesis $h$ with small empirical error $R_S(h)$ has small deviation bound 
of empirical fairness from true one.
Theorem \ref{thm:simple-gen-fair-constraint-tight} is a simplifed version of it 
and provides $\psi_{\alpha}$   that   depend  on empirical error $R_S(h)$, 
for   $\alpha=0, ~1,$ and 2. 
(refer to Appendix \ref{appendix:proof-simple} for the proof.  )

\begin{theorem} \label{thm:simple-gen-fair-constraint-tight}
For any distribution $P$ over $\Xcal \times \Ycal$, let $S =\{(\bx_i, y_i)\}_{i=1}^n$ 
be a sample data set i.i.d  according to 
$P$. 
For any $0 < \delta < 1$,
with probability  at least $1-\delta$,    
if $R_S(h) < \frac{1}{2}$ and $n$ is sufficiently large so that $\varepsilon_2< \frac{1}{5}$,  
then  for  each   $h \in \mathcal{H}$ and $\alpha \in \{0,1,2\}$  
it holds that
\beqa 
\Big |I_{\alpha}(h, P) - I_{\alpha}(\ub_h;n)  \Big|  & \leq &  \tilde{\psi}_{\alpha}(r, \varepsilon_2) \cdot \varepsilon_2
\eeqa
where $\varepsilon_2 =  \sqrt{ \frac{   8d_{\Hcal}\ln \big(\frac{2en}{d_{\Hcal}}\big) +8\ln \frac{8}{\delta}}{n}}$,  
\beqa
\tilde{\psi}_{\alpha}  =
   \left \{ \begin{array}{l}
     \frac{1}{ r-R_S(h)-  \varepsilon_2} + \ln\Big( \frac{r}{r-1}\Big)  \qquad \qquad ~\mbox{for $\alpha=0$}, \\
     \frac{1}{r-R_S(h)-\varepsilon_2}\Big[1+  \frac{ r\big(1+  2\ln(ar+a)\big)}{r-R_S(h)}\Big] ~\mbox{for $\alpha=1$} ,\\
     \frac{1}{(\alpha-1)}\Big( \frac{r}{r-R_S(h)}\Big)^{2} V_2
            \qquad \qquad \quad ~~ \mbox{for } \alpha =2,
\end{array} \right.     
\eeqa
and $V_2=  \frac{1}{r} +\frac{3}{r^2}  + \frac{1}{r-R_S(h)}(12+\frac{3R_S(h)}{2r})$.
\end{theorem}

\section{FERM-GE} \label{sec:FERM-GE}

\begin{algorithm}[t!]
\caption{Hedge algorithm  for  randomized FERM-GE} \label{alg:fair}
\begin{algorithmic}
   \STATE {\bfseries Input:} A sample data set with size $n$,  
          $\lambda_{\max}$, $\nu$,  and
          an oracle finding $\arg \min_{h\in \Hcal} L(h, \lambda)$.
   \STATE Initialize $w_0^{(0)}=w_1^{(0)} =1$.  
   \STATE Set  $\lambda_0 ~=~ 0, ~~ \lambda_1  ~=~ \lambda_{\max},
          ~~ T ~=~ \frac{4 A_{\alpha}^2 \ln 2}{\nu^2},  
          ~~ \kappa ~=~ \frac{\nu}{2A_{\alpha}}. $
   \FOR{$t=1$ {\bfseries to} $T$}
   \STATE 1. Nature chooses   $\hat{\lambda}^{(t)}$:  
     \beqa
        \hat{\lambda}^{(t)} = \left \{ \begin{array} {ll} 
                               \lambda_0  & \mbox{with prob. } \frac{w_0^{(t-1)} }{ w_0^{(t-1)} + w_1^{(t-1)}},  \\
                               \lambda_1   & \mbox{with prob.} \frac{w_1^{(t-1)} }{ w_0^{(t-1)} + w_1^{(t-1)}}. \\
                              \end{array} \right.
     \eeqa   
  where $\lambda_0=0$ and $\lambda_1 =\lambda_{\max}$.  
   \STATE  2. The learner chooses a hypothesis 
   \beqa 
   \hat{h}^{(t)} = \arg \min_{h \in \Hcal}~L(h, \hat{\lambda}^{(t)})
   \eeqa 
   \STATE  3. Nature updates the weight vector:  for  $j \in \{0,1\}$, 
         \beqa
         w_j^{(t)} =  w_j^{(t-1)}\cdot (1+\kappa)^{l(\hat{h}^{(t)}, \lambda_j)}
         \eeqa 
         where $l(h,\lambda) =  \frac{L(h, \lambda)+ B}{A_{\alpha}}$.

   \ENDFOR
   \STATE {\bfseries Output:} $\bar{D} = \frac{1}{T} \sum_{t=1}^{T} \hat{h}^{(t)}$ and 
 $\bar{\lambda} = \frac{1}{T} \sum_{t=1}^{T} \hat{\lambda}^{(t)}$
\end{algorithmic}
\end{algorithm}

This section considers an ERM with a fairness constraint specified by $I_{\alpha}$ 
 to find an optimal (\emph{randomized}) hypothesis  among a given $\Hcal$. 
We consider randomized  hypotheses for good accuracy-fairness tradeoff. 
A randomized hypothesis $D$ is a probability distribution on $\Hcal$, 
that is $D = \sum_{h \in \Hcal} D_h h$ with $\sum_{h\in\Hcal} D_h =1$. 
The sample error of $D$ is given by
$R_S(D) = \sum_{h \in \Hcal} D_h R_S(h)$ 
and the corresponding generalized entropy is  
$I_{\alpha}(D;n) = \sum_{h \in \Hcal} D_h I_{\alpha}(\ub_h;n)$ 
for some $\alpha \in[0, \infty)$. 
Let $\Delta \Hcal$ be  the set of all probability distributions over $\Hcal$. 
For given $\alpha$ and $\gamma$, we consider the following (randomized) FERM-GE problem:   
\be
      \min\limits_{D \in \Delta \mathcal{H}} ~R_S(D)  
    ~~\mbox{s. t. }   I_{\alpha}(D;n)   \leq \gamma. \label{eqn:randomized_ferm}
\ee  
The above problem (\ref{eqn:randomized_ferm}) is a linear optimization with linear objective function and linear constraints. 
We want to find an approximated optimal solution of (\ref{eqn:randomized_ferm}) whose error and fairness 
degree are sufficiently close to the optimal one. 
This can be done by considering Lagrangian of  (\ref{eqn:randomized_ferm}), 
$L(D, \lambda)= R_S(D) + \lambda (I_{\alpha}(D;n)- \gamma)$ with $\lambda \in \real^+$.  
With the assumption $\gamma >\inf_{h \in \Hcal} I_{\alpha}(\ub_h;n)$,  
FERM-GE  is a feasible linear optimization, 
which guarantees the strong duality   for $\Hcal$ with finite VC dimension \cite{Boyd04} 
\footnote{Even though $\Hcal$ has infinitely many hypotheses, since   
the hypotheses in $\Hcal$ are applied to the finite sample space $S$, 
the number of   different labellings of $h \in \Hcal$ 
is finite, i.e., cardinality of $\{ h(S)~|~ h \in \Hcal\}$ is finite by Sauer's Lemma \cite{Mohri2018} 
(i.e., count only once if $h_1(S) = h_2(S)$ for $h_1 \neq h_2$)
when VC dimension of  $\Hcal$ is finite.}  
;
\be 
L^* =  \max_{\lambda \in \real^+} ~\min_{D \in \Delta{\Hcal}} L(D, \lambda) 
=   \min_{D \in \Delta{\Hcal}} ~ \max_{\lambda \in \real^+} L(D, \lambda).  \label{eqn:minmax}
\ee  
For the practical issue of  convergence,  we put a bound for 
the  dual variable,  
$ \lambda \in \Lambda = [0, \lambda_{\max}]$. 
After putting the bound for  $\lambda$,  the duality gap is still zero 
by the compactness and convexity of $\Lambda$;   
\be 
L^*_{\Lambda} =  \max_{\lambda \in \Lambda} ~  \min_{D \in \Delta {\Hcal }} L(D, \lambda) 
= \min_{D \in \Delta {\Hcal}} ~ \max_{\lambda \in \Lambda} L(D, \lambda).\label{eqn:bounded-minmax}
\ee 
The  optimal solution $(D^*_{\Lambda}, \lambda^*_{\Lambda})$ of (\ref{eqn:bounded-minmax}), 
can be  found 
as the equilibrium of a  repeated zero sum game  of two players, 
the learner  seeking  $D$ that minimizes  $ L(D, \lambda)$ and 
Nature  seeking  $\lambda$  that maximizes  $ L(D, \lambda)$ \cite{Boyd04}. 

The seminal paper  of \cite{Freund96} propose 
how  to find an approximated solution of (\ref{eqn:bounded-minmax}), denoted by $(\bar{D}, \bar{\lambda})$, 
using Hedge algorithm.  
The nice property of $(\bar{D}, \bar{\lambda})$ is 
that   each of its error and fairness degree is close to 
that of the unconstrained optimal solution of (\ref{eqn:minmax}), respectively as in Theorem \ref{thm:compare-opt}  
(Appendix \ref{append:learning-alg} provides detail explanations on finding an approximated optimal solution as well as the proof of Theorem  \ref{thm:compare-opt}.)
\begin{theorem}  \label{thm:compare-opt} 
Suppose that  $\gamma > \inf_{h \in \Hcal}  I_{\alpha}(\ub_h;n)$. 
For any given $\nu$,  after $T=\frac{4 A_{\alpha}^2 \ln 2}{\nu^2}$ iterations,
the randomized hypothesis $\bar{D}$   satisfies  
$~~R_S(\bar{D}) ~ \leq ~  L^* + 2 \nu$ and $I_{\alpha}(\bar{D};n) ~ \leq ~  \gamma + \frac{1+2\nu}{\lambda_{\max}}. 
$
where $L^*$ is the optimal value of (\ref{eqn:randomized_ferm}) and $A_{\alpha}$ is a constant depending on $r, ~\alpha, ~\gamma, ~\lambda_{\max}$. 
\end{theorem} 
 
\section{Experiments} \label{sec:experiments}

This section shows our experimental results  of  the approximated solution $\bar{D}$ on   real data sets, 
"Adult income data set" \cite{Lichman13} and 
"COMPAS"  recidivism data set \cite{Angwin16} 
The task of Adult income data set is  to predict if a person's income is no less than $\$ 50$k per year.  
For COMPAS data set, we use data samples whose race is either Caucasian or African-American. 
The task of COMPAS data set is to predict if an individual is rearrested within two years after the first arrest. 
In our setting, the case of recidivism within two years corresponds to label ``0", since 
the label ``0" indicates an undesirable decision to an individual and nobody wants re-arrest. 
(see  Appendix  \ref{appendix:experiments}  for detail explanations of data sets  and 
detail description of experiments. Appendix  \ref{appendix:experiments} provides additional experiments, too.)     
We use $\lambda_{max} = 20, \nu = 0.005$, 
which is empirically found  
so that the set $\Lambda = [0, \lambda_{\max}]$ is sufficiently large 
and that $\gamma >\inf_{h \in \Hcal} I_{\alpha}(\ub_h;n)$. 
Regarding to  $b_i = a(h(x_i) - y_i) +c$, 
we fix the value of $a$ as 5, and change the values 
of $c \in \{8, 9, 10\}$ to investigate 
the effect of $r= \frac{c}{a}$ on the performance.

\subsection{ Individual Fairness $I_{\alpha}$ vs. Group Fairness  $V$} \label{subsec:vs}

We investigate how between-group term $V$ behaves 
as   $\gamma$, the fairness constraint on GE (individual fairness), changes.

Fig. \ref{fig:ge_v_alpha1} shows the graphs of  $I_{\alpha}$ and $V$ for  various $\gamma$  values and the datasets. 
Recall that  the fairness constraint is  $I_{\alpha} \leq \gamma$ and  
small $I_{\alpha}$ implies    low degree of  unfairness (i.e., high degree of fairness).

\begin{figure}[bt]
  \centering
  \captionsetup{justification=centering}
  \subfigure[Adult income]
  {\label{subfig:adult_1_ge_v0}
  \includegraphics[width=0.42\columnwidth]{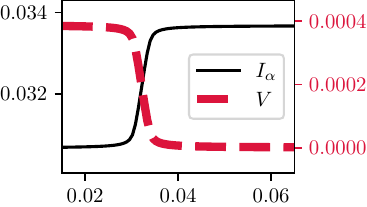}
  }
  \hskip 3.5mm
  \subfigure[COMPAS]
  {\label{subfig:compas_1_ge_v0}
  \includegraphics[width=0.42\columnwidth]{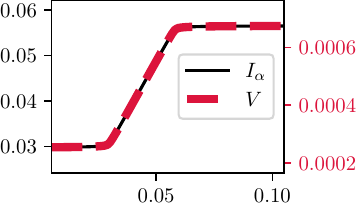}
  }  
  \subfigure[Law school]
  {\label{subfig:law_school_1_ge_v0}
  \includegraphics[width=0.38\columnwidth]{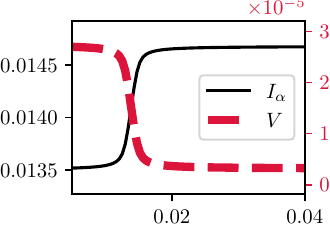}
  }  
  \subfigure[Dutch census]
  {\label{subfig:dutch_census_1_ge_v0}
  \includegraphics[width=0.44\columnwidth]{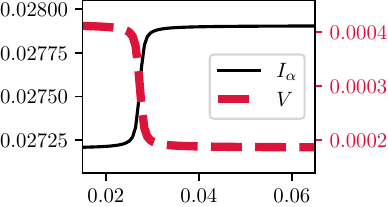}
  }
  \caption{   $I_{\alpha}$ and $V$  when $\alpha =1$\\
        ($x$ axis is $\gamma$, left $y$ axis  $I_{\alpha}$, and right $y$ axis $V$)} 
  \label{fig:ge_v_alpha1}
\end{figure}

\begin{table}[bth]
\centering
   \begin{NiceTabular}{lrr}  
      \toprule
      \bfseries Adult inc.  & $n_g$ & $r^{POS}_g$ \\
      \midrule 
        $g=0$ (Female)   & 4217             & 0.114 \\   
        $g=1$ (Male)   & \textbf{\textit{8723}} &  \textbf{\textit{0.310}} \\
      \toprule  
      \bfseries COMPAS   & $n_g$ & $r^{POS}_g$  \\
      \midrule  
        $g=0$ (African-American)   & \textbf{966} & 0.491  \\ 
        $g=1$ (Caucasian)       & 618 &  \textbf{0.607} \\
      \toprule
      \bfseries Law school  & $n_g$ & $r^{POS}_g$ \\
      \midrule  
        $g=0$ (Female)   & 2464              & 0.884 \\ 
        $g=1$ (Male)   & \textbf{\textit{3144}}  &  \textbf{\textit{0.912}}\\  
      \toprule
      \bfseries Dutch census   & $n_g$ & $r^{POS}_g$ \\ 
      \midrule  
        $g=0$ (Female)  & \textbf{\textit{9145}} & 0.332  \\
        $g=1$ (Male)   & 8981 &  \textbf{\textit{0.625}} \\
      \bottomrule 
   \end{NiceTabular}
\captionsetup{justification=centering}
\caption{ $r^{POS}_g$ for Datasets and Groups} \label{table:total_r_g_POS}
\end{table}

\begin{figure*}[t] 
  \centering 
  \subfigure[$\alpha =0$]{\label{subfig:err_0} 
  \includegraphics[width=0.65\columnwidth]{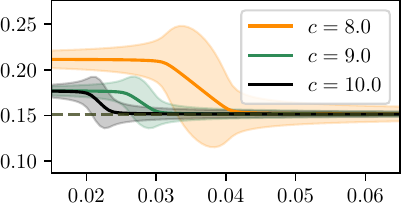}
  }
  \hfill
  \subfigure[$\alpha=1$]{\label{subfig:err_1} 
  \includegraphics[width=0.65\columnwidth]{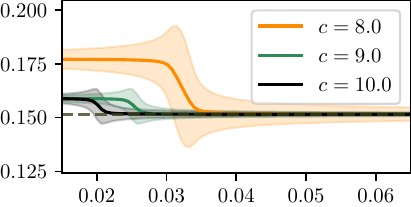}
  }
  \hfill
  \subfigure[$\alpha=2$]{\label{subfig:err_2}
  \includegraphics[width=0.65\columnwidth]{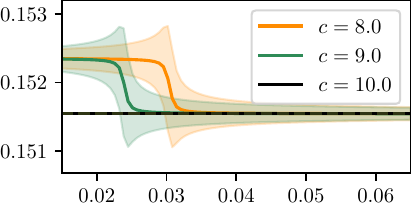}
  }
  \caption{Adult income:  test error when $a=5$ ($x$ axis is $\gamma$)}  \label{fig:test_error} 
  \vskip 1mm
  \subfigure[$\alpha =0$]{\label{subfig:ge_0}
  \includegraphics[width=0.65\columnwidth]{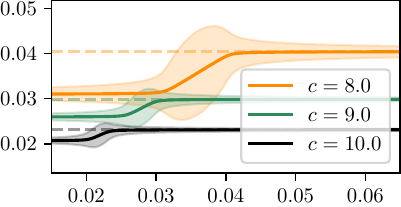}
  }
  \hfill
  \subfigure[$\alpha=1$]{\label{subfig:ge_1} 
  \includegraphics[width=0.65\columnwidth]{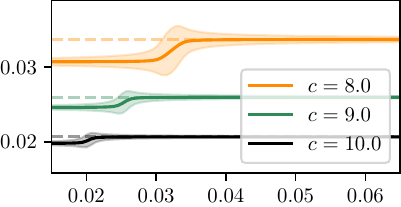}
  }
  \hfill
  \subfigure[$\alpha=2$]{\label{subfig:ge_2}
  \includegraphics[width=0.65\columnwidth]{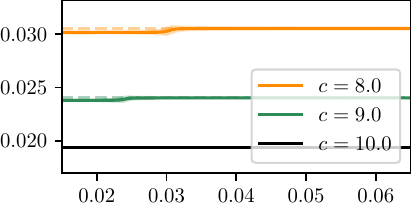}
  } 
  \caption{Adult income:  $I_{\alpha}$ when $a=5$ ($x$ axis is $\gamma$)}\label{fig:test_ge}   
\end{figure*}

In Fig. \ref{fig:ge_v_alpha1},   only COMPAS data set has 
the positive correlation between $I_{\alpha}$ and $V$ 
and other data sets have the negative correlation. 
In Fig. \ref{subfig:compas_1_ge_v0},  
$V$ and $I_{\alpha}$   are  decreasing as  $\gamma$ is decreasing: 
strengthening individual fairness indeed yields enhancing group fairness.  
However, in Fig.  \ref{subfig:adult_1_ge_v0}, \ref{subfig:law_school_1_ge_v0}, and \ref{subfig:dutch_census_1_ge_v0},  
as $\gamma$ decreases,    
between-group term $V$ increases even though $I_{\alpha}$ decreases;   
 when we strengthen the individual fairness degree,  
the degree of group fairness   gets degraded. 
Based on the experiments, we conclude that strengthening individual fairness degree 
does not always lead to enhancement of group  fairness. 
This observation has not been found in existing literature as far as we know 
since there is little work investigating 
how group fairness   changes when individual fairness is controlled. 

To  investigate  when there is a positive correlation between $I_{\alpha}$  and $V$   for binary group cases, 
we examine the cardinality $n_g$ and the value of  base rate  $r^{POS}_g$  of group $g$   for the datasets, 
which  are listed in Table \ref{table:total_r_g_POS},   
where  $r^{POS}_g$  of group $X^g$  is defined as  $X^g$ $r^{POS}_g = \Psf[y_i=1|x_i \in X^g]$. 
The female or African-American group  corresponds to $g=0$ 
and  the male or Caucasian group corresponds to   $g=1$.  
From Table \ref{table:total_r_g_POS}, we observe  that 
i)  $r^{POS}_1 > r^{POS}_0$ for all datasets; the group $g=1$ 
has higher number of true labels than the group $g=0$, and    
ii) $n_1$ is greater than or almost equal to $n_0$ except in COMPAS dataset;   
COMPAS dataset has $\frac{n_1}{n_0} =0.640$ while other datasets  has $\frac{n_1}{n_0} > 0.98$.
Based on the  observations, for binary groups, it seems  that 
$I_{\alpha}$  and $V$ have a positive correlation   
when 
  $r_g^{POS}$    of the  group with small cardinality  is  larger  than that  of the other  group.

Another observation  is that  
 the  values of between group term $V$ are very small compared to $I_{\alpha}$. 
(recall that   V is always smaller than $I_{\alpha}$   
from the property of additive decomposability),   
which implies that when we control individual unfairness $I_{\alpha}$ so that it takes a small value, 
the group unfairness is also kept as a small value. 
The reason   why  $V$ is very small compared to $I_{\alpha}$   s 
is that  there are only two groups. 
 From the equation (\ref{eqn:add_decomp}) of additive decomposability, 
we know that $V \leq I_{\alpha}$ and that $V$ approaches to $I_{\alpha}$ 
as  the number of groups is increasing (i.e., group size is decreasing) by taking intersection over features;  
in the ultimate case when  each individual becomes a group with group size one, 
it holds that  $V=I_{\alpha}$. 
The observation that  the degree of group unfairness, 
between-group term $V$,  increases as  
 the number of groups increases (or the cardinality of each group is decreasing)  
 is consistent with  the well known fact 
 that the degree of unfairness gets larger as the number of subgroups is increasing \cite{Kearns18,Kearns18a,Foulds20}.

\subsection{Trade-off between Fairness and Accuracy} 
Fig. \ref{fig:test_error} shows the trade-off between fairness and efficiency for Adult income data set. 
Each graph shows the averaged value of  test error of the randomized hypothesis $\bar D$ 
and the shaded regions over the 
graphs represent  95\% confidence intervals. 
The dotted lines represent  the  test error of the empirical risk minimizer, 
$ h_0 = \arg \min_{h \in \Hcal} L(h, 0)$. 
In Fig. \ref{fig:test_error}, all the graphs of test error exhibit the decreasing behavior 
as $\gamma$ increases. 
Based on this observation, in general, test error gets low (i.e., accuracy is enhanced) 
as fairness constraint gets loose. 
Fig. \ref{fig:test_ge} illustrate 
how  $I_{\alpha}$ of $\bar{D}$ varies  as $\gamma$  changes  
for each $\alpha =0, 1, 2$.  
Every graph in Fig. \ref{fig:test_ge} shows 
that $I_{\alpha}$ increases as $\gamma$  increases (i.e., the fairness constraint gets loose.)  
For each fixed $\alpha$, as the value of $c$ increases, 
i.e., $r = \frac{c}{a}$  increases, we observe that 
the unfairness degree $I_{\alpha}$ decreases   
since  the quantity $\frac{c-a}{c}$ which is the relative difference among benefits, 
$(c-a, c, c+a)$, gets diminishing. 
From our experiments, 
we assure that fairness is achieved at the cost of accuracy and 
that trade-off between accuracy and fairness is more sensitive to low $\alpha$ than to high $\alpha$.

\section{Summary} \label{summary-sec}
We examine the impact of controlling individual fairness ($I_{\alpha}$)  
on group fairness (between-group term) $V$ through FERM-GE, a  (randomized)   ERM 
with a fairness constraint given by GE.  
We theoretically and experimentally analyze  FERM-GE  and  prove 
that  the   randomized FERM-GE  is PAC learnable.     
Our experiments show that enhancing individual fairness, $I_{\alpha}$, does not always improve group fairness $V$  
and that   controlling individual fairness  makes group unfairness  small enough. 
 
\bibliography{aaai2026,../../bibtex_files/MLfairness,../../bibtex_files/algorithm,../../bibtex_files/economics,../../bibtex_files/math}

\begin{thebibliography}{28}
\providecommand{\natexlab}[1]{#1}

\bibitem[{Agarwal et~al.(2018)Agarwal, Beygelzimer, Dudík, Langford, and
  Wallach}]{Agarwal18}
Agarwal, A.; Beygelzimer, A.; Dudík, M.; Langford, J.; and Wallach, H. 2018.
\newblock A reductions approach to fair classification.
\newblock In \emph{Proc. of the 35th {I}nternational {C}onference on {M}achine
  {L}earning (I{CML} 2018)}.

\bibitem[{Angwin et~al.(2016)Angwin, Larson, Mattu, and Kircher}]{Angwin16}
Angwin, J.; Larson, J.; Mattu, S.; and Kircher, L. 2016.
\newblock Machine bias.
\newblock ProPublica.

\bibitem[{Boyd and Vandenberghe(2004)}]{Boyd04}
Boyd, S.; and Vandenberghe, L. 2004.
\newblock \emph{Convex {O}ptimization}.
\newblock New York: Cambridge {U}niversity {P}ress.

\bibitem[{Cousins(2021)}]{Cousins21}
Cousins, C. 2021.
\newblock An axiomatic theory of provably-fair welare-centric machine learning.
\newblock In \emph{Proc. of the 35th {C}onference on {N}eural {I}nformation
  {P}rocessing {S}ystems (Neur{IPS} 2021)}.

\bibitem[{der Laan(2017)}]{Laan00}
der Laan, P.~V. 2017.
\newblock The 2001 census in the {N}etherlands.
\newblock In \emph{Proc. of the Census of Population}.

\bibitem[{Donini et~al.(2018)Donini, Oneto, Ben-David, Shawe-Taylor, and
  Pontil}]{Donini18}
Donini, M.; Oneto, L.; Ben-David, S.; Shawe-Taylor, J.; and Pontil, M. 2018.
\newblock Empirical risk minimization under fairness constraints.
\newblock In \emph{Proc. of the 32nd Neural Information Processing Systems
  (Neur{IPS} 2018)}, 2796--2806.

\bibitem[{Dwork et~al.(2012)Dwork, Hardt, Pitaasi, Reingold, and
  Zemel}]{Dwork12}
Dwork, C.; Hardt, M.; Pitaasi, T.; Reingold, O.; and Zemel, R. 2012.
\newblock Fairness through awareness.
\newblock In \emph{Proc. of the 3rd {I}nnovations in {T}heoretical {C}omputer
  {S}cience {C}onference}, 214--226.

\bibitem[{Feldman et~al.(2015)Feldman, Fiedler, Moeller, Scheidegger, and
  Venkatasubramanian}]{Feldman15}
Feldman, M.; Fiedler, S.~A.; Moeller, J.; Scheidegger, C.; and
  Venkatasubramanian, S. 2015.
\newblock Certifying and removing dispate impact.
\newblock In \emph{Proc. of the 21th {ACM SIGKDD} of {I}nternational
  {C}onference on {K}nowledge {D}iscovery and {D}ata {M}ining (K{DD} 2015)},
  259--268.

\bibitem[{Foulds et~al.(2020)Foulds, Islam, Keya, and Pan}]{Foulds20}
Foulds, J.~R.; Islam, R.; Keya, K.~N.; and Pan, S. 2020.
\newblock An Intersectional Definition of Fairness.
\newblock In \emph{Proc. of the 36th {C}onference on {D}ata {E}ngineering
  (I{CDE} 2020)}.

\bibitem[{Freund and Schapire(1996)}]{Freund96}
Freund, Y.; and Schapire, R.~E. 1996.
\newblock Game theory, on-line prediction and boosting.
\newblock In \emph{Proc. of the 9th Annual Conference on Computational Learning
  Theory (COLT 1996)}, 325--332.

\bibitem[{Freund and Schapire(1997)}]{Freund97}
Freund, Y.; and Schapire, R.~E. 1997.
\newblock A decision-theoretic generalization of on-line learning and
  application to boosting.
\newblock \emph{Journal of Computer and System Sciences}, 55(1): 119--139.

\bibitem[{Hardt, Price, and Srebor(2016)}]{Hardt16}
Hardt, M.; Price, E.; and Srebor, N. 2016.
\newblock Equality of opportunity in supervised learning.
\newblock In \emph{Proc. of the 30th {C}onference on {N}eural {I}nformation
  {P}rocessing {S}ystems (Neur{IPS} 2016)}.

\bibitem[{Heidari et~al.(2018)Heidari, Ferrari, Gummandi, and
  Krause}]{Heidari18}
Heidari, H.; Ferrari, C.; Gummandi, K.~P.; and Krause, A. 2018.
\newblock Fairness behind a veil of Ignorance: a welfare analysis for automated
  decision making.
\newblock In \emph{Proc. of the 32nd {C}onference on {N}eural {I}nformation
  {P}rocessing {S}ystems (Neur{IPS} 2018)}, 1273--1283.

\bibitem[{Kearns et~al.(2018)Kearns, Neel, Roth, and Wu}]{Kearns18}
Kearns, M.; Neel, S.; Roth, A.; and Wu, Z.~S. 2018.
\newblock Preventing fairness gerrymandering: auditing and learning for
  subgroup fairness.
\newblock In \emph{Proc. of the 35th {I}nternational {C}onference on {M}achine
  {L}earning (I{CML} 2018)}.

\bibitem[{Kearns et~al.(2019)Kearns, Neel, Roth, and Wu}]{Kearns18a}
Kearns, M.; Neel, S.; Roth, A.; and Wu, Z.~S. 2019.
\newblock An empirical study of rich subgroup fairness for machine learning.
\newblock In \emph{Proc. of the {C}onference on {F}airness, {A}ccountablity,
  and {T}ransparency}.

\bibitem[{Kearns, Roth, and Sharifi-Malvajerdi(2019)}]{Kearns19}
Kearns, M.; Roth, A.; and Sharifi-Malvajerdi, S. 2019.
\newblock Average Individual Fairness: Algorithms, Generalization and
  Experiments.
\newblock In \emph{Proc. of the 33rd {C}onference on {N}eural {I}nformation
  {P}rocessing {S}ystems (Neur{IPS} 2019)}.

\bibitem[{Kleinberg, Mullainathan, and Raghavan(2017)}]{Kleinberg17}
Kleinberg, J.; Mullainathan, S.; and Raghavan, M. 2017.
\newblock Inherent trade-offs in the fair determination of risk scores.
\newblock In \emph{Proc. of the 8th Innovations in Theoretical Computer Science
  Conference}.

\bibitem[{Kusner et~al.(2017)Kusner, Loftus, Russel, and Silva}]{Kusner17}
Kusner, M.; Loftus, J.; Russel, C.; and Silva, R. 2017.
\newblock Counterfacutal fairness.
\newblock In \emph{Proc. of the 31st Conference on Neural Information
  Processing Systems (Neur{IPS} 2017)}.

\bibitem[{Lichman(2013)}]{Lichman13}
Lichman, M. 2013.
\newblock {UCI} machine learning repository.
\newblock \url{http://archive.ics.uci.edu/ml}.

\bibitem[{Lowy et~al.(2022)Lowy, Baharlouei, Pavan, Razaviyayn, and
  Beirami}]{Lowy22}
Lowy, A.; Baharlouei, S.; Pavan, R.; Razaviyayn, M.; and Beirami, A. 2022.
\newblock A stochastic optimization framework for fair risk minimization.
\newblock \emph{Transactions on {M}achine {L}earning {R}esearch}, 2022.

\bibitem[{Mohr, Rostamizadeh, and Talwalkar(2018)}]{Mohri2018}
Mohr, M.; Rostamizadeh, A.; and Talwalkar, A. 2018.
\newblock \emph{Foundations of {M}achine {L}earning}.
\newblock The MIT Press, 2nd edition.

\bibitem[{Rudin(1976)}]{Rudin76}
Rudin, W. 1976.
\newblock \emph{Principles of {M}athmatical {A}nalysis}.
\newblock McGraw-Hill, 3rd edition.

\bibitem[{Rudin(1987)}]{Rudin87}
Rudin, W. 1987.
\newblock \emph{Real and {C}omplex {A}nalysis}.
\newblock McGraw-Hill, 3rd edition.

\bibitem[{Sheng and Ling(2006)}]{Sheng06}
Sheng, V.~S.; and Ling, C.~X. 2006.
\newblock Thresholding for making classifiers cost-sensitive.
\newblock In \emph{Proc. of the 21st National Conference on American
  Association for Artificial Intelligence (A{AAI} 2006)}, 476--481.

\bibitem[{Shorrocks(1980)}]{Shorrocks80}
Shorrocks, A.~F. 1980.
\newblock The class of additively decomposable inequalty measures.
\newblock \emph{Econometrica: Journal of the Econometric Society}, 48(3):
  613--625.

\bibitem[{Speicher et~al.(2018)Speicher, Heidari, Grgic-Hlaca, Gummandi,
  Singla, Weller, and Zafar}]{Speicher18}
Speicher, T.; Heidari, H.; Grgic-Hlaca, N.; Gummandi, K.~P.; Singla, A.;
  Weller, A.; and Zafar, M.~B. 2018.
\newblock A unified approach to quantifying algorithmic unfairness: measuring
  individual \& group fairness via inequality indices.
\newblock In \emph{Proc. of the 24th {ACM SIGKDD} of {I}nternational
  {C}onference on {K}nowledge {D}iscovery \& {D}ata {M}ining (K{DD} 2018)},
  2239--2248.

\bibitem[{Wightman(1998)}]{Wightman98}
Wightman, L. 1998.
\newblock L{SAC} national longitudinal bar passage study.
\newblock LSAC research report series.

\bibitem[{Zafar et~al.(2017)Zafar, Valer, Rodriguez, and Gummadi}]{Zafar17}
Zafar, M.~B.; Valer, I.; Rodriguez, M.~G.; and Gummadi, K. 2017.
\newblock Fairness beyond disparate treatment \& disparate impact: learning
  classification without disparate mistreatment.
\newblock In \emph{Proc. of International World Wide Web Conference Commitee
  (I{W}3{C}2)}, 1171--1180.

\end{thebibliography}

\newpage
\appendix
\onecolumn
\section{Detail Explanation of Algorithm \ref{alg:fair}: A Learning Algorithm Achieving FERM} \label{append:learning-alg}

This section studies how to find an optimal (\emph{randomized}) hypothesis  satisfying 
the fairness constraint $I_{\alpha} \leq \gamma$ for given $\alpha$ and $\gamma$.  
By randomizing hypotheses, we can achieve better accuracy-fairness tradeoffs than  using only the pure hypotheses. 
A randomized hypothesis $D$ is a probability distribution on $\Hcal$, 
that is $D = \sum_{h \in \Hcal} D_h h$ with $\sum_{h\in\Hcal} D_h =1$. 
The sample error of $D$ is given by
$R_S(D) = \sum_{h \in \Hcal} D_h R_S(h)$ 
and the corresponding generalized entropy is  
$I_{\alpha}(D;n) = \sum_{h \in \Hcal} D_h I_{\alpha}(\ub_h;n)$ 
for some $\alpha \in[0, \infty)$. 
Let $\Delta \Hcal$ be  the set of all probability distributions over $\Hcal$. 
We consider the following (randomized) FERM-GE problem: 
\be
   & &  \min\limits_{D \in \Delta \mathcal{H}} ~R_S(D) \label{eqn:randomized_ferm_app}\\
   & & ~~\mbox{subject to }   I_{\alpha}(D;n)   \leq \gamma.  \nonumber
\ee  
The above problem (\ref{eqn:randomized_ferm_app}) is 
a convex optimization: 
the  objective function is linear in $D_h$,  $\sum_{h\in\Hcal} D_h R_S(h)$  
and the linear constraint is also linear in $D_h$, $\sum_{h\in\Hcal} D_h I_{\alpha}(\ub_h;n) \leq \gamma$, 
and we we want to find $\{D_h\}_{h\in \Hcal}$ with $\sum_h D_h =1$.   

We assume that $\gamma >\inf_{h \in \Hcal} I_{\alpha}(\ub_h;n)$. 
Since (\ref{eqn:randomized_ferm_app}) is a feasible convex problem defined on a finite dimensional space, 
\footnote{Even though $\Hcal$ has infinitely many hypotheses, since   
the hypotheses in $\Hcal$ are applied to the finite sample space $S$, the number of   different labelings of $h \in \Hcal$ 
is finite, i.e., cardinality of $\{ h(S)~|~ h \in \Hcal\}$ is finite by Sauer's Lemma \cite{Mohri2018} 
(i.e., count only once if $h_1(S) = h_2(S)$ for $h_1 \neq h_2$)
when VC dimension of is finite.}    
the duality gap is zero, i.e., 
\be 
L^* =  \max_{\lambda \in \real^+} ~\min_{D \in \Delta{\Hcal}} L(D, \lambda) 
=   \min_{D \in \Delta{\Hcal}} ~ \max_{\lambda \in \real^+} L(D, \lambda)  \label{eqn:minmax_app}
\ee  
where  $L(D, \lambda)$   is the Lagrangian of (\ref{eqn:randomized_ferm_app}), 
$L(D, \lambda) = R_S(D) + \lambda \Big(I_{\alpha}(D;n)- \gamma \Big)$.    
Note that $\lambda \in [0, \infty)$ in (\ref{eqn:minmax_app}).   
We bound  the range of $\lambda$  so that  
$ \lambda \in \Lambda = [0, \lambda_{\max}]$ to ensure the convergence to an equilibrium. 
Since $\Lambda$ is   compact and convex, it holds that 
\be 
L^*_{\Lambda} =  \max_{\lambda \in \Lambda} ~  \min_{D \in \Delta {\Hcal }} L(D, \lambda) 
= \min_{D \in \Delta {\Hcal}} ~ \max_{\lambda \in \Lambda} L(D, \lambda).\label{eqn:bounded-minmax_app}
\ee 
The  optimal solution $(D^*_{\Lambda}, \lambda^*_{\Lambda})$ of (\ref{eqn:bounded-minmax_app}), 
usually called the saddle point of $L^*_{\Lambda}$, 
can be  found 
as the equilibrium of a  repeated zero sum game  of two players, 
the learner seeking $D$ that minimizes  $ L(D, \lambda)$ and 
Nature seeking $\lambda$ that maximizes $ L(D, \lambda)$ \cite{Boyd04}.  

 In \cite{Freund96,Freund97}, the authors have proposed a provable method to find an approximate solution of 
(\ref{eqn:bounded-minmax_app}) using Hedge  algorithm.   
Exploiting  Hedge  algorithm\footnote{We modify the original Hedge algorithm for our objective 
that Nature finds $\lambda \in [0, \lambda_{\max}]$ maximizing $L(D, \lambda)$, 
since   the original Hedge algorithm is to find a randomized hypothesis minimizing loss.},
we will  find  an approximate equilibrium   $(\bar{D}, \bar{\lambda})$ 
of  $L^*_{\Lambda}$ such that  
for given $\nu >0$, it holds that 
\be 
L^*_{\Lambda}- \nu & \leq & \min_{D}~L(D, \bar{\lambda}), \label{eqn:mu-approximate-min}\\  
\max_{\lambda \in \Lambda} ~L(\bar{D},  \lambda) & \leq & L^*_{\Lambda} + \nu. \label{eqn:mu-approximate-max}
\ee

Hedge algorithm assumes the existence of an oracle yielding 
the best response of the learner: 
in our case, 
it is the item 2, $\hat{h}(t) = \arg \min_{h \in \Hcal} L(h, \hat{\lambda}^{(t)})$, 
in the for loop of Algorithm \ref{alg:fair}. 
Another assumption in Hedge algorithm is that  the amount of gain  of a strategy,   
which corresponds to  $L(h, \lambda)$, should take values between 0 and 1. 
This  amount of gain is used in item 3, the process of updating $w^{(t)}_j$. 
Since   $L(h, \lambda) \notin [0,1]$,  
we will find  $A_{\alpha}$ and $B$ that makes  
$0 \leq  l(h, \lambda) \leq 1$ where $  l(h, \lambda)=  \frac{L(h, \lambda) + B}{A_{\alpha}}$.  
Because  $L(h, \lambda) = R_S(h) + \lambda \big(I_{\alpha}(\ub_h;n) -\gamma \big)$, 
it holds that  
\beqa
\lambda (  I_{\alpha}^{\min}   -\gamma) \leq L(h,\lambda) \leq 1 + \lambda ( I_{\alpha}^{ \max}-\gamma)
\eeqa
where $I_{\alpha}^{\min} = \min_{\ub} I_{\alpha}(\ub;n)$ and $I_{\alpha}^{\max} = \max_{\ub} I_{\alpha}(\ub;n)$. 
We have $I_{\alpha}^{\min}=0$ because $I_{\alpha}(\ub;n) \geq 0$ . 
It can be easily checked  that $I_{\alpha}^{\max} \leq I_{\alpha}^{UP}(r)$ 
where 
\beqa
I_{\alpha}^{UP}(r) 
:= \left\{ \begin{array}{ll}
        \ln \Big( \frac{r+1}{r-1} \Big) & \mbox{for } \alpha =0, \\
        \Big(\frac{r+1}{r-1}\Big)\ln \Big(\frac{r+1}{r-1}\Big)   & \mbox{for } \alpha =1,\\
        \frac{1}{|\alpha(\alpha-1)|} \Big[ \Big(\frac{r+1}{r-1}\Big)^{\alpha} -1 \Big] 
                 & \mbox{for } \alpha \neq 0,1. 
   \end{array} \right.
\eeqa
Therefore, by setting  $A_{\alpha}= 1+ \lambda_{\max}(\gamma + I_{\alpha}^{UP}(r))$ 
and $B= \gamma\lambda_{\max}$, 
we have $l(h, \lambda) \in [0,1]$. 
We will use $l(h, \lambda)$ instead of $L(h, \lambda)$ 
in updating the weight vector, $(w^{(t)}_0,~w^{(t)}_1)$, in Hedge algorithm.
Note that $\arg \min_{h \in \Hcal} L(h, \lambda) = \arg \min_{h \in \Hcal} l(h, \lambda)$ and 
$\arg \max_{\lambda \in \Lambda} L(D, \lambda) = \arg \max_{\lambda \in \Lambda} l(D, \lambda)$ 
since adding and multiplying a positive  constant has no effect on optimization. 

Applying Hedge Algorithm to our case,  we have Theorem \ref{thm:approx-eq},  
a direct result of the analysis of \cite{Freund97}.  
\begin{theorem} \label{thm:approx-eq}
After $T=\frac{4 A_{\alpha}^2 \ln 2}{\nu^2}$ iterations, 
the output of the algorithm,  $(\bar{D},~\bar{\lambda})$, 
satisfies   (\ref{eqn:mu-approximate-min}) and (\ref{eqn:mu-approximate-max}). 
\end{theorem}

\begin{theorem}[Repetition of Theorem \ref{thm:compare-opt}]
Suppose that  $\gamma > \inf_{h \in \Hcal}  I_{\alpha}(\ub_h;n)$. 
For any given $\nu$,  
the randomized hypothesis $\bar{D}$   satisfies\\ 
\beqa
~~~~~~~~R_S(\bar{D}) ~ \leq ~  L^* + 2 \nu,~~~~
I_{\alpha}(\bar{D};n) ~ \leq ~  \gamma + \frac{1+2\nu}{\lambda_{\max}} 
\eeqa
after $T=\frac{4 A_{\alpha}^2 \ln 2}{\nu^2}$ iterations.
\end{theorem}
\begin{proof}(Note that this is the Proof of Theorem \ref{thm:compare-opt}):   
From the assumption  $ \inf_{D \in \Hcal} I_{\alpha}(\ub_h;n)  \leq \gamma$, 
we can find  an  optimal $D^*_{\Lambda} \in \Delta \Hcal$ such that 
\beqa
  D^*_{\Lambda} &\in &  \arg \min_{D \in \Delta \Hcal} R_S(D)   \quad \mbox{subject to } I_{\alpha}(D;n) \leq \gamma.
\eeqa 
Since $I_{\alpha}(D^*_{\Lambda};n) -\gamma  \leq 0,$ we have for any $\lambda \in [0, \lambda_{\max}]$, 
\be 
  L(D^*_{\Lambda}, \lambda) = R_S(D^*_{\Lambda}) + \lambda \big(I_{\alpha}(D^*_{\Lambda};n) -\gamma \big)  
                             \leq  R_S(D^*_{\Lambda}).  \label{eqn:aa}
\ee

\noindent Case i)  $I_{\alpha}(\bar{D};n) -\gamma  \leq 0$: \\
In this case, it is enough to  check 
$R_S(\bar{D}) \leq L^*+ 2\nu$. 
From the assumption   $I_{\alpha}(\bar{D};n) -\gamma  \leq 0$,  
we have that 
\beqa 
   L(\bar{D}, \lambda) ~= ~ R_S(\bar{D}) + \lambda(I_{\alpha}(\bar{D};n) -\gamma) ~\leq ~R_S(\bar{D}) 
\eeqa 
and know that $\max_{\lambda \in \Lambda} L(\bar{D}, \lambda)  = R_S(\bar{D})$  when $\lambda=0$. 
Therefore 
\beqa
  R_S(\bar{D}) &=&  \max_{\lambda} L(\bar{D}, \lambda) \\
               &\leq & L^* + \nu  
                        \qquad \qquad \qquad \qquad \qquad ~~ (\mbox{by  Theorem  \ref{thm:approx-eq}})\\
               &\leq & \min_{D} L(D, \bar{\lambda}) +2\nu  
                        \qquad \qquad \qquad  (\mbox{by  Theorem  \ref{thm:approx-eq}})\\
               &\leq & L(D^*_{\Lambda}, \bar{\lambda}) +2\nu \\
               &\leq & \max_{\lambda \in \Lambda} L(D^*_{\Lambda},  \lambda ) + 2\nu\\
               &=&   L(D^*_{\Lambda},  \lambda^*_{\Lambda}) +2\nu ~~=~~L^* + 2\nu . 
\eeqa  

\noindent Case ii)    $I_{\alpha}(\bar{D};n) -\gamma >0$: \\ 
Since $(\bar{D}, \bar{\lambda})$ satisfies (\ref{eqn:mu-approximate-min}) and (\ref{eqn:mu-approximate-max}) 
by Theorem  \ref{thm:approx-eq}, 
we have 
\be 
\max_{\lambda \in \Lambda} L(\bar{D}, \lambda) - \nu 
~~\leq~~ L^*
~~\leq ~~ \min_{D \in \Delta \Hcal} L({D}, \bar{\lambda}) +  \nu. \label{eqn:barD}
\ee 
By (\ref{eqn:barD}), we  have
\be 
 \max_{\lambda \in \Lambda} L(\bar{D}, \lambda) 
  \leq   \min_{D \in \Delta \Hcal} L({D}, \bar{\lambda}) +  2\nu 
  \leq   L(D^*_{\Lambda},  \bar{\lambda}) + 2\nu  
  \leq   \max_{\lambda \in \Lambda} L(D^*_{\Lambda}, \lambda) + 2\nu 
  =  L^*+ 2\nu. \label{eqn:bb}
\ee 
By the assumption $I_{\alpha}(\bar{D};n) >\gamma$, we have 
\be 
 \max_{\lambda \in \Lambda} L(\bar{D}, \lambda) 
&=& R_S(\bar{D}) + \lambda_{\max}(I_{\alpha}(\bar{D};n) -\gamma) \label{eqn:kk}\\
& \geq &  R_S(\bar{D}). \label{eqn:cc}
\ee 
By (\ref{eqn:bb}) and (\ref{eqn:cc}), 
we have 
\beqa
   R_S(\bar{D}) ~~\leq~~  \max_{\lambda \in \Lambda} L(\bar{D};n) ~~\leq~~ L^* + 2\nu. 
\eeqa 
Moreover,  
\beqa
  \lambda_{\max}(I_{\alpha}(\bar{D};n) -\gamma) 
& \leq & \max_{\lambda \in \Lambda} L(\bar{D}, \lambda) \qquad \qquad  \big(\mbox{by  }  (\ref{eqn:kk}) \big)\\
& \leq &  L^* + 2 \nu   \qquad  \qquad \qquad \big(\mbox{by  }  ~(\ref{eqn:bb}) \big) \\ 
& \leq & R_S(D^*_{\Lambda}) + 2\nu    \qquad  ~~~~~~  \big(\mbox{by  (\ref{eqn:aa})  and }  L^*= L(D^*_{\Lambda}, \lambda^*_{\Lambda})  \big) \\
& \leq & 1 + 2\nu.  \qquad  \qquad \qquad ~~\big(\mbox{since   } R_S(D^*_{\Lambda}) \leq 1 \big)
\eeqa 
Hence, $I_{\alpha}(\bar{D};n) < \gamma + \frac{2\nu}{\lambda_{\max}}$. 
\end{proof}

\section{Experimental Setup and Supplementary Experiments} \label{appendix:experiments}
\subsection{Data and Implementation} 
The data sets we used are 
\begin{itemize}
\item Adult income data set  of \cite{Lichman13} : The task is  to predict 
      if a person's income is no less than $\$ 50$K per year. 
      The population is partitioned to male, corresponding to $g=0$, 
      and female, corresponding to $g=1$. 
      The label value 1 indicates that the income of an individual is greater 
      than or equal to $\$ 50$K per year. 
\item COMPAS recidivism data set  of\cite{Angwin16}: 
      Our experiments have used  data samples  whose race attribute is either Caucasian 
      or African-American.  
      The task is to predict 
      if an individual is rearrested within two years after the first arrest. 
      The population is partitioned to   Caucasian, corresponding to $g=0$,  
      and African-American attributes, corresponding to $g=1$. 
      The label value 1 indicates no re-arrest within two years and 
      the label value 0 indicates re-arrest within two years.
\item Law school data  set of \cite{Wightman98}: The task is to predict if a student passes the bar exam. 
      The population is partitioned to male, corresponding to $g=0$, 
      and female, corresponding to $g=1$. 
      The label value 1 indicates that a student passes the bar exam.      
\item Dutch census data set of \cite{Laan00}: 
      The task is to predict if an individual has a prestigious occupation. 
      The population is partitioned to male, corresponding to $g=0$, 
      and female, corresponding to $g=1$. 
      The label value 1 indicates that an individual has a prestigious job.
\end{itemize}

Algorithm \ref{alg:fair} assumes the existence of an oracle $W(\lambda)$. 
For the implementation of  an oracle, finding 
$\hat{h}^{(t)} = \min\arg_{h \in \Hcal} L(h, \lambda)$ for $\lambda \in \{0, \lambda_{\max}\}$, 
in item 2 of Algorithm \ref{alg:fair},  
we use thresholding   of  \cite{Sheng06},  a simple  technique to directly find the best decision threshold 
for a given objective from the training data and use this 
predict the class label for test data. 
a smaller probability than this threshold then it is classified as 0, 
otherwise as  1.    
Logistic regression has been used as a base classifier for thresholding. 
We describe below  the  thresholding technique used for the implementation of  oracle 
finding $\hat{h}^{(t)} =\min \arg_{h \in \Hcal} L(h, \lambda)$ for $\lambda \in \{0, \lambda_{\max}\}$, 
in item 2 of Algorithm \ref{alg:fair}.  
First, we train logistic regression with the training data set of adult. 
During the training,  $0.5$ is the typical threshold value of logistic regression for decision: the label of an instance is predicted as 1
if its predicted probability for the positive class is higher than 
or equal to $\frac{1}{2}$, 
otherwise, its label is decided as 0.  
Second, we divide the interval $[0, ~1]$ into 201  points 
with step size $\frac{5}{1000}$. 
Each point of  201 points is used a threshold and plays the role of a hypothesis.  
Third, for each threshold, we predict the labels of instances 
by comparing the threshold value and the probability of the positive class:  
for a give instance, 
if its probability of the positive class is higher or equal to the threshold value, 
then the label of the instance is 1, otherwise, the label is 0. 
Finally,  to find the oracle for a given $\lambda$, 
we examine the value of $L(h, \lambda)$  
for every $h \in \Hcal$ and take $\hat{h} = \arg \min_{h \in \Hcal} L(h, \lambda)$. 

Each data set is split into training examples (70\%) and test examples (30\%).  
We use $\lambda_{max} = 20, \nu = 0.005$. 
Regarding to the value of $b_i = a(h(x_i) - y_i) +c$, 
we fix the value of $a$ as 5, and change the values 
of $c \in \{8, 9, 10\}$ to investigate 
the effect of $r = \frac{c}{a}$ on the performance for all data sets except Dutch census data set. 
For Dutch data set, we use $ \{9, 9.5, 10\}$ for $c$ values.  
For the figures showing the trade-off between efficiency and fairness, 
we varies the value of $ \gamma \in [0.02,  ~0.11]$ with step size 0.002. 
All figures  are obtained after  $10^4$ times running $\bar{D}$,  
on the test data set.  

{\bf Computational Resources} 
All experiments were run on a server with about  250GB RAM. 
The server is not equipped  with GPU acceleration. 
About 10 minutes is the run-time of an experiment to  generate 
all related graphs for a fixed $\alpha$ and the set of $c$ values. 
 
Appendix \ref{appendix:comparison} discusses experiments comparing our FERM-GE and existing algorithms seeking group fairness. 
 
Appendices \ref{appendix:compas-dutch} show the experimental results for the tradeoff between  fairness and accuracy 
for  COMPAS, Law school,  and Dutch census data sets, respectively. 

\subsection{Comparison  with Existing Algorithms Seeking Group Fairness} \label{appendix:comparison}

We have conducted several  experiments comparing FERM-GE and existing algorithms 
seeking   traditional 
fairness definitions such as demographic parity (DP), equalized odds (EO). 
Before stating main experimental results, we consider 
 difference between generalized entropy and traditional 
fairness definitions including DP and equalized odds, 
which helps understanding  the experimental results.  

~\\
{\bf i) Intrinsic Difference between generalized entropy and traditional fairness definitions} \\
It is worthwhile to note the intrinsic differences between generalized entropy and 
traditional fairness definitions such as  demographic parity (DP), equalized odds (EO), and equal opportunity. 
We mainly focus on DP and EO here. 

First, generalized entropy is   individual fairness    
but DP and EO  are group fairness concepts. 

Second,   FERM-GE and existing algorithms have different objectives. 
Generalized entropy  quantifies the degree of   inequality of individuals' benefits 
resulting from prediction of a hypothesis $h$. Our FERM-GE find a hypothesis whose empirical error  is small 
and empirical   $I_{\alpha}$ is low 
so that there exists small degree of inequality of  individuals' benefits. 
It is well-known that   DP and  EO  seek   (conditional) 
independence between sensitive attributes and prediction results (given ground truth values), 
as we mention  in Section \ref{sec:formulation}.  
Existing algorithms pursuing DP/EO fairness mitigates the degree of dependence between 
sensitive attributes and prediction results. 
Recall the definitions of demographic parity and equalized odds. 
For simplicity, we consider only   binary classification. 
Regarding to DP, a hypothesis $h$ achieves demographic parity fairness if 
$\Psf[h(x)= 1 | x \in X^{g} ] = \Psf[h(x)=1 ]$ for all $g$ 
where $g$ is a  value of  sensitive attribute such as gender or race;  
$g$ is either male or  female, 
if a sensitive attribute is gender and there are only male and female. 
Hence DP pursues independence between sensitive attributes and prediction results. 
Regarding to EO, a hypothesis $h$ satisfies equalized odds fairness if 
$\Psf[h(x)=1, x \in X^g | y=0] = \Psf[h(x)=1| y=0]$ and 
$\Psf[h(x)=0, x \in X^g | y=1] = \Psf[h(x)=0 | y=1]$  for all $g$ 
(which is equivalent to $\Psf[h(x)=1| y=0,   x \in X^g ] = \Psf[h(x)=1| y=0]$ and 
$\Psf[h(x)=0 | y=1,  x \in X^g] = \Psf[h(x)=0 | y=1]$  for all $g$.) 
Equalized odds seeks conditional independence between sensitive attributes and prediction results.  
Hence algorithms seeking DP or EO  fairness  reduce    (conditional) 
dependence between sensitive attributes and prediction results.

\begin{table*}
\centering
\begin{NiceTabular}{ rr|    r r r r r|r r r r r |     c|c }
\toprule 
\multicolumn{2}{r| }{Group}  &\multicolumn{5}{c|}{Male} &  \multicolumn{5}{c| }{Female} &  GE ($I_{\alpha}$) & $V$ \\
\toprule
 Individuals & $\bm{x}_i$  &$\bm{x}_1$ & $\bm{x}_2$ & $\bm{x}_3$ & $\bm{x}_4$ & $\bm{x}_5$ & $\bm{x}_6$ & $\bm{x}_7$ & $\bm{x}_8$ & $\bm{x}_9$ & $\bm{x}_{10}$ &  & \\
\midrule
True label & $y_i$    &1 & 1 & 0 & 0& 0 & 1 & 1 & 0  & 0 & 0 & - &- \\ 
\toprule
\multirow{2}{*}{  $h_0$} & Prediction    &  TP & FN & TN & FP & \textbf{\textit{TN}}  & TP & FN & TN  & FP  & \textbf{\textit{FP}} & \multirow{2}{*}{0.078498} & \multirow{2}{*}{0.003204}\\ 
& $b_i $  & 5 & 2 & 5 & 8 &\textbf{\textit{5}} & 5 & 2 & 5 & 8 & \textbf{\textit{8}} &  &  \\
\midrule
\multirow{2}{*}{  $h_1$} & Prediction    &  TP & FN & TN & FP &\textbf{FP} & TP & FN & TN  & FP  & \textbf{\textit{FP}} & \multirow{2}{*}{0.080367} & \multirow{2}{*}{0}\\ 
& $b_i $  & 5 & 2 & 5 & 8 & \textbf{8} & 5 & 2 & 5 & 8 & \textbf{\textit{8}} & &\\
\midrule  
\multirow{2}{*}{  $h_2$} & Prediction    &  TP & FN & \textbf{FP} & FP &\textbf{FP} & TP & FN & TN  & FP  & \textbf{\textit{FP}} & \multirow{2}{*}{0.078557} & \multirow{2}{*}{0.002585}\\ 
& $b_i $  & 5 & 2 &\textbf{8} & 8 & \textbf{8} & 5 & 2 & 5 & 8 & \textbf{\textit{8}} & & \\
\bottomrule
\end{NiceTabular}
\caption{Generalized entropy (GE)  and group fairness Equalized Odds of hypotheses $h_0, ~h_1,$ and $h_2$} \label{table:toy}
\end{table*}

We examine how  the degree of  fairness based on  existing group fairness definitions   behaves when we control GE,  
using a  toy example.  
Table \ref{table:toy}   shows   how GE $I_{\alpha}$,   between-group term V and EO change  
depending on hypotheses $h_0,  ~h_1,$ and $h_2$.  
We use  $(a,c)= (3,5) $ and $\alpha =2$  
for the computation of $I_{\alpha}$  of the hypotheses.  
We first consider hypothesis $h_0$. 
Obviously $h_0$ does not meet EO fairness conditions.  
The value of $I^{h_0}_{\alpha} = 0.078498$ and between-group term  $V=0.003204$. 

Consider hypothesis $h_1$  
that is identical to $h_0$ except   the prediction of $\ux_5$: 
$h_1$ has one  more error than $h_0$. 
Under  $h_1$, the prediction  of male group is  identical to that of female group. 
Hence $h_1$  meets EO fairness condition.
The degree of group fairness in EO (or $V$) is enhanced; 
we can say that group fairness EO (or $V$)  is achieved at the cost of accuracy.  
However, regarding to GE,  
we observe that $I^{h_1}_{\alpha} > I^{h_0}_{\alpha}$, 
where $I^h_{\alpha}$ denotes the value of  GE for hypothesis $h$. 
This is an example showing that enhancing group fairness (EO or $V$) 
does not guarantee the improvement of individual fairness ($I_{\alpha}$).   

Finally, consider hypothesis $h_2$ which is identical to  $h_1$ except the prediction of $\ux_3$. 
Obviously, $h_2$ does not meet the group fairness EO conditions; hence the degree of  group fairness in EO gets deteriorated. 
Note that $I^{h_2}_{\alpha} =0.078557 < I^{h_1}_{\alpha}$. 
Comparing $h_1$ and $h_2$, we know that 
the degree of group fairness EO gets degraded but individual fairness degree  $I_{\alpha}$ 
is improved.  
This example shows that enhancing individual fairness does not result in the improvement of group fairness.

~\\
{\bf ii) Comparisons between FERM-GE and  Existing Fair Classification Algorithms} \\
\begin{figure*}[th]  
\centering
  \hspace*{7mm} \includegraphics[trim=0 50 0 0, clip, totalheight=0.033\textheight]{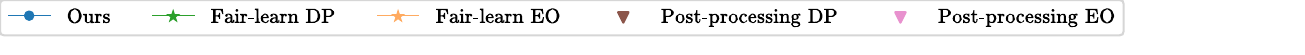}
  \captionsetup{justification=centering}
  \subfigure[EO Violation - Error Comparison]
  {\label{fig:EOV} 
  \includegraphics[width=0.28\columnwidth,  height=1.in]{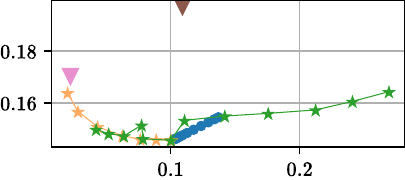} 
  } 
  \hfill
  \subfigure[DP Violation - Error Comparison]
  {\label{fig:DP}
  \includegraphics[ width=0.28\columnwidth,  height=1.in]{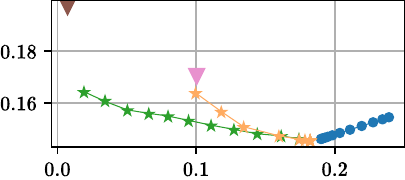}   
  }
  \hfill  
  \subfigure[GE - Error Comparison]
  {\label{fig:ge} 
  \includegraphics[width=0.28\columnwidth,  height=1.in]{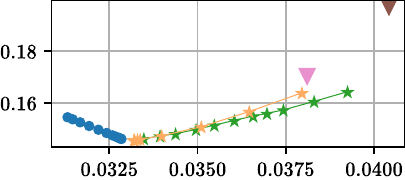}
  }
  \caption{Comparison with existing algorithms ($y$ axis is error)}  
  \label{fig:compare_ge_v_alpha1}
\end{figure*}

For comparisons with existing algorithms seeking group fairness 
such as demographic parity (DP) and equalized odds (EO), 
 we provide several experiments with the algorithms proposed by  \cite{Agarwal18} and \cite{Hardt16}. 
Our experiments are done with codes provided by ``fair-learn"   algorithm proposed by \cite{Agarwal18},  
which can be found in  https://github.com/fairlearn/fairlearn. 
The codes in fair-learn also perform  the approaches in  \cite{Hardt16} 
which are  post-processing fair-algorithms for demographic parity (DP) and equalized odds (EO). 
We used  Adult data set provided by the fair-learn package in github 
and logistic regression as the base classifier for all the new experiments. 

Five algorithms are considered;
\begin{enumerate}[label=\arabic*)]
\item our FERM-GE: minimize empirical error with fairness condition $I_{\alpha}<\gamma$ with $ 0.03 \leq \gamma \leq 0.037$, $a=5, ~c=8$.
\item  the fair-learn algorithm for DP (Fair-learn DP) : minimize empirical error with DP fairness constraints
\item  the fair-learn algorithm for EO (Fair-learn EO) :  minimize empirical error with EO fairness constraints
\item  the post-processing algorithm in \cite{Hardt16} for DP (Post-processing DP), and 
\item  the post-processing algorithm in \cite{Hardt16} for EO (Post-processing EO).
\end{enumerate}

Figure \ref{fig:EOV} shows test error and EO violations for  Fair-learn DP,  Fair-learn EO, Post-processing EO, and ours. 
The algorithm Post-processing EO is represented by a down-pointing triangle. 
Since Fair-learn EO  has the  objective (or constraints) to mitigate EO unfairness, 
its graphs has the behavior that test error is decreasing as EO violation is increasing. 
The graph of our algorithm exhibits the opposite behavior to the graph of Fair-learn EO. 
Remark the graph of Fair-learn DP with green circles  does not show any specific behavior 
even though Fair-learn DP and Fair-learn EO are fair-learn algorthms in \cite{Agarwal18};  
  the only difference between Fair-learn DP  and Fair-learn EO is the objective, mitigating DP or EO. 
These are not surprising nor strange, 
simply because our algorithm and Fair-learn DP do not aim to reduce Equalized odds unfairness degree: 
our FERM-GE targets diminishing GE and Fair-learn reducing DP unfairness.  
Since \cite{Hardt16} proposes a post-processing algorithm, 
it is represented by a single down-pointing triangle in Figure \ref{fig:EOV}.  
It  achieves the smallest EO deviation among all of the algorithms, but has the highest test error. 
This phenomena can also be interpreted as that enhancing group fairness (EO) 
may result in degradation of individual fairness (generalized entropy $I_{\alpha}$), which is well known. 

Figure \ref{fig:DP} shows test error and demographic parity (DP) violations for  Fair-learn DP,  Fair-learn EO, 
Post-processing DP, and ours. 
The algorithm, Post-processing DP is represented by a down-pointing triangle.  
Since Fair-learn DP has the  objective (or constraints) to mitigate DP unfairness, 
test error of its graph  is decreasing as demographic parity violation is increasing. 
But the graph of ours  shows an opposite behavior to the graph of Fair-learn DP. 
Another possible interpretation is that enhancing individual fairness may cause degradation of group fairness, 
which is a coherent interpretation as in Section \ref{subsec:vs}.

Figure \ref{fig:ge} shows test error and the values of GE for all of the five algorithms. 
Since our algorithm wants small GE values, the test error of  our algorithm  is decreasing as  $I_{\alpha}$ is increasing. 
But each of the test errors of  algorithms  Fair-learn DP and  Fair-learn EO is   increasing as $I_{\alpha}$ is increasing.  

As we mentioned above, 
the fairness metric $I_{\alpha}$ does not consider such independence  but difference 
among benefits of individuals, 
the intrinsic differences  between GE and existing fairness concepts including DP and EO 
  i) GE is a metric for individual fairness and DP/EO are group fairness concepts, 
ii) our FERM-GE pursues fairness of individuals benefit but 
DP and EO seek independence of sensitive attributes and prediction results. 
The intrinsic differences yield that enhancing DP/EO can result in large $I_{\alpha}$ 
and improving $I_{\alpha}$ may lead to degrade the degree of fairness DP/EO. 
As a result,  increasing group fairness may result in degradation of individual fairness and vice versa.
 
\subsection{Experiments on COMPAS, Law School Admissions, and Dutch Census data set} \label{appendix:compas-dutch}

\begin{figure}[h]
  \begin{center}
  \vspace{-3mm}
  \subfigure[$\alpha =0$]{\label{subfig:compas_err_0}
  \includegraphics[width=0.25\columnwidth]{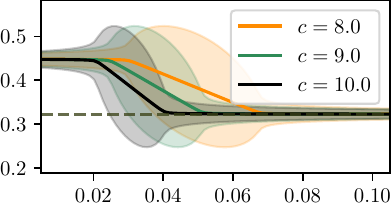}
  }
  \hfill
  \subfigure[$\alpha=1$]{\label{subfig:compas_err_1} 
  \includegraphics[width=0.25\columnwidth]{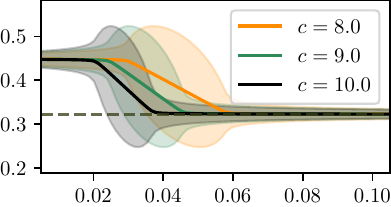}
  }
  \hfill
  \subfigure[$\alpha=2$]{\label{subfig:compas_err_2}
  \includegraphics[width=0.25\columnwidth]{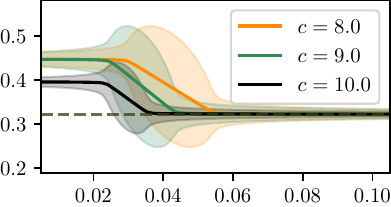}
  }       
  \vspace{-2mm}
  \caption{COMPAS: Averaged test error when $a=5$ ($x$ axis is $\gamma$) }  \label{fig:compas_test_error} 
  \vspace{5mm}
  \subfigure[$\alpha =0$]{\label{subfig:compas_ge_0}
  \includegraphics[width=0.28\columnwidth]{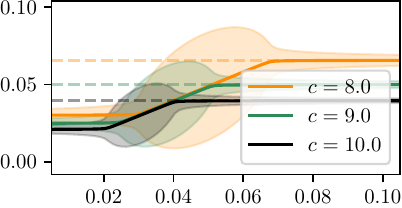}
  }
  \hfill
  \subfigure[$\alpha=1$]{\label{subfig:compas_ge_1} 
  \includegraphics[width=0.28\columnwidth]{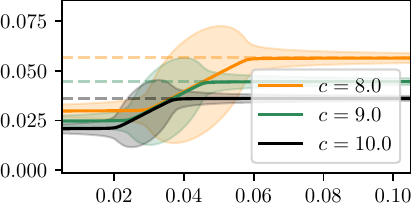}
  }
  \hfill
  \subfigure[$\alpha=2$]{\label{subfig:compas_ge_2}
  \includegraphics[width=0.28\columnwidth]{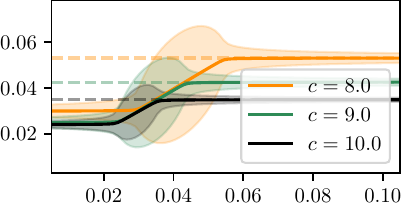}
  }
  \vspace{-2mm}
  \caption{COMPAS: Averaged test $I_{\alpha}$ when $a=5$ ($x$ axis is $\gamma$)}\label{fig:compas_test_ge}  
 \end{center} 
\end{figure}
This subsection provides the plots illustrating the change of test error 
and test $I_{\alpha}$ obtained by the hypothesis $\bar{D}$ of Algorithm \ref{alg:fair} 
on COMPAS, law school admissions and Dutch census data sets. 
Figures \ref{fig:compas_test_error} and \ref{fig:compas_test_ge} are for average test error and $I_{\alpha}$ for COMPAS data set. The behaviors of them are similar to those 
of adult income data set. 
Figures \ref{fig:law_school_err} and \ref{fig:law_school_ge} are for  law shool admissions  data set.
  Figures \ref{fig:dutch_census_err} and \ref{fig:dutch_census_ge} are for Dutch census  data set.

\begin{figure}[h]
  \begin{center} 
  \subfigure[$\alpha =0$]{\label{subfig:law_school_err_0}
  \includegraphics[width=0.28\columnwidth]{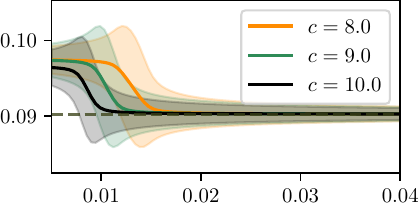}
  }
  \hfill
  \subfigure[$\alpha=1$]{\label{subfig:law_school_err_1} 
  \includegraphics[width=0.28\columnwidth]{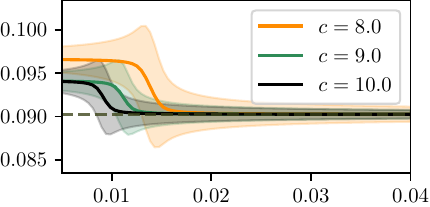}
  }
  \hfill
  \subfigure[$\alpha=2$]{\label{subfig:law_school_err_2}
  \includegraphics[width=0.28\columnwidth]{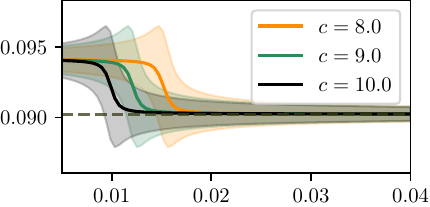}
  }     
  \caption{Law school: Averaged test error when $a=5$   ($x$ axis is $\gamma$)}  
  \label{fig:law_school_err} 
  \vspace{4mm}
  \subfigure[$\alpha =0$]{\label{subfig:law_school_ge_0}
  \includegraphics[width=0.28\columnwidth]{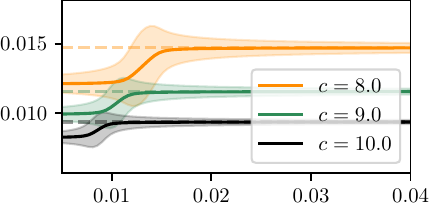}
  }
  \hfill
  \subfigure[$\alpha=1$]{\label{subfig:law_school_ge_1} 
  \includegraphics[width=0.28\columnwidth]{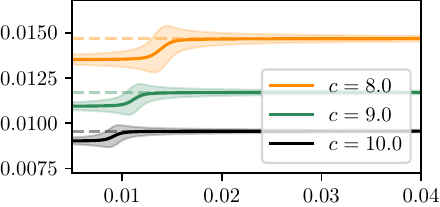}
  }
  \hfill
  \subfigure[$\alpha=2$]{\label{subfig:law_school_ge_2}
  \includegraphics[width=0.28\columnwidth]{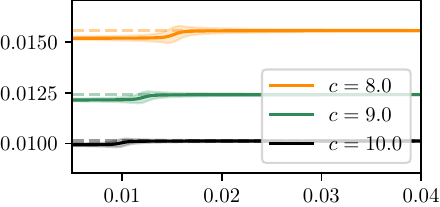}
  } 
  \caption{Law school: Averaged test $I_{\alpha}$ when $a=5$   ($x$ axis is $\gamma$)}
  \label{fig:law_school_ge}
 
\vspace{6mm}
  \subfigure[$\alpha =0$]{\label{subfig:dutch_census_err_0}
  \includegraphics[width=0.28\columnwidth]{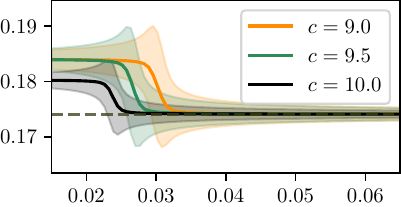}
  }
  \hfill 
  \subfigure[$\alpha=1$]{\label{subfig:dutch_census_err_1} 
  \includegraphics[width=0.28\columnwidth]{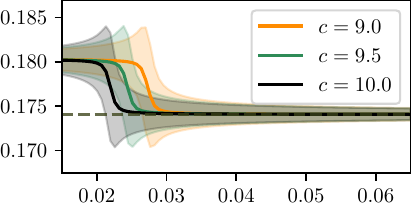}
  }
  \hfill
  \subfigure[$\alpha=2$]{\label{subfig:dutch_census_err_2}
  \includegraphics[width=0.28\columnwidth]{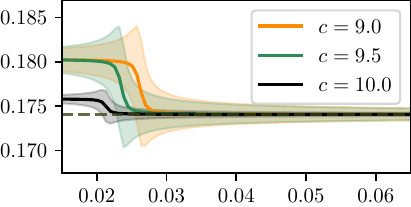}
  }   
  \caption{Averaged test error when $a=5$ for Dutch census data set  ($x$ axis is $\gamma$)}  
  \label{fig:dutch_census_err}  
  \vspace{4mm}
  \subfigure[$\alpha =0$]{\label{subfig:dutch_census_ge_0}
  \includegraphics[width=0.28\columnwidth]{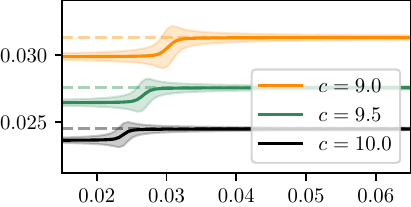}
  }
  \hfill
  \subfigure[$\alpha=1$]{\label{subfig:dutch_census_ge_1} 
  \includegraphics[width=0.28\columnwidth]{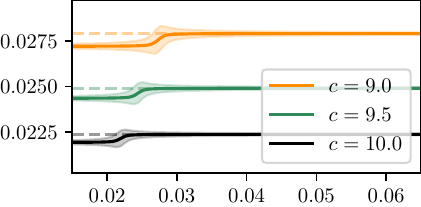}
  }
  \hfill
  \subfigure[$\alpha=2$]{\label{subfig:dutch_census_ge_2}
  \includegraphics[width=0.28\columnwidth]{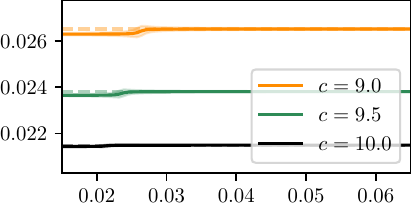}
  }   
  \caption{Dutch census:  Averaged test $I_{\alpha}$ when $a=5$ ($x$ axis is $\gamma$)}
  \label{fig:dutch_census_ge}
  \end{center} 
\end{figure}

\newpage
\section{A Computation Example of Additive Decomposability} \label{appendix:add-decom}
We explain with an example how to compute generalized entropy $I_{\alpha}(\ub;n)$ 
and check the additive decomposability, (\ref{eqn:add_decomp}),   
\beqa 
 I(\bm{b}^1, ..., \bm{b}^G;n)= \sum_g w_g^G(\boldsymbol{\mu}, \bm{n})I(\bm{b}^g;n_g) + V.
\eeqa  
In this example, we set $\alpha =1$  and use $I$ instead of $I_{\alpha}$.
Consider 9 individuals $\{\bm{x}_1, \bm{x}_2,...,\bm{x}_{9}\}$ and 
two groups $X^1$ and $X^2$ such that 
$X^1 = \{\bm{x}_1,...\bm{x}_5\}$ and $X^2 = \{\bm{x}_6, .., \bm{x}_{9}\}$. 
The true label of $i$ is denoted by $y_i$ and its predicted label by $h(\bm{x}_i)$. 
The benefit value $b_i$ is defined by  $b_i =  h(\bm{x}_i) - y_i  + 3 $ for 
a given classifier $h$. 
Then $\bm{b} = (\bm{b}^1, \bm{b}^2) = (b_{\bm{x}_1}, ..., b_{\bm{x}_{9}})$. 
For groups $X^1$ and $X^2$,  
we have $\bm{b}^1= (b_1, ...,b_5)$ 
and $\bm{b}^2=(b_6, ..., b_{9})$.
For this special case of groups $X^1$ and $X^2$, 
if we use a simple notation $w_i$ instead of  $w^G_i(\boldsymbol{\mu}, \bm{n})$, 
then Axiom 4 (the additivie decomposability) is written by 
\beqa
I(b_1, ...,b_{9};9) = w_1 I(b_1,...,b_5;5) + w_2 I(b_6,...,b_{9};4) +V
\eeqa
Note that $I(b_1,...,b_5;5)$ is the generalized entropy 
for group $X^1$ whose members have benefits 
$b_1, b_2,...,b_5$. 
Similarly $I(b_6, ...,b_{9})$  
is the generalized entropy for  group $X^2$ 
with the benefits, $b_6, b_7,...,b_{9}$.  
 
Consider a classifier $h$ whose prediction $h(x_i)$ is given in Table \ref{table:appendix}.  
Let $n_C$ be the number of correct labels, $n_{FP}$ the number of false positive labels, and $n_{FN}$ the number of false negative labels. 
Then the average of the value $b_i$'s and 
generalized entropy $I^{h}(\bm{b}_{h};10)$ 
for the whole population  are  
\beqa
 \mu &=& \frac{1}{9} \sum_{i=1}^{9} b_i  
        ~=~ \frac{1}{9}\Big(n_C*3 +  n_{FP}*4 + n_{FN}*2 \Big) 
        ~=~  \frac{26}{9},\\
  I(\bm{b}_{h};9) &=& \frac{1}{9}\Big( n_C * \frac{3}{\mu} \ln\frac{3}{\mu}
+  n_{FP} * \frac{4}{\mu} \ln\frac{4}{\mu}   
+ n_{FN} * \frac{2}{\mu} \ln\frac{2}{\mu}\Big) 
~=~  0.032869. 
\eeqa

Let's compute the generalized entropy of  group $X^1$ 
that has five individuals $x_1, .., x_5$. 
Consider the average of $b_i$ for the individuals of $X^1$, 
i.e., 
\beqa
\mu_1 = \frac{1}{5} \sum_{i=1}^5 b_i 
= \frac{2*3 + 1*4 + 2*2}{5} = \frac{14}{5}.
\eeqa 
The generalized entropy for $X^1$ is 
\beqa 
 I(b_{x_1},.., b_{x_5};5)  
&=& \frac{1}{5}\Big(2* \frac{3}{\mu_1} \ln \frac{3}{\mu_1}  
+ 1* \frac{4}{\mu_1} \ln \frac{4}{\mu_1}  + 2* \frac{2}{\mu_1} \ln \frac{2}{\mu_1}  \Big)\\
&=& 0.035341.
\eeqa

\begin{table} 
\begin{center}
\begin{NiceTabular}{ c|r r r r r|r r r r }
\toprule 
Group & \multicolumn{5}{c|}{$X^1$} &  \multicolumn{4}{c}{$X^2$} \\
\midrule 
$\bm{x}_i$ & $\bm{x}_1$ & $\bm{x}_2$ & $\bm{x}_3$ & $\bm{x}_4$ & $\bm{x}_5$ & $\bm{x}_6$ & $\bm{x}_7$ & $\bm{x}_8$ & $\bm{x}_9$ \\
\midrule
$y_i$ & 1 & 1 & 1 & 1& 0 & 1 & 0 & 1   & 1 \\ 
\midrule
$h(\ux_i)$  & 1 & 0 & 1 & 0 & 1 & 1 & 1 & 1 &0  \\ 
\midrule
C/FP/FN & C & FN & C & FN & FP & C & FP & C &FN  \\ 
\midrule
$b_i $  & 3 & 2 & 3 & 2 & 4 & 3 & 4 & 3 & 2 \\
\bottomrule  
\end{NiceTabular} 
\end{center}
\caption{Classifier $h$} \label{table:appendix}
\end{table}

Similarly, for group $X^2 = \{x_6, ..., x_{9}\}$, 
we have $\mu_2 = 3$ and 
$I(b_{x_6}, ...,b_{x_{9}};4) = 0.028317. $
Now we can find $w_1 = \frac{n_1}{n}\frac{\mu_1}\mu = \frac{5}{9}\frac{\frac{14}{5}}{\frac{26}{9}} = \frac{14}{26}$ and 
$w_2 =  \frac{4}{9}\frac{3}{\frac{26}{9}} = \frac{12}{26}$. 

Now compute $V$. For this, 
we consider two groups,  
$\{j_1, j_2, ..., j_5\}$ and $\{j_6, j_7, ..., j_{9}\}$,  and 
the benefit of each  individual is   such that 
$b_{j_1} = b_{j_2} = \cdots= b_{j_5} = \mu_1 = \frac{14}{5}$  
and $b_{j_6} = b_{j_7} =\cdots= b_{j_{9}} = \mu_2  =3$.
Note that every member in $\{j_1,.., j_5\}$ have 
the identical benefit $\frac{14}{5}$  and 
every member member in $\{j_6,.., j_{9} \}$ have 
identical benefit $3$.  
Definitely the average value of the benefit 
for the individuals $j_1,...j_{9}$  is $\mu$.  
Hence 
\beqa 
V &=& I( b_{j_1},..,b_{j_{9}};9)  
= \frac{1}{9}\Big(5 *\frac{ b_{j_1} }{\mu} \ln  \frac{ b_{j_1} }{\mu} 
    + 4* \frac{b_{j_6} }{\mu} \ln \frac{ b_{j_6} }{\mu} \Big)\\
&=& \frac{1}{9} \Big( 5* \frac{14/5}{26/9} \ln \frac{14/5}{26/9} + 4* \frac{3}{26/9} \ln \frac{3}{26/9}\Big) \\
&=& 0.00059.
\eeqa 

Indeed we can check that 
$I(b_{x_1}, ..., b_{x_{9}};9) = w_1  I(X^1;5)  +  w_2  I(X^2;4) + V$ 
with the values of $w_1 I(X^1;5)=0.01903$,  
$w_2 I(X^2;4) = 0.013069$, and $V = 0.00059$.

\section{Proof of Theorem \ref{thm:gen-fair-constraint} } \label{appendix:proof-gen-fair-constraint}
\begin{theorem}[McDiarmid's Inequality]
Let  $X_1, X_2, \ldots, X_n$ be i.i.d random variables defined on $\Xcal$.   
Consider a  function $ \phi:\Xcal^n \rightarrow \real$.  
 For all $ 1 \leq i \leq n$ and all $\ux_1, \ldots, \ux_{i-1}, \ux_{i+1},\ldots, \ux_n \in \Xcal$, 
if the function satisfies 
\be  
\sup_{\ux \in V} \phi(\ux, \ux_{-i}) - \inf_{\ux \in V}\phi(\ux, \ux_{-i}) \leq q_i  \label{eqn:bounded_diff}
\ee   
 with   $q_i>0$ and $(\ux, \ux_{-i}) = (\ux_1, \ldots, \ux_{i-1}, \ux, \ux_{i+1}, \ldots, \ux_n)$,  then 
\beqa
  \Psf \Big[ \big |\phi(X_1,.., X_n) - \Esf[\phi(X_1,.., X_n)] \big| \geq \varepsilon \Big] 
  \leq 2 \exp\Big(\frac{-2\varepsilon^2}{\sum_{i=1}^n q_i^2}\Big).
\eeqa    
\end{theorem}

\begin{lemma} \label{lemma:McDiarmid-b} With probability  at least $1 - \frac{\delta}{2}$, each of the followings holds:\\
i)   $\Big|\Esf[\beta]-\mu \Big|  < 2a\sqrt{\frac{1}{2n} \ln \frac{4}{\delta}}$, 
\\
ii)  $ \Big|  \Esf[\ln \beta] -\frac{1}{n}\sum_{i=1}^n \ln b_i \Big|   
 \leq \ln \Big(\frac{c+a}{c-a} \Big)\sqrt{\frac{1}{2n} \ln \frac{4}{\delta}}$,  
\\
iii) 
$\Big| \frac{1}{n}\sum_{i=1}^n b_i \ln b_i-  \Esf[\beta  \ln \beta] ~\Big|
  \leq     \Delta_1 \sqrt{\frac{1}{2n}\ln \frac{4}{\delta}}$, \\ 
\\
iv) 
$\Big|\Esf[\beta^{\alpha}]-\frac{1}{n} \sum_i b_i^{\alpha}\Big| \leq \Delta_{\alpha} \sqrt{\frac{1}{2n} \ln\frac{4}{\delta}}  $ 
 for $\alpha \neq 0, 1$\\
where $\Delta_1 = (c+a)\ln(c+a)-(c-a)\ln(c-a)$ and $\Delta_{\alpha} =(c+a)^{\alpha} -(c-a)^{\alpha}$.
\end{lemma}
\begin{proof}
We will apply McDiarmid's Inequality  for i) and ii). \\
i): \\
We set $\phi(\ux_1, \ldots, \ux_n) = \frac{1}{n} \sum_{i=1}^n  b(\ux_i)$. 
Then  $\phi(\ux_1, \ldots, \ux_n)$ satisfies (\ref{eqn:bounded_diff}) with  $q_i = \frac{2a}{n}$, 
since $b(\ux) \in \{c-a, c, c+a\}$. 
By McDiarmid's Inequality, we have  
\beqa 
\Psf \Big[~ \Big|  \frac{1}{n}\sum_{i=1}^n b_i   - \Esf[\beta]\Big| > \varepsilon \Big] & \leq & 2\exp \Big(\frac{-2n \varepsilon^2}{4a^2} \Big). 
\eeqa
If we take $2\exp \Big(\frac{-2n \varepsilon^2}{4a^2} \Big) = \frac{\delta}{2}$, 
then   i) holds. 
\\
ii): \\
We set function $\phi(\ux_1, \ldots, \ux_n) = \frac{1}{n} \sum_{i=1}^n  b(\ux_i)$. 
Then, $\phi(\ux_1, \ldots, \ux_n)$ satisfies (\ref{eqn:bounded_diff})  
with $q_i = \frac{\ln\big(\frac{c+a}{c-a}\big)}{n}$. 
Moreover 
$\Esf[\phi(\ux_1, \ldots, \ux_n)] = \frac{1}{n} \Esf\Big[\sum_{i=1}^n \ln b(\ux_i)\Big] = \Esf[\ln b(\ux)] =\Esf[\ln \beta].$  
By McDiarmid's Inequality,  we have 
\beqa 
\Psf \Big[~~\Big|  \frac{1}{n}\sum_{i=1}^n \ln b_i - \Esf[\ln \beta]~\Big| > \varepsilon ~~\Big] 
 & \leq & 2 \exp \Big( \frac{-2 n \varepsilon ^2}{\ln^2 \big( \frac{c+a}{c-a}\big)} \Big) 
\eeqa  
If we take $2 \exp\Big(\frac{-2 n \varepsilon ^2}{\ln^2 \big( \frac{c+a}{c-a}\big)}\Big)= \frac{\delta}{2}$, 
then ii) holds. 
\\
iii): \\
Apply  McDiarmid's Inequality by setting 
$\phi(\ux_1, \ldots, \ux_n) = \frac{1}{n} \sum_{i=1}^n b(\ux_i) \ln b(\ux_i)$ 
and  $q_i = \frac{\Delta_1}{n}$  
and  taking $ 2 \exp \Big(\frac{-2n \varepsilon^2}{\Delta_1^2} \Big) =  \frac{\delta}{2}$.
\\
iv):\\
Apply  McDiarmid's Inequality
for $\phi(\ux_1, \ldots, \ux_n) = \frac{1}{n} \sum_{i=1}^n b(\ux_i)^{\alpha}$ with 
$q_i=  \frac{\Delta_{\alpha}}{n}$. 
\end{proof}

For the purpose of  distinguishing $b_i$ and $b(\bx)$ and simple notation,  
we use $\beta$ for $b(\ux)$ instead of  $b(\bx)$: for example,  $\Esf[\beta]$  stands for  $\Esf[b(\bx)]$ 
and $\Esf[\beta \ln(\beta)]$ for $\Esf[b(\bx) \ln b(\bx)]$. 

Let  $M = \max(\mu, \Esf[\beta])$ and $m= \min(\mu, \Esf[\beta])$. 
Note that $ c-a \leq m \leq M \leq c+a$. 

\subsection{When $\alpha=0$}~\\
Recall that  $\mu  = \frac{1}{n} \sum_{i=1}^n b_i$  
where $b_i = a(h(\ux_i) - y_i) + c$ with  $c>a>0$ and $c-a\geq 1$. 
Consider $ \Big| I_0(h, P) - I_0(\ub;n) \Big|$:
\be  
\Big| I_0(h, P) - I_0(\ub;n) \Big| 
 &=& \Big| \ln \Esf[\beta] - \ln \mu - \Big( \Esf[\ln \beta]- \frac{1}{n}\sum_{i=1}^n \ln b_i \Big) \Big| \nonumber\\
 &\leq &  \Big| \ln \Esf[\beta] - \ln \mu  \Big| + \Big| \Esf[\ln \beta]- \frac{1}{n}\sum_{i=1}^n \ln b_i \Big| 
        \label{eqn:pre_diff_I0}\\
 &\leq &
       \frac{|\Esf[\beta]-\mu|}{c-a} + \Big| \Esf[\ln \beta]- \frac{1}{n}\sum_{i=1}^n \ln b_i \Big|.  \label{eqn:diff_I0} 
\ee 
The last inequality holds since 
\be
\Big| \ln \Esf[\beta] - \ln \mu  \Big| 
 & =  & \ln \Big(1+ \frac{M-m}{m}\Big)    \nonumber \\
 & \leq & \frac{M-m}{m}      \quad (\because \ln(1+x) < x) \label{eqn:before_diff_lnEb} \\
 & \leq & \frac{|\Esf[\beta]-\mu|}{c-a} \quad (\because m \geq c-a).   \label{eqn:diff_lnEb}
\ee  
By Lemma \ref{lemma:McDiarmid-b}  and union bounds, 
(\ref{eqn:diff_I0}) becomes   
\beqa
  \Big| I_0(h, P) - I_0(\ub;n) \Big| 
  & \leq & \Big( \frac{2a}{c-a} +  \ln \Big(\frac{c+a}{c-a}\Big) \Big)\sqrt{\frac{1}{2n} \ln \frac{4}{\delta}}. 
\eeqa
with probability at least $1-\delta$.

\subsection{When $\alpha=1$}~\\
Recalling   $ \Esf[\beta]=\Esf[b(\ux)], ~\Esf[\beta \ln \beta] = \Esf[b(\ux)\ln b(\ux],$ 
and  $I_1(\ub;n) = \frac{1}{n}\sum_{i=1}^n \frac{b_i}{\mu} \ln \frac{b_i}{\mu}$,  
we have  
\be 
I_1(h, P)  &=& \frac{\Esf[\beta \ln \beta)]}{\Esf[\beta]}  - \ln \Esf[\beta] \label{eqn:I_1(h,P)} \\ 
I_1(\ub;n) &=& \Big(\frac{1}{n \mu}\sum_{i=1}^n  b_i \ln b_i \Big)- \ln \mu. \label{eqn:I_1(b;n)}
\ee    
It holds that 
\beqa   
|I_1(h, P) - I_1(\ub;n) 
&\leq & \Big| \frac{\Esf [\beta \ln \beta]}{\Esf[\beta]} - \frac{1}{n \mu}\sum_{i=1}^n  b_i \ln b_i \Big| 
   ~+~ \Big|\ln\Esf[\beta] - \ln \mu\Big|   \\
&\leq &  \Big| \frac{\Esf [\beta \ln \beta]}{\Esf[\beta]} - \frac{1}{n \mu}\sum_{i=1}^n  b_i \ln b_i \Big| 
   ~+~ \frac{| \Esf[\beta]-\mu|}{c-a}  ~~ (\mbox{by } (\ref{eqn:diff_lnEb}))   \\
&\leq &   \frac{(1+\ln (c+a))|\mu -\Esf[\beta]|  }{ c-a}   
       + \frac{ | \Esf[\beta \ln \beta] - \frac{1}{n}\sum_{i=1}^n  b_i \ln b_i|}{c-a}.
\eeqa  
The second inequality holds since     
\beqa
\Big|\frac{\Esf [\beta \ln \beta]}{\Esf[\beta]} - \frac{1}{n \mu}\sum_{i=1}^n  b_i \ln b_i \Big|
&=&  \Big|\frac{(\mu-\Esf[\beta]) \Esf[\beta \ln \beta]}{\mu \Esf[\beta]} 
     + \frac{\Esf[\beta](\Esf[\beta \ln \beta]- \frac{1}{n}\sum_{i=1}^n  b_i \ln b_i  \Big)}{\mu\Esf[\beta]} \Big|   \\
&\leq & \Big|\mu -\Esf[\beta]\Big| \Big| \frac{\Esf[\beta \ln \beta]}{ \mu \Esf[\beta]} \Big|  
       + \frac{1}{\mu} \Big| \Esf[\beta \ln \beta] - \frac{1}{n}\sum_{i=1}^n  b_i \ln b_i\Big| \\
&\leq &    \frac{\ln (c+a)|\mu -\Esf[\beta]|}{ c-a}   
       + \frac{| \Esf[\beta \ln \beta] - \frac{1}{n}\sum_{i=1}^n  b_i \ln b_i|}{c-a}  \\
& & \quad (\because \Esf[\beta \ln \beta] \leq \Esf[\beta \ln(c+a)],  ~c-a \leq b(\ux) \leq   c+a ).
\eeqa
By Lemma \ref{lemma:McDiarmid-b} and union bounds,  with probability at least $1-\delta$, it holds that  
\beqa  
|I_1(h, P) - I_1(\ub;n)|
&\leq & \frac{2a(1+\ln(c+a)) + \Delta_1}{c-a}\sqrt{\frac{1}{2n} \ln \frac{4}{\delta}}     \\
&=& \big(\frac{2a+4a\ln(c+a)}{c-a}+\ln \frac{c+a}{c-a} \big)\sqrt{\frac{1}{2n} \ln \frac{4}{\delta}} \\
&& \qquad (\because \Delta_1=(c-a)\ln \frac{c+a}{c-a}+2a \ln(c+a)). 
\eeqa

\subsection{When  $\alpha \neq 0, 1$}~\\
Note that $\alpha \neq 0, ~1$  means   $\alpha \in [0, 1 ) \cup (1, \infty)$, 
since $\alpha \in [0, \infty)$.   
If we let  $B_0=\Esf^{\alpha}[\beta]\Esf[\beta^{\alpha}]$,    it holds that  
\be 
\Big| I_{\alpha}(h, P) - I_{\alpha}(\ub;n) \Big| 
&=& \Big|\frac{ \mu^{\alpha} \Esf[\beta^{\alpha}] - B_0+B_0  
         - \Esf^{\alpha}[\beta] \frac{1}{n}\sum_{i} b_i^{\alpha}}{\alpha(\alpha-1)(\mu \Esf[\beta])^{\alpha}} \Big|  \nonumber\\
 & \leq &  \frac{\Esf[\beta^{\alpha}]|\mu^{\alpha}- \Esf^{\alpha}[\beta] |}{|\alpha(\alpha-1)|(\mu \Esf[\beta])^{\alpha}} 
            + \frac{|\Esf[\beta^{\alpha}] - \frac{1}{n} \sum_i b_i^{\alpha}|}{|\alpha(\alpha-1)|\mu^{\alpha}}. \label{eqn:diff_Ialpha}
\ee
Recall that  $c-a \leq m \leq M \leq c+a$ with $M = \max(\mu, \Esf[\beta])$ and $m= \min(\mu, \Esf[\beta])$. 
By Mean Value Theorem,  
there exists $z_0 \in (m, M)$ such that
\beqa
\Big| {\mu^{\alpha}- \Esf^{\alpha}[\beta]}\Big|  
        &=& \alpha |\Esf[\beta] - \mu| \cdot | z_0^{\alpha-1}|   \\
        &\leq &   \left\{ \begin{array}{lll}
              \alpha { m^{\alpha-1}} |\Esf[\beta] - \mu| 
                     & \mbox{for} & 0 < \alpha  <  1, \\
              \alpha M^{\alpha-1}  |\Esf[\beta] - \mu|  
                    & \mbox{for} & \alpha > 1. 
               \end{array} \right.
\eeqa 
The last inequality holds since  
$f(x) =x^{\alpha-1}$ is  decreasing (increasing)  with positive $x \geq 0$  
for $0< \alpha < 1$ (for $\alpha > 1$, respectively).
Since
$0< \frac{\Esf[\beta^{\alpha}]}{(\mu \Esf[\beta])^{\alpha}} \leq \frac{(c+a)^{\alpha}}{m^{\alpha}M^{\alpha}}$, 
we have 
\be 
\frac{\Esf[\beta^{\alpha}]}{(\mu \Esf[\beta])^{\alpha}} \Big|\mu^{\alpha}- \Esf^{\alpha}[\beta] \Big|
  \leq   \frac{\alpha}{c-a} \Big(\frac{c+a}{c-a}\Big)^{\alpha} |\Esf[\beta] - \mu|  
\label{eqn:Ekbeta}    
\ee 
Applying  (\ref{eqn:Ekbeta}) and   Lemma \ref{lemma:McDiarmid-b} to (\ref{eqn:diff_Ialpha}), 
  it holds at least with probability $1-\delta$ that 
\beqa 
\Big| I_{\alpha}(h, P) - I_{\alpha}(\ub;n) \Big| 
& \leq & \frac{2\alpha\frac{a}{c-a} \Big(\frac{c+a}{c-a}\Big)^{\alpha}+\frac{\Delta_{\alpha}}{\mu^{\alpha}}}{|\alpha(\alpha-1)|} \sqrt{\frac{1}{2n}\ln\frac{4}{\delta}} \\ 
&\leq & \frac{\Big(1+\frac{2\alpha}{r-1}\Big)\Big(1+\frac{2}{r-1}\Big)^{\alpha}-1}{|\alpha(\alpha-1)|} \sqrt{\frac{1}{2n}\ln\frac{4}{\delta}} \\
& & \quad (\because  r=\frac{c}{a}, ~ \frac{\Delta_{\alpha}}{\mu^{\alpha}} \leq \big(\frac{c+a}{c-a}\big)^{\alpha}-1 ).  
\eeqa

\section{Proof of Theorem \ref{thm:simple-gen-fair-constraint-tight}} \label{appendix:proof-simple}
Theorem  \ref{thm:simple-gen-fair-constraint-tight} is a simplified version of 
Theorem \ref{thm:gen-fair-constraint-tight} below.
\begin{theorem} \label{thm:gen-fair-constraint-tight}
For any distribution $P$ over $\Xcal \times \Ycal$, let $S =\{(\bx_i, y_i)\}_{i=1}^n$ 
be a sample data set  identically and independently  drawn according to 
$P$. 
For any $0 < \delta < 1$,
with probability  at least $1-\delta$,    
it holds that for each $h \in \mathcal{H}$ and  $\alpha \in [0, \infty)$,  
\beqa 
\Big |I_{\alpha}(h, P) - I_{\alpha}(\ub_h;n)  \Big|  & \leq &  \tilde{\psi}_{\alpha}(r, \varepsilon_2) \cdot \varepsilon_2
\eeqa
where $\varepsilon_2 =  \sqrt{ \frac{   8d_{\Hcal}\ln \big(\frac{2en}{d_{\Hcal}}\big) +8\ln \frac{8}{\delta}}{n}}$,  
\beqa
\tilde{\psi}_{\alpha}  =
   \left \{ \begin{array}{ll}
     \frac{1}{ r-R_S(h)-  \varepsilon_2} + \ln\Big( \frac{r}{r-1}\Big)  &\mbox{for $\alpha=0$}, \\
     \frac{1}{r-R_S(h)-\varepsilon_2}\Big[1+  \frac{ r\big(1+  2\ln(ar+a)\big)}{r-R_S(h)}\Big] &\mbox{for $\alpha=1$} ,\\
     \frac{U_{\alpha}}{|\alpha-1|}\Big( \frac{r}{r-R_S(h)}\Big)^{\alpha}  
             & \mbox{otherwise},
\end{array} \right.     
\eeqa
and
\beqa
U_{\alpha} = \left \{ \begin{array}{l}
     \frac{1}{r}\Big(\frac{r+1}{r}\Big)^{\alpha} 
        + \frac{\big(\frac{r-R_S(h)}{r-R_S(h)-\varepsilon_2}\big)^{\alpha} 
       }{r-R_S(h)}\Big[ 1+ \Big( \big( \frac{r+1}{r}\big)^{\alpha}-1 \Big)R_S(h) 
     + \varepsilon_2\Big(\big(\frac{r+1}{r}\big)^{\alpha}+1\Big)\Big]
      \quad \mbox{for }   \alpha >1,\\
     \frac{1}{r} +\frac{ 1+ \Big( \big( \frac{r+1}{r}\big)^{\alpha}-1 \Big)R_S(h) 
     + \varepsilon_2\Big(\big(\frac{r+1}{r}\big)^{\alpha}+1\Big)}{r-R_S(h)-\varepsilon_2}
            \qquad \qquad \qquad ~~ \mbox{for $0 <\alpha<1$}.          
\end{array} \right.     
\eeqa 
\end{theorem}

Suppose that Theorem \ref{thm:gen-fair-constraint-tight} holds. 
Note that the cases of $\alpha=0,~1$ of Theorem  \ref{thm:simple-gen-fair-constraint-tight} are 
identical with those of Theorem  \ref{thm:gen-fair-constraint-tight}. 
We can derive $\tilde{\psi}_{\alpha} $ for $\alpha=2$ of Theorem  \ref{thm:simple-gen-fair-constraint-tight} 
with the assumptions of $R_S(h)< \frac{1}{2}$ and sufficiently large $n$ 
such that $\varepsilon_2 \leq \frac{1}{5}$ 
as follows. 
For $\alpha =2$, $U_{\alpha}$ becomes 
\beqa
U_{\alpha} 
= \frac{1}{r} \Big(\frac{r+1}{r} \Big)^{2} 
   + \frac{1}{r-R_S(h)}\Big(\frac{r-R_S(h)}{r-R_S(h)-\varepsilon} \Big)^{2} C_{\tilde{\psi}}
\eeqa
where 
$C_{\tilde{\psi}} =   1+ \Big( \big( \frac{r+1}{r}\big)^{2}-1 \Big)R_S(h) 
     + \varepsilon_2\Big(\big(\frac{r+1}{r}\big)^{2}+1\Big)$. 

Now we fix $\alpha =2$ and assume that $n$ is sufficiently large that $\varepsilon_2 < \frac{1}{5}$. 
It is obvious that 
 $\Big(\frac{r+1}{r} \Big)^{2} =  1+\frac{2}{r}+\frac{1}{r^2} < 1+\frac{3}{r}$ 
since $r \geq 1$. 
Similarly, $\Big(\frac{r-R_S(h)}{r-R_S(h)-\varepsilon} \Big)^{2} = \Big(1+\frac{\varepsilon_2}{r-R_S(h)-\varepsilon}\Big)^2< 1+\frac{3\varepsilon_2}{r-R_S(h)-\varepsilon}$ 
since $\frac{\varepsilon_2}{r-R_S(h)-\varepsilon_2} \geq \Big(\frac{\varepsilon_2}{r-R_S(h)-\varepsilon_2} \Big)^2$ when $R_S(h)< \frac{1}{2}$, $\varepsilon_2 < \frac{1}{5}$ and $r \geq 1$. 

Consider  $C_{\tilde{\psi}}$; when $\alpha=2$, 
\beqa 
C_{\tilde{\psi}} 
&= &  1+R_S(h)\Big(\big(1+\frac{1}{r}\big)^{2}-1 \Big)+ \varepsilon_2\Big( \big(1+\frac{1}{r} \big)^2+1 \Big) \\
&\leq & 1+\frac{3}{r} R_S(h) + \varepsilon_2 \Big(2+\frac{3}{r}\Big)  \\
& \leq & 2+\frac{3}{r} R_S(h) \qquad (\mbox{since $2+\frac{3}{r} \leq 5$ and $\varepsilon_2 \leq \frac{1}{5}$)  }  
\eeqa 
Then $U_{\alpha}$ becomes 
\be 
U_{\alpha} 
&\leq & \frac{1}{r} \Big(1+\frac{3}{r}\Big) + \frac{1}{r-R_S(h)} \Big( 1+ \frac{3\varepsilon_2}{r-R_S(h)-\varepsilon_2} \Big)
   \Big(2+\frac{3}{r} R_S(h) \Big). \label{eqn:pre_U_alpha}
\ee 
Consider $ \Big( 1+ \frac{3\varepsilon_2}{r-R_S(h)-\varepsilon_2} \Big)
   \Big(2+\frac{3}{r} R_S(h)\Big)$; 
\beqa
\lefteqn{   \Big( 1+ \frac{3\varepsilon_2}{r-R_S(h)-\varepsilon_2} \Big)
   \Big(2+\frac{3}{r} R_S(h)\Big) }\\
&\leq &  
     2+\frac{3}{r}R_S(h)+  \frac{3}{r-R_S(h)-\varepsilon_2} \cdot \varepsilon_2 \Big(2+ \frac{3}{r}R_S(h) \Big)  \\
& \leq & 2+\frac{3}{r}R_S(h)+  \frac{3}{r-R_S(h)-\varepsilon_2} 
   \qquad  (\mbox{since $\varepsilon_2 \big(2 +\frac{3R_S(h)}{r}\big) < \varepsilon_2\big(2+\frac{3}{r}\big) < 1$}).
\eeqa
Hence (\ref{eqn:pre_U_alpha}) becomes
\be 
U_{\alpha} \leq \frac{1}{r}\Big(1+\frac{3}{r}\Big) + \frac{1}{r-R_S(h)}\Big( 2+ \frac{3R_S(h)}{r} + \frac{3}{r-R_S(h)-\varepsilon_2}\Big) \label{eqn:U_alpha}
\ee 
Using (\ref{eqn:U_alpha}), we have 
\beqa
  \tilde{\psi}_{\alpha}\big |_{\alpha=2} 
  &=& \frac{U_{\alpha}}{\alpha-1} \Big(\frac{r}{r-R_S(h)}\Big)^{\alpha} \\
  & \leq &   \Big(\frac{r}{r-R_S(h)}\Big)^2 \Big(\frac{1}{r} +\frac{3}{r^2}  
     + \frac{1}{r-R_S(h)}\Big( 2+\frac{3R_S(h)}{r}+ \frac{3}{r-R_S(h)-\varepsilon_2}\Big)  \Big)\\
  & \leq &  \Big(\frac{r}{r-R_S(h)}\Big)^2 \Big(\frac{1}{r} +\frac{3}{r^2}  
     + \frac{1}{r-R_S(h)}\Big(12+\frac{3R_S(h)}{2r}\Big)\Big).    
\eeqa
The last inequality holds since $r-R_S(h)-\varepsilon_2 > r-\frac{7}{10} \geq \frac{3}{10}$. 
We have proved Theorem \ref{thm:simple-gen-fair-constraint-tight}. 
 
Now we  prove Theorem \ref{thm:gen-fair-constraint-tight}. 
\begin{lemma}(See Problem 23 on page 101 of  \cite{Rudin76}) 
\label{lemma:convex-prop}
Suppose that $f(x)$ is a convex function over $(a_0,b_0)$ and 
$ a_0 < s< t< u< b_0$. 
Then, it holds that 
\beqa
\frac{f(t)- f(s)}{t-s} \leq \frac{f(u)-f(s)}{u-s} \leq \frac{f(u)-f(t)}{u-t}. 
\eeqa
\end{lemma} 

We denote by $n_{FP}$  the number of false positive labels and 
by $n_{FN}$ the number of false negative labels and 
let 
\beqa
 \qFP = \frac{\nFP}{n}, & & \qFN = \frac{\nFN}{n}.
\eeqa 
The empirical average $\mu$ becomes as below; 
\beqa
\mu ~=~ \frac{1}{n} \sum_{i=1}^n b_i ~=~ c+ a\frac{n_{FP}-n_{FN}}{n} 
~=~ c+a(\qFP -\qFN).
\eeqa 
Let 
\beqa
m_C &=& \Psf_{\bx  \sim P_x } \big[~\bx: ~h(\bx) = y ~\big],  \\
m_{FP} &=& \Psf_{\bx  \sim P_x } \big[~\bx: ~h(\bx) = 1, ~y=0 ~\big],\\ 
m_{FN} &=& \Psf_{\bx  \sim P_x } \big[~\bx: ~ h(\bx) = 0, ~y=1 ~\big]
\eeqa 
that is, $m_C$ is the probability (measure)  of the set of 
individuals  $\bx$ with  correct labels, 
$m_{FP}$ is the probability (measure) of the set of $\bx$  
with false positive labels and 
$m_{FN}$ the the probability (measure) of the set of $\bx$ 
with  false negative labels. Using these quantities, 
we express $\Esf[b(\ux)]$ as below; 
\beqa
\Esf[b(\ux)]&=& c+a(\mFP-\mFN) .
\eeqa
Let  $\varepsilon_2 = \sqrt{ \frac{   8d_{\Hcal}\ln \big(\frac{2en}{d_{\Hcal}}\big) +8\ln \frac{8}{\delta}}{n}}$ 
and
\beqa
 S_p = \qFP + \qFN, & & S_m = \qFP- \qFN,\\
 T_p = \mFP + \mFN, & & T_m = \mFP - \mFN.  
\eeqa 
\begin{lemma}\label{lemma:together} 
With probability at least $1-\delta$, the  following two inequalities  
`` simultaneously"  hold
\be 
|T_m-S_m| \leq  \varepsilon_2, \qquad |T_p-S_p| \leq \varepsilon_2. \label{eqn:together} 
\ee 
\end{lemma}
\begin{proof}
From the definitions of $S_p, ~S_m, ~T_p$, and $T_m$, the followings are hold;
\be 
 \qFP = \frac{S_p+S_m}{2}, & & \qFN = \frac{S_p-S_m}{2}, \label{eqn:q}\\
 \mFP = \frac{T_p+T_m}{2}, & & \mFN = \frac{T_p-T_m}{2}, \label{eqn:m}.
\ee

Let  $\varepsilon_1 = \sqrt{\frac{2}{n} \ln \frac{4}{\delta}}$ and 
$\varepsilon_2= \sqrt{\frac{8d_{\Hcal}\ln\big(\frac{2en}{d_{\Hcal}}+ 8\ln\frac{8}{\delta}\big)}{n}}$.
From Lemma (\ref{lemma:McDiarmid-b}), with probability at least $1-\frac{\delta}{2}$, 
it holds that $  \Big| \Esf[b(\ux)] -\mu \Big| ~\leq~ a\varepsilon_1$. 
By Theorem \ref{thm:gen-error}, with probability at least 
$1-\frac{\delta}{2}$, we have $|R_{\Xcal}(h) -R_S(h)| \leq \varepsilon_2.$ 
By using the union bound for the above two inequalities and the fact that 
$ \varepsilon_1 < \varepsilon_2 $,   
the   two inequalities  `` simultaneously"  hold 
\be  
 \Big|\Esf[b(\ux)]  - \mu\Big| & \leq & a\varepsilon_2,  \label{eqn:Eb-mu}  \\
 |R_{\Xcal}(h) -R_S(h)|  & \leq &  \varepsilon_2. \label{eqn:RX-RS}
\ee
Lemma \ref{lemma:together} follows from the observations;
\beqa
  \mFP-\mFN-(\qFP-\qFN) &=& T_m-S_m,\\
  \mFP+\mFN-(\qFP+\qFN) &=& T_P-S_p.
\eeqa  
and 
\be 
|\Esf[b(\ux)]-\mu| &=&  a|\mFP-\mFN-(\qFP-\qFN)| ~=~ a|T_m-S_m|, \label{eqn:FP-FN}\\
|R_{\Xcal}(h)- R_S(h)| &=&   |\mFP+\mFN-(\qFP+\qFN)| ~=~ |T_p-S_p|.  \label{eqn:FP+FN} 
\ee 

We have proved Lemma \ref{lemma:together}
\end{proof}

\subsection{When $\alpha =0$}
Consider $I_0(h,P) - I_0(\ub;n)$. 
From (\ref{eqn:pre_diff_I0}), we have 
\beqa
\Big|I_0(h,P) - I_0(\ub;n)\Big| \leq 
      \Big| \ln \Esf[\beta] - \ln \mu \Big| 
      + \Big|\Esf[\ln \beta] -\frac{1}{n} \sum_i \ln b_i \Big| 
\eeqa
We find each upper bound,  
which is expressed in terms of $|T_m-S_m|$ and $|T_p-S_p|$, 
of $\Big| \ln \Esf[\beta] - \ln \mu \Big|$ and 
$\Big|\Esf[\ln \beta] -\frac{1}{n} \sum_i \ln b_i \Big|$.
\begin{itemize}
\item Upper bound of $\Big| \ln \Esf[\beta] - \ln \mu \Big|$:\\
Let $\varepsilon_0 = |\Esf[\beta]-\mu|$. 
By (\ref{eqn:before_diff_lnEb}), it holds that  
\be
|\ln\Esf[\beta] - \ln \mu | 
~<~ \frac{|\Esf[\beta]- \mu| }{m} 
~=~ \frac{\varepsilon_0 }{m}  \label{eqn:pre-ln}
\ee 
where $m = \min(\Esf[\beta], \mu)$ and $M = \max(\Esf[\beta], \mu)$. 
If $m = \mu$, then  $m= \mu = c+a(\qFP-\qFN) > c-aR_S(h)$.  
If $m= \Esf[\beta]$, then $m= \mu-\varepsilon_0$; 
summarizing, it holds that 
\be
m = \min(\Esf[\beta],\mu) 
 &\geq&    \left\{ \begin{array}{ll}
         c-aR_S(h) & \mbox{if }   \mu < \Esf[\beta]\\
         c-aR_S(h) -\varepsilon_0 & \mbox{otherwise}.  
         \end{array}   \right. \label{eqn:mM}
\ee
In either case, we have $m > c-aR_S(h)- \varepsilon_0$. 
By (\ref{eqn:Eb-mu}),  
 (\ref{eqn:pre-ln}) becomes  
\be 
|\ln\Esf[\beta] - \ln \mu |
~<~  \frac{\varepsilon_0 }{c-aR_S(h) - \varepsilon_0}  
~=~ \frac{a|T_m-S_m|}{c-aR_S(h)-a|T_m-S_m|}. \label{eqn:new_b}
\ee 

\item Upper bound of $\Big|\Esf[\ln \beta] -\frac{1}{n} \sum_i \ln b_i \Big|$:\\
Using (\ref{eqn:q}) and  (\ref{eqn:m}), we can write 
$\Big| \frac{1}{n} \sum_i \ln b_i - \Esf[\ln \beta] \Big|$ 
as follows;
\beqa
\lefteqn{\Big| \frac{1}{n} \sum_i \ln b_i - \Esf[\ln \beta] \Big|} \\
& =& \Big|\Big(\mFP+\mFN- (\qFP+\qFN) \Big) \ln c +
         (\qFP-\mFP) \ln(c+a) + (\qFN-\mFN) \ln(c-a) \Big| \\
&=& \Big|(T_p-S_p) \ln c + \Big(\frac{S_p+S_m}{2}-\frac{T_p + T_m}{2} \Big) \ln(c+a) + \Big(\frac{S_p-S_m}{2}-\frac{T_p - T_m}{2} \Big) \ln(c-a)\Big| \\
&\leq & \Big| \frac{\ln(c+a)+\ln(c-a)-2\ln c}{2}  \cdot(S_p-T_p) 
       + \frac{\ln(c+a)-\ln(c-a)}{2}\cdot (S_m-T_m)\Big| \\
&\leq & \Big| \frac{\ln(c+a)+\ln(c-a)-2\ln c}{2}\Big| \cdot |S_p-T_p| 
       + \Big|\frac{\ln(c+a)-\ln(c-a)}{2}\Big| \cdot |S_m-T_m|. 
\eeqa 
Since  $\ln(c+a)+\ln(c-a) -2\ln c = \ln(\frac{c^2-a^2}{c^2}) < 0$, 
we have  $\Big| \frac{\ln(c+a)+\ln(c-a)-2\ln c}{2}\Big| =\frac{2 \ln c -\ln(c+a)-\ln(c-a)}{2}.$

Therefore, if we let $c_p =\frac{2\ln c-\ln(c+a)-\ln(c-a)}{2}$ 
and $c_m= \frac{\ln(c+a)-\ln(c-a)}{2}$, 
it holds that 
\be 
 \Big|\Esf[\ln \beta] -\frac{1}{n} \sum_i \ln b_i \Big| 
&\leq & c_p |S_p-T_p| + c_m |S_m-T_m|.  \label{eqn:new_b2}  
\ee 
\end{itemize}
Combining (\ref{eqn:new_b}),  (\ref{eqn:new_b2}), and  Lemma \ref{lemma:together}, 
with probability at least $1-\delta$,  we have 
\beqa
\lefteqn{|I_0(h,P) - I_0(\ub;n)| }\\
&\leq & \frac{a|T_m-S_m|}{c-aR_S(h)-a|T_m-S_m|} + c_p|S_p-T_p| +c_m|S_m-T_m| \\
&\leq&  \varepsilon_2 \Big[\frac{a}{c-aR_S(h) - a\varepsilon_2}  
      + \ln \Big(\frac{c}{c-a}\Big) \Big] \qquad(\mbox{by Lemma \ref{lemma:together})}\\
&=&  \varepsilon_2 \Big[\frac{1}{ r-R_S(h)-  \varepsilon_2}
      + \ln\Big(1+\frac{1}{r-1}\Big) \Big]
\qquad (\mbox{because }  r=\frac{c}{a}). 
\eeqa
Note that there is no need to take union bound when combining  
(\ref{eqn:new_b}) and (\ref{eqn:new_b2}), 
because the two inequalities  in Lemma \ref{lemma:together} simultaneously hold.

\subsection{When $\alpha =1$}
Let $\omega = \Esf[b(\ux)]$. 
From (\ref{eqn:I_1(h,P)}) and (\ref{eqn:I_1(b;n)}),     $I_1(h,P)-I_1(\ub;n)$ is written as  
\beqa
 I_1(h,P) - I_1(\ub;n)  
&=& \ln \mu - \ln \omega + c\ln c \Big(\frac{1}{\omega}-\frac{1}{\mu} + F_1 \Big)  
      + (c+a) \ln(c+a) F_2 + (c-a)\ln(c-a)F_3.
\eeqa
where   
\beqa
F_1 &=&\frac{\qFP+\qFN}{\mu} - \frac{\mFP+\mFN}{\omega}, \\
F_2 &=& \frac{\mFP}{\omega} - \frac{\qFP}{\mu}, \\
F_3 &=& \frac{\mFN}{\omega} - \frac{\qFN}{\mu}.
\eeqa 
After rewriting $F_1, ~F_2,$ and $F_3$ as follows,
\beqa
F_1 &=& -\Big(\frac{1}{\omega}-\frac{1}{\mu}\Big)(\qFP+\qFN) + \frac{1}{\omega}\Big[\qFP+\qFN-(\mFP+\mFN) \Big]\\
F_2 &=& \Big(\frac{1}{\omega} - \frac{1}{\mu}\Big) \qFP +\frac{\mFP-\qFP}{\omega},\\
F_3 &=& \Big(\frac{1}{\omega} - \frac{1}{\mu}\Big)\qFN + \frac{\mFN-\qFN}{\omega} ,
\eeqa 
we have  
\beqa   
I_1(h,P) - I_1(\ub;n) 
&=& \ln \mu - \ln \omega +  \Big(\frac{1}{\omega}-\frac{1}{\mu}\Big) H_1 + \frac{1}{\omega} H_2  
\eeqa
where 
\beqa
H_1 &=& c\ln c\big(1-(\qFP+\qFN) \big)+(c+a)\ln(c+a)\qFP+(c-a)\ln(c-a) \qFN, \\
H_2 &=& c\ln c(\qFP+\qFN-(\mFP+\mFN)) \\
    & & \quad +(c+a)\ln(c+a)(\mFP-\qFP)+ (c-a)\ln(c-a)(\mFN-\qFN).
\eeqa 
Therefore, we have 
\beqa
| I_1(h,P) - I_1(\ub;n)| 
&\leq & |\ln \mu -\ln \omega| +  \Big|\frac{1}{\omega}-\frac{1}{\mu} \Big| \cdot\Big|H_1 \Big| + \frac{1}{\omega} \cdot\Big|H_2\Big|. 
\eeqa
We find each upper bound  of $ |\ln \mu -\ln \omega|, ~  \Big|\frac{1}{\omega}-\frac{1}{\mu} \Big|, ~| H_1 |$ and $| H_2|$.  
\begin{itemize}
\item An upper bound of $ |\ln \mu -\ln \omega|$\\
By   (\ref{eqn:new_b}), we have  
\beqa 
|\ln \mu -\ln \omega| & \leq &   \frac{a|T_m-S_m|}{c-aR_S(h)-a|T_m-S-M|}.
\eeqa 
 
\item An upper bound of $\Big|\frac{1}{\mu} -\frac{1}{\omega} \Big|$\\
By Lemma \ref{lemma:together}, with probability at least $1-\delta$, 
it holds that 
\be 
\Big| \frac{1}{\mu} -\frac{1}{\omega} \Big| 
&=& \frac{|\mu-\omega|}{\mu \omega} ~\leq~ \frac{a|T_m-S_m|}{mM}.  \label{eqn:pre-inverse-mu}
\ee  
where  $m = \min(\mu, \Esf[\beta])$ and $M=\max(\mu, \Esf[\beta])$. 
By (\ref{eqn:mM}), $m > c-aR_S(h) -|\Esf[\beta]-\mu|=c-aR_S(h)-a|T_m-S_m|$. 
Note that $M \geq \mu  \geq c-aR_S(h)$.  
Therefore (\ref{eqn:pre-inverse-mu}) becomes
\be
\Big| \frac{1}{\mu} -\frac{1}{\omega} \Big|&\leq&    \frac{a|T_m-S_m|}{(c-aR_S(h))(c-aR_S(h) - a|T_m-S_m|)}. \label{eqn:inverse-mu}
\ee  
\item An upper bound of $H_1$\\

Let $\gamma_p = \frac{(c+a) \ln(c+a) + (c-a) \ln(c-a)}{2}$ and 
$\gamma_m=\frac{(c+a) \ln(c+a) - (c-a) \ln(c-a) }{2}$. 
Using (\ref{eqn:q}) and (\ref{eqn:m}), we compute an upper bound of  $|H_1|$   as below;
\be 
|H_1| &=& |(1-S_p)c\ln c +\gamma_p  S_p + \gamma_m S_m| \nonumber \\
      &=& |c\ln c +(\gamma_p-c\ln c)S_p +  \gamma_m S_m| \nonumber\\
      &=&  | c\ln c +  (\gamma_p-c\ln c) S_p +  \gamma_m S_m | \nonumber\\
      &\leq & c\ln c +  |\gamma_p-c\ln c|\cdot |S_p| + |\gamma_m|\cdot|S_m|. \label{eqn:pre_pre_H_1}
\ee 
We will show that $\gamma_p-c\ln c \geq 0$. 
For $\gamma_p-c\ln c \geq 0$, it is enough to show that 
\be 
 \frac{(c+a)\ln(c+a)-c\ln c}{a} ~\geq ~ \frac{c\ln c -(c-a)\ln(c-a)}{a}. \label{eqn:gamma_p-clnc}
\ee 
It can be checked that (\ref{eqn:gamma_p-clnc}) holds  
by applying  Lemma \ref{lemma:convex-prop} for $f(x)=x\ln x$ since $f(x)=x\ln x$ is a convex function over $x\geq \frac{1}{e}$.   
It obviously holds that  $\gamma_m \geq 0$. 
Hence (\ref{eqn:pre_pre_H_1}) becomes 
\be
|H_1| &\leq & c\ln c +  (\gamma_p-c\ln c) |S_p| + \gamma_m |S_m| \nonumber\\
      &\leq & c\ln c+ (\gamma_p-c\ln c) R_S(h) + \gamma_m R_S(h) 
          \qquad (\mbox{because }  S_p=R_S(h), |S_m|< S_p) \nonumber\\
          &\leq & c\ln c+ (\gamma_p+\gamma_m -c\ln c) R_S(h). \label{eqn:H_1}
\ee

\item An upper bound  of $\frac{1}{\omega}$: \\
Recall that $\omega =\Esf[b(\ux)]$ and  $m= \min[\mu, \Esf[b(\ux)]]$. 
\beqa
   \frac{1}{\omega} < \frac{1}{m} < \frac{1}{c-aR_S(h)-a|T_m-S_m|}. \label{eqn:overomega}
\eeqa

\item An upper bound of $|H_2|$\\ 
Using (\ref{eqn:q}) and (\ref{eqn:m}), we rewrite   $H_2$ as below 
\beqa
H_2 &=& c\ln c(S_p-T_p) + \gamma_p (T_p-S_p) + \gamma_m (T_m-S_m).
\eeqa
Recall that $\gamma_p = \frac{(c+a) \ln(c+a) + (c-a) \ln(c-a)}{2}$ and 
$\gamma_m=\frac{(c+a) \ln(c+a) - (c-a) \ln(c-a) }{2}$. 
Similarly as we have done in finding an upper bound of $|H_1|$, with probability at least $1-\delta$, 
we have 
\be 
 |H_2| &=& \Big| c\ln c(S_p-T_p) + \gamma_p (T_p-S_p) + \gamma_m (T_m-S_m)\Big| \nonumber\\
       &=& \Big|(\gamma_p - c\ln c)(T_p-S_p)+\gamma_m(S_m-T_m) \Big| \nonumber\\
       &\leq & (\gamma_p-c\ln c)|T_p-S_p| + \gamma_m|S_m-T_m| 
          \qquad (\mbox{because }  \gamma_p-c\ln c \geq 0, ~\gamma_m \geq0 ). \label{eqn:H_2} 
\ee 
\end{itemize}

By (\ref{eqn:new_b}), (\ref{eqn:inverse-mu}), (\ref{eqn:H_1}), (\ref{eqn:overomega}),and (\ref{eqn:H_2}), 
an upper bound of $|I_1(h, P) - I_1(\ub;n)|$ is
\beqa
 |I_1(h, P) - I_1(\ub;n)|  &\leq & S_1 + S_2 + S_3 
\eeqa
where 
\beqa
S_1 &=& \frac{a|T_m-S_m|}{c-aR_S(h)-a|T_m-S_m|},\\
S_2 &=& \frac{a|T_m-S_m|\big( c\ln c+(\gamma_p+\gamma_m-c\ln c) R_S(h)\Big)}{(c-aR_S(h))(c-aR_S(h)-a|T_m-S_m|}, \\
S_3 &=& \frac{(\gamma_p-c\ln c) |T_p-S_p|+\gamma_m|T_m-S_m|}{c-aR_S(h)-a|T_m-S_m|}. 
\eeqa
Applying Lemma \ref{lemma:together} to $S_1+S_2+S_3$, we have
\beqa
S_1 + S_2 +S_3  
    &\leq & \frac{a\varepsilon_2}{c-aR_S(h)- a \varepsilon_2}
            \Big[1+\frac{c\ln c + \big(\gamma_p+\gamma_m-c\ln c)R_S(h) }{c-aR_S(h)} 
            + \frac{\gamma_p+\gamma_m-c\ln c}{a}\Big]  \\
  &\leq & \frac{a\varepsilon_2}{c-aR_S(h)- a \varepsilon_2}
      \Big[1+ \frac{c\ln c}{c-aR_S(h)}  +\Big(\frac{R_S(h)}{c-aR_S(h)}+\frac{1}{a}\Big) (\gamma_p+\gamma_m-c\ln c )\Big]\\
  & = &  \frac{\varepsilon_2}{r-R_S(h)-\varepsilon_2}
      \Big[ 1+\frac{r\ln(ar)}{r-R_S(h)} 
           + \Big(\frac{R_S(h)}{r- R_s(h)}+1\Big)\Big(r\ln\big(1+\frac{1}{r}\big) + \ln\big(a(r+1)\big)\Big) \Big]\\
   & & \qquad (\mbox{since $r=\frac{c}{a}$ and $\gamma_p+\gamma_m-c\ln c= c\ln\Big(\frac{c+a}{c}\Big)+a\ln(c+a)$}).
\eeqa
Note that 
\be 
\lefteqn{1+\frac{r\ln(ar)}{r-R_S(h)} 
           + \Big(\frac{R_S(h)}{r- R_s(h)}+1\Big)\Big(r\ln\big(1+\frac{1}{r}\big) + \ln\big(a(r+1)\big)\Big)} \nonumber\\
& = & 1+  \frac{r\big( 2\ln a + \ln r +\ln(r+1) +r\ln\big(1+\frac{1}{r}\big)\Big)}{r-R_S(h)}\nonumber\\
&=& 1+  \frac{ r\big(1+ 2\ln a  + 2\ln(r+1)\big)}{r-R_S(h)} 
   \qquad (\mbox{since $\ln(1+x) \leq x$}). \label{eqn:simple}
\ee 
Using (\ref{eqn:simple}), we have 
\beqa
S_1+S_2+S_3 
&\leq & \frac{\varepsilon_2}{r-R_S(h)-\varepsilon_2}
     \Big[1+  \frac{ r}{r-R_S(h)}\big(1+ 2\ln a  + 2\ln(r+1)\big)\Big].
\eeqa
Recalling $\varepsilon_2= \sqrt{\frac{8d_{\Hcal}\ln\big(\frac{2en}{d_{\Hcal}}+ 8\ln\frac{8}{\delta}\big)}{n}}$, 
we have proved the case of $\alpha=1$ of Theorem \ref{thm:gen-fair-constraint-tight}.

\subsection{When $  0 < \alpha <1$}
After simple algebra, we have  
\beqa
I_{\alpha}(h, P)  - I_{\alpha}(\ub;n) 
&=& \frac{1}{\alpha(\alpha-1)}
   \Big(\frac{J_n}{J^{\alpha}} -\frac{K_n}{K^{\alpha}}\Big).
\eeqa 
where  $\theta=\frac{a}{c} =\frac{1}{r}$ and 
\beqa
 J = 1+\theta (\mFP-\mFN), & & 
J_n = 1-(m_{FP}+m_{FN}) + (1+\theta)^{\alpha}m_{FP}+(1-\theta)^{\alpha}m_{FN},\\
 K = 1+\theta (\qFP-\qFN), & & K_n = 1-(q_{FP}+q_{FN})+(1+\theta)^{\alpha}q_{FP}  +(1-\theta)^{\alpha}q_{FN}.    
\eeqa
Note that all of $J, ~J_n, K$, and $K_n$ are positive 
since $\theta= \frac{a}{c}$ with $c > a >0$ and $c-a\geq 1$, $ 0 \leq \mFP+\mFN \leq 1$, and $ 0 \leq \qFP+\qFN \leq 1$.  
Moreover, 
\beqa
\frac{J_n}{J^{\alpha}} -\frac{K_n}{K^{\alpha}} &=&   T_1 + T_2 + T_3 + T_4 
\eeqa
where   
\beqa
T_1 &=& \frac{1}{J^{\alpha}} - \frac{1}{K^{\alpha}}, \\ 
T_2 &=& \frac{\qFP+\qFN}{K^{\alpha}} -\frac{\mFP+\mFN}{J^{\alpha}},  \\
T_3 &=& (1+\theta)^{\alpha} \Big[\frac{\mFP}{J^{\alpha}} - \frac{\qFP}{K^{\alpha}}\Big],\\
T_4 &=& (1-\theta)^{\alpha} \Big[\frac{\mFN}{J^{\alpha}} - \frac{\qFN}{K^{\alpha}}\Big] .  
\eeqa
After rewriting  $T_2$, $T_3$, and $T_4$ as follows, 
\beqa
T_2 &=& \frac{\qFP+\qFN-(\mFP+\mFN)}{K^{\alpha}}- \Big(\frac{1}{J^{\alpha}} - \frac{1}{K^{\alpha}} \Big)(\mFP+\mFN), \\
T_3 &=& (1+\theta)^{\alpha}\Big(\frac{1}{J^{\alpha}} - \frac{1}{K^{\alpha}} \Big)\mFP + (1+\theta)^{\alpha}\frac{\mFP-\qFP}{K^{\alpha}} \\
T_4 &=& (1-\theta)^{\alpha}\Big(\frac{1}{J^{\alpha}} - \frac{1}{K^{\alpha}} \Big)\mFN + (1-\theta)^{\alpha}\frac{\mFN-\qFN}{K^{\alpha}}, 
\eeqa 
we have 
\beqa 
\frac{J_n}{J^{\alpha}} - \frac{K_n}{K^{\alpha}} 
&=& \Big(\frac{1}{J^{\alpha}} - \frac{1}{K^{\alpha}} \Big) (Z_1+Z_2) + \frac{1}{K^{\alpha}} (Z_3+Z_4) 
\eeqa
where 
\beqa
Z_1 &=& 1-(\mFP+\mFN) \\
Z_2 &=& (1+\theta)^{\alpha}\mFP + (1-\theta)^{\alpha} \mFN, \\
Z_3 &=& \qFP+\qFN -(\mFP+\mFN),\\
Z_4 &=& (1+\theta)^{\alpha}(\mFP-\qFP)+(1-\theta)^{\alpha}(\mFN-\qFN). 
\eeqa
Therefore, it holds that 
\beqa
| I_{\alpha}(h, P)  - I_{\alpha}(\ub;n)|   
&\leq &  \Big|\frac{1}{J^{\alpha}} - \frac{1}{K^{\alpha}}\Big| \cdot \frac{ |Z_1| + |Z_2|}{|\alpha(\alpha-1)|}   ~ + ~
      \frac{1}{|K^{\alpha}|} \cdot \Big|\frac{Z_3+Z_4}{\alpha(\alpha-1)}\Big|.  
\eeqa

We will find  each  upper bound  of 
$\Big|\frac{1}{J^{\alpha}} - \frac{1}{K^{\alpha}}\Big| , ~|Z_1|, ~|Z_2|, ~  \frac{1}{|K^{\alpha}|}$, 
and $\big|\frac{Z_3+Z_4}{\alpha(1-\alpha)}\big|$.   
Recall that $K = 1+\theta(\qFP-\qFN)$. 
Obviously $1-\theta R_S(h) \leq K \leq 1+\theta R_S(h)$. Hence 
\be 
\frac{1}{|K^{\alpha}|} \leq \frac{1}{(1-\theta R_S(h))^{\alpha}}. 
\label{eqn:K^alpha_final}
\ee 
\begin{itemize}
\item An upper bound of $\big|\frac{1}{J^{\alpha}} - \frac{1}{K^{\alpha}}\big|$\\ 
We compute  $\big|\frac{1}{J^{\alpha}} - \frac{1}{K^{\alpha}}\big|$;
\beqa  
\Big|\frac{1}{J^{\alpha}} - \frac{1}{K^{\alpha}}\Big| 
= \frac{|K^{\alpha} - J^{\alpha}|}{J^{\alpha} K^{\alpha}}. 
\eeqa
By Mean Value Theorem, 
there exists $x_0  \in (m_0, M_0)$ such that 
$
K^{\alpha} - J^{\alpha}= (K-J) \cdot \alpha x_0^{\alpha-1}  
$
where   $m_0= \min(K, J)$, and $M_0=\max(K, J)$. 
Since $0 < \alpha < 1$, the function $f(x) = x^{\alpha-1}$ is 
a decreasing function of $x$. 
Hence $x_0^{\alpha-1} < m_0^{\alpha -1}$. 
Recall that $J=1+\theta(\mFP-\mFN)$ and $K= 1+\theta(\qFP-\qFN)$.  
Therefore $|K-J|=\theta|\mFP-\mFN-(\qFP-\qFN)|=\theta|T_m-S_m|$. 
Summarizing these, we have 
\beqa  
\Big|\frac{1}{J^{\alpha}} - \frac{1}{K^{\alpha}}\Big|  
 =   \frac{|K-J| \cdot \alpha x_0^{\alpha-1}}{J^{\alpha}  K^{\alpha}}  
 \leq   \frac{\theta |T_m-S_m| \cdot \alpha m_0^{\alpha-1}}{m_0^{\alpha}  M_0^{\alpha}}  &=& \frac{\alpha \theta |T_m-S_m|}{m_0 M_0^{\alpha}}. 
\eeqa

Consider the case that  $J \geq  K$.
Note that $K=1+\theta(\qFP-\qFN) > 1-\theta R_S(h)$ 
and $J = K+|K-J|=K+ \theta|T_m-S_m| > 1+\theta(|T_m-S_m|-R_S(h))$.
Hence $m_0 M_0^{\alpha} = K J^{\alpha} > (1-\theta R_S(h))(1-\theta R_S(h)+ \theta|T_m-S_m|)^{\alpha}$.   
  
Consider the case that $ J< K$. 
In this case, we have  $J = K-|K-J|=K - \theta|T_m-S_m| > 1-\theta R_S(h) -\theta |T_m-S_m|$ 
and  $m_0 M_0^{\alpha} = J K^{\alpha} > (1-\theta R_S(h)- \theta|T_m-S_m|)(1-\theta R_S(h))^{\alpha}$.

In either case, we have $m_0 M_0^{\alpha} > (1-\theta R_S(h)- \theta|T_m-S_m|)(1-\theta R_S(h))^{\alpha}$.
Therefore, $\Big|\frac{1}{J^{\alpha}} - \frac{1}{K^{\alpha}}\Big|$ becomes;
\be 
\Big|\frac{1}{J^{\alpha}} - \frac{1}{K^{\alpha}}\Big|  
   < \frac{\alpha \theta |T_m-S_m|}{(1-\theta R_S(h) - \theta|T_m-S_m|) (1-\theta R_s(h))^{\alpha}}. \label{eqn:T1_final}
\ee 

\item  An upper bound of $|Z_1|$\\ 
Recall that $|Z_1| = |1-(\mFP + \mFN)|$. 
\be 
|Z_1| &=& |1-(\mFP+\mFN)| \nonumber\\
      &\leq & |1-(\qFP+\qFN)| + |\qFP+\qFN-(\mFP+\mFN)| \nonumber \\
      &= & 1-R_S(h) + |T_p-S_p| \label{eqn:Z1_final}       
\ee  

\item  An upper bound of $|Z_2|$\\ 
Recall that $Z_2= (1+\theta)^{\alpha}\mFP + (1-\theta)^{\alpha} \mFN$. 
Using (\ref{eqn:m}), we rewrite $Z_2$ as follows,
\beqa 
Z_2 = \frac{(1+\theta)^{\alpha}+(1-\theta)^{\alpha}}{2} T_p + \frac{(1+\theta)^{\alpha}-(1-\theta)^{\alpha}}{2} T_m. 
\eeqa 
Hence
\be  
|Z_2| &\leq & \Big|\frac{(1+\theta)^{\alpha}+(1-\theta)^{\alpha}}{2}|\cdot |T_p| + \Big|\frac{(1+\theta)^{\alpha}-(1-\theta)^{\alpha}}{2}\Big| \cdot |T_m| \nonumber\\
&\leq & \frac{(1+\theta)^{\alpha}+(1-\theta)^{\alpha}}{2} \cdot R_{\Xcal}(h) +  \frac{(1+\theta)^{\alpha} -(1-\theta)^{\alpha}}{2}  \cdot R_{\Xcal}(h) \nonumber\\
&= &  (1+\theta)^{\alpha}  R_{\Xcal}(h). 
   \qquad (\mbox{because } , ~T_P=R_{\Xcal}(h), ~|T_m| \leq  R_{\Xcal}(h)). 
\label{eqn:Z2_final} 
\ee 

\item An upper bound of $\Big|\frac{Z_3+Z_4}{\alpha(\alpha-1)}\Big|$\\ 
Recall that 
\beqa 
Z_3 &=&  \qFP+\qFN -(\mFP+\mFN), \\ 
Z_4 &=& (1+\theta)^{\alpha}(\mFP-\qFP)+(1-\theta)^{\alpha}(\mFN-\qFN).
\eeqa 
Using (\ref{eqn:m}) and (\ref{eqn:q}), we rewrite $Z_3$ and  $Z_4$ 
as below;
\beqa
Z_3 &=&  S_p-T_p, \\
Z_4 &=& \frac{(1+\theta)^{\alpha}+(1-\theta)^{\alpha}}{2} (T_p -S_p)
       +\frac{(1+\theta)^{\alpha}-(1-\theta)^{\alpha}}{2} (T_m - S_m)       
\eeqa
We find an upper bound of  $\Big|\frac{Z_3+Z_4}{\alpha(\alpha-1)}\Big|$; 
\beqa
\Big|\frac{Z_3+Z_4}{\alpha(\alpha-1)}\Big|
& =&
\Big| \frac{(1+\theta)^{\alpha}+(1-\theta)^{\alpha}-2}{2 \alpha(\alpha-1)} (T_p -S_p)
       +\frac{(1+\theta)^{\alpha}-(1-\theta)^{\alpha}}{2\alpha(\alpha-1)} (T_m - S_m)\Big|       
\eeqa
Note that $f(x) = \frac{x^{\alpha}}{\alpha(\alpha-1)}$  is a convex function for $ \alpha \neq 0,~1$.   
If we apply Lemma \ref{lemma:convex-prop} for 
$f(x) = \frac{x^{\alpha}}{\alpha(\alpha-1)}$ with 
the values of  $s = 1-\theta, ~t=1$, and  $u=1+\theta$, 
it holds that
\be 
\frac{1-(1-\theta)^{\alpha}}{\alpha (\alpha-1) \theta}
& \leq & \frac{(1+\theta)^{\alpha} - (1-\theta)^{\alpha}}{\alpha (\alpha-1)\theta}
~\leq~ \frac{(1+\theta)^{\alpha}-1}{\alpha (\alpha-1) \theta}. 
\label{eqn:convex-prop2}
\ee  
By (\ref{eqn:convex-prop2}) and $\theta = \frac{1}{r} >0$, 
it holds that 
\beqa
\frac{(1+\theta)^{\alpha} + (1-\theta)^{\alpha} -2}{\alpha(\alpha-1) } \geq 0.
\eeqa
Using the above property, we have 
\be 
\Big|\frac{Z_3+Z_4}{\alpha(\alpha-1)}\Big|
& \leq &
 \Big|\frac{(1+\theta)^{\alpha}+(1-\theta)^{\alpha}-2}{2 \alpha(\alpha-1)}\Big| \cdot |T_p -S_p| 
       +\Big|\frac{(1+\theta)^{\alpha}-(1-\theta)^{\alpha}}{2\alpha(\alpha-1)}\Big| \cdot |T_m - S_m|\Big| \nonumber\\ 
& \leq  & \frac{2-(1+\theta)^{\alpha}-(1-\theta)^{\alpha}}{2 \alpha(1-\alpha)}\cdot |T_p -S_p| 
   + \frac{(1+\theta)^{\alpha}-(1-\theta)^{\alpha}}{2\alpha(1-\alpha)} \cdot |T_m - S_m|. \label{eqn:Z3Z4_final}
\ee 
The second inequality holds since $0< \alpha <1$ (hence $ \alpha -1<0$). 
\end{itemize}

By (\ref{eqn:K^alpha_final}), (\ref{eqn:T1_final}),(\ref{eqn:Z1_final}), (\ref{eqn:Z2_final}),  (\ref{eqn:Z3Z4_final}), 
and applying Lemma \ref{lemma:together}, with probability at least $1-\delta$, 
we have 
\beqa
\lefteqn{|I_{\alpha}(h, P)- I_{\alpha}(\ub:n)| } \\
&\leq & \frac{\alpha \theta \varepsilon_2}{(1-\theta R_S(h) -\theta\varepsilon_2)(1-\theta R_S(h))^{\alpha}} 
           \cdot \frac{1 -R_S(h)+\varepsilon_2+(1+\theta)^{\alpha} R_{\Xcal}(h)}{\alpha(1-\alpha)}\\
& &      ~~+ \frac{1}{(1-\theta R_S(h))^{\alpha}}
      \Big[\frac{2-(1+\theta)^{\alpha}-(1-\theta)^{\alpha}}{2\alpha(1-\alpha)} \varepsilon_2  
           + \frac{(1+\theta)^{\alpha}-(1-\theta)^{\alpha}}{2\alpha(1-\alpha)} \varepsilon_2\Big]
\\
&\leq & \frac{\varepsilon_2}{\alpha(1-\alpha)(1-\theta R_S(h))^{\alpha}}  
\Big[\frac{\alpha \theta \Big(1+ \big(  (1+\theta)^{\alpha}-1 \big)R_S(h) 
+ \varepsilon_2\big((1+\theta)^{\alpha}+1\big)\Big)}{1-\theta R_S(h) -\theta \varepsilon_ 2} +1-(1-\theta)^{\alpha}\Big]\\
&=& \frac{\varepsilon_2}{\alpha(1-\alpha)}\Big( \frac{r}{r-R_S(h)}\Big)^{\alpha}
    \Big[1-\Big(1-\frac{1}{r}\Big)^{\alpha}+ \frac{\alpha C_{\tilde{\psi}}}{r-R_S(h)-\varepsilon_2}\Big]
\eeqa
where 
\beqa
C_{\tilde{\psi}} =  1+ \big((1+\frac{1}{r})^{\alpha} -1\big)R_S(h) 
         + \varepsilon_2\big((1+\frac{1}{r})^{\alpha}+1\big).
\eeqa
The second inequality holds since $R_{\Xcal}(h) \leq R_S(h) +\varepsilon_2$ by Theorem \ref{thm:gen-error} 
and the first equality is obtained by $\theta = \frac{1}{r}$. 

By Mean Value Theorem and $f(x) = x^{\alpha-1}$ is decreasing for $0 <\alpha<1$ and $x>0$,
we have $z_0 \in (1-\frac{1}{r}, 1)$ such that  
\beqa 
   1-\Big(1-\frac{1}{r}\Big)^{\alpha} ~=~  \frac{\alpha z_0^{\alpha}}{r} \leq \frac{\alpha}{r}.
\eeqa
Using the above inequality, we have found an upper bound of $|I_{\alpha}(h, P)- I_{\alpha}(\ub:n)|$;
\beqa
 |I_{\alpha}(h, P)- I_{\alpha}(\ub:n)|  
&\leq &\frac{\varepsilon_2}{1-\alpha}\Big( \frac{r}{r-R_S(h)}\Big)^{\alpha}\Big( \frac{1}{r} +\frac{C_{\tilde{\psi}}}{r-R_S(h)-\varepsilon_2}\Big)\\
&=&\frac{1}{1-\alpha}\Big( \frac{r}{r-R_S(h)}\Big)^{\alpha}\Big( \frac{1}{r} +\frac{C_{\tilde{\psi}}}{r-R_S(h)-\varepsilon_2}\Big)\sqrt{\frac{8d_{\Hcal}\ln\big(\frac{2en}{d_{\Hcal}}+ 8\ln\frac{8}{\delta}\big)}{n}}.
\eeqa

\subsection{When $    \alpha > 1$}
The analysis is  almost  same as in the case of $0 <  \alpha < 1$ 
but needs some modification because   the value of $\alpha$ is bigger than one. 
The upper bounds of 
$\Big|\frac{1}{K^{\alpha}} -\frac{1}{J^{\alpha}}  \Big|$ 
and $\Big|\frac{Z_3+Z_4}{\alpha(\alpha-1)}\Big|$ have different values 
when  $\alpha >1$. 
For the case $\alpha >1 1$, 
we can show that (\ref{eqn:Z2_final}) and (\ref{eqn:Z3Z4_final}) are replaced by the following 
inequalities, respectively,   
\be
\Big|\frac{1}{K^{\alpha}} -\frac{1}{J^{\alpha}}  \Big|
& \leq & \frac{\alpha \theta |T_m-S_m|}{(1-\theta R_S(h))\cdot(1-\theta R_S(h)-\theta |T_m-S_m|)^{\alpha}}   \label{eqn:T1_alpha1}
\ee
and 
\be
 \Big|\frac{Z_3+Z_4}{\alpha(\alpha-1)}\Big|
 & \leq  & \frac{(1+\theta)^{\alpha}+(1-\theta)^{\alpha}-2}{2\alpha(\alpha-1)}\cdot |T_p -S_p| 
   + \frac{(1+\theta)^{\alpha}-(1-\theta)^{\alpha}}{2\alpha(\alpha-1)} \cdot |T_m - S_m|.  \label{eqn:Z3Z4_alpha1}
\ee
By (\ref{eqn:K^alpha_final}), (\ref{eqn:T1_alpha1}), (\ref{eqn:Z1_final}), (\ref{eqn:Z2_final}), and (\ref{eqn:Z3Z4_alpha1}), 
we have 
\beqa
\lefteqn{ |I_{\alpha}(h, P)- I_{\alpha}(\ub:n)| } \\
&\leq & \frac{\alpha \theta\varepsilon_2}{(1-\theta R_S(h)) (1-\theta R_S(h)-\theta \varepsilon_2)^{\alpha}}
        \cdot \frac{1-R_S(h)+\varepsilon_2 + (1+\theta)^{\alpha}R_{\Xcal}(h)}{\alpha (\alpha-1)} \\
       & & \quad +~ \frac{1}{(1-\theta R_S(h))^{\alpha}} 
           \frac{\big((1+\theta)^{\alpha} + (1-\theta)^{\alpha} -2 \big) \varepsilon_2}{2 \alpha (\alpha-1)}\\
&\leq & \frac{\alpha \theta\varepsilon_2}{(1-\theta R_S(h)) (1-\theta R_S(h)-\theta \varepsilon_2)^{\alpha}}
        \cdot \frac{1-R_S(h)+\varepsilon_2 + (1+\theta)^{\alpha}(R_S(h)+\varepsilon_2)}{\alpha (\alpha-1)} \\
       & & \quad +~ \frac{1}{(1-\theta R_S(h))^{\alpha}} 
           \frac{\big((1+\theta)^{\alpha} + (1-\theta)^{\alpha} -2 \big) \varepsilon_2}{2 \alpha (\alpha-1)}\\
&\leq & \varepsilon_2 \cdot
\frac{(1+\theta)^{\alpha}-1 +  
\frac{\alpha \theta}{1-\theta R_S(h)}
\big[\frac{1-\theta R_S(h)}{1-\theta (R_S(h)+\varepsilon_2)}\big]^{\alpha}
\big[1+ \big(  (1+\theta)^{\alpha}-1 \big)R_S(h) 
+ \varepsilon_2\big((1+\theta)^{\alpha}+1\big)\big]
 }{\alpha(\alpha-1)(1-\theta R_S(h))^{\alpha}}\\
\eeqa
We rewrite $|I_{\alpha}(h, P)- I_{\alpha}(\ub:n)|$ in terms of $r =\frac{1}{\theta}$ instead of  $\theta$. 
Recalling  that $C_{\tilde \psi} = 1+ \big(  (1+\frac{1}{r})^{\alpha}-1 \big)R_S(h) 
+ \varepsilon_2\big((1+\frac{1}{r})^{\alpha}+1\big)$, we have 
\beqa
|I_{\alpha}(h, P)- I_{\alpha}(\ub:n)|  
&\leq &  \frac{\varepsilon_2 }{\alpha(\alpha-1)(1- \frac{R_S(h)}{r})^{\alpha}} 
\Big[ \Big(1+\frac{1}{r}\Big)^{\alpha}-1 +  
\frac{\alpha \big(1+\frac{\varepsilon_2}{r-R_S(h)-\varepsilon_2}\big)^{\alpha} C_{\tilde \psi}}{r-R_S(h)}\Big].
\eeqa
By Mean Value Theorem and the fact that $f(x) = x^{\alpha-1}$ is increasing with $x>0$ for $\alpha >1$, 
we have 
\beqa
\Big(1+\frac{1}{r}\Big)^{\alpha}-1 &\leq & \frac{\alpha}{r} \Big(1+\frac{1}{r}\Big)^{\alpha-1}. 
\eeqa
Hence  
\beqa
|I_{\alpha}(h, P)- I_{\alpha}(\ub:n)|
&\leq & \frac{\varepsilon_2}{\alpha(\alpha-1)} \Big(\frac{r }{r-R_S(h)}\Big)^{\alpha}
        \Big[\frac{\alpha}{r}\Big(1+\frac{1}{r}\Big)^{\alpha} 
        + \frac{\alpha \big(1+\frac{\varepsilon_2}{r-R_S(h)-\varepsilon_2}\big)^{\alpha} C_{\tilde \psi}}{r-R_S(h)}\Big]\\
&=& \frac{\varepsilon_2}{(\alpha-1)} \Big(\frac{r }{r-R_S(h)}\Big)^{\alpha}
        \Big[\frac{1}{r}\Big(1+\frac{1}{r}\Big)^{\alpha} 
        + \frac{\big(1+\frac{\varepsilon_2}{r-R_S(h)-\varepsilon_2}\big)^{\alpha} C_{\tilde \psi}}{r-R_S(h)}\Big].
\eeqa
Recalling the definition of $\varepsilon_2$, we have proved the case of $\alpha>1$ of Theorem \ref{thm:gen-fair-constraint-tight}.

\section{Extension of Generalized Entropy} \label{append:ext_gen_entropy} 

Axioms 4 and 6 should  be modified to reflect the change. 
\begin{itemize}  
\item   Axiom 4$'$: 
For any partition $\{\Xcal^g\}_{g=1}^G$ of $\Xcal$ where 
$\Xcal^g$ has the  probability distribution $P_x^g$ 
such that $P_x^g(A) = \frac{P_x(A)}{P_x(\Xcal^g)}$ for any $A \subseteq \Xcal^g$, 
it holds that 
\beqa
I_\alpha(b, \Xcal,P_x) 
&=& \sum_{g=1}^{G} w_{\alpha, g}^G ~I_{\alpha}(b|_{\Xcal^g}, \Xcal^gg, P_x^g) + I_{\alpha}(\upsilon, \Xcal, P_x)\\
&=& \sum_{g=1}^G m_g \Big(\frac{\mu_g}{\mu}\Big)^{\alpha}~I_{\alpha}(b|_{\Xcal^g}, \Xcal^g, P_x^g) 
                           +  \sum_{g=1}^G m_g f_{\alpha}\Big(\frac{\mu_g}{\mu}\Big) 
\eeqa 
where   
$b|_{\Xcal^g}$ is the restricted function of $b(\ux)$ on $\Xcal^g$, 
$w_{\alpha, g}^G =m_g  \big(\frac{\mu_g}{\mu}\big)^{\alpha}$, 
$m_g=  P_x(\Xcal^g), ~\mu = \Esf[b(\ux)], ~\mu_g = \int_{\Xcal^g} b(\ux) dP_x^g$,  
and $\upsilon(\ux) = \sum_{g=1}^G \mu_g \mathds{1} \{\ux \in \Xcal^g \}$.  
\item   Axiom 6$'$ (Pigou-Dalton transfer principle): 
Consider $M_i, M_j \subset \Xcal$ such that  $M_i \cap M_j = \emptyset$ and 
$b(\bx_i) > b(\bx_j)$ for 
any $\ux_i \in M_i$   and $\ux_j \in M_j$. Define $\tilde{b}(\bx)$ as 
\beqa
  \tilde{b}(\bx) = \left\{ \begin{array}{ll}
                        b(\bx) - \delta & \mbox{if } \bx \in M_i, \\ 
                        b(\bx) + \delta & \mbox{if } \bx \in M_j, \\
                        b(\bx)          & \mbox{otherwise}.
                   \end{array} \right.                  
\eeqa
It holds that $I_{\alpha}(b,\Xcal, P_x) > I_{\alpha}(\tilde{b},\Xcal, P_x)$, 
if $\Esf[b(\ux)] =  \Esf[\tilde{b}(\ux)]$ and 
 $b(\ux_i) -\delta > b(\ux_j) + \delta$   for any $\ux_i \in M_i, ~\ux_j \in M_j$.
\end{itemize} 

Now, we prove Theorem \ref{thm:prop_ext_gen_entropy}. 
  It is easy to see that Axioms 1,3, 5, 7 hold after the extension. 
To prove   the remaining  Axioms 2, 4$'$, and 6$'$,   
we need Jensen's inequality stated in below.
\begin{theorem}[Jensen's Inequality, \cite{Rudin87}] 
Let $Q$ be a probability distribution on $\Xcal$ and dQ the probability density function of $Q$. 
If $b$ is a real function with $\int_{\Xcal} |b(\ux)| dQ$, 
$ a_1< b(\bx)< a_2$ for all $\bx \in \Xcal$, and $\phi$ is convex on $(a_1, a_2)$, then 
\beqa
 \phi\Big(\int_{\Xcal} b(\ux) dQ \Big) \leq \int_{\Xcal}  \phi \big(b(\ux)\big) dQ.
\eeqa 
\end{theorem}

\subsection{Proof for Axiom 2}
Note that $f_{\alpha}(x)$ is convex over $\real$ for all $\alpha \in [0, \infty)$. 
We will show that $I_{1}(b, \Xcal, P_x)\geq 0$ for $\alpha =1$. 
\beqa 
I_{1}(b, \Xcal, P_x)
                  =   \frac{1}{\mu} \Big(\int_{\Xcal} b(\ux) \ln b(\ux) dP_x -    \mu \ln \mu \Big) 
                  \geq   \frac{1}{\mu} \Big[f_{1}\Big( \int_{\Xcal} b(\ux) \ln b(\ux) dP_x \Big)  - \mu \ln \mu\Big] 
                  \geq   0. 
\eeqa 
The first inequality holds by Jensen's inequality since $b(\ux)$ is bounded and $f_1(x)$ is convex. 
In a similar way,   we can show that $I_{\alpha}(b, \Xcal, P_x) \geq 0$ for $\alpha \neq 1$. 

\subsection{Proof for Axiom 4$'$ (Additive decomposability)} 
Recall that $\upsilon(\ux) = \sum_{g=1}^G \mu_g \mathds{1}\{\ux \in \Xcal^g\}$ and    $m_g = P_x(\Xcal^g)$. 
For any $A \subset \Xcal^g$, we have $ P_x^g(A) =  \frac{P_x(A)}{P_x(\Xcal^g)} = \frac{P_x(A)}{m_g}$. 
Hence $P_x^g = \frac{P_x}{m_g} $ and $dP_x = m_g dP_x^g$ for all $g$. 
\begin{lemma} \label{lemma:Eupsilon(x)}
\beqa
 \Esf[\upsilon(\ux)] =  \mu =  \sum_{g=1}^G m_g \mu_g.
\eeqa
\end{lemma}
\begin{proof}
Since $\Xcal =  \dot \bigcup_{g=1}^G \Xcal^g$ where $\dot\bigcup$ means a disjoint union, we have 
\be 
   \mu = \int_{\Xcal}b(\ux) dP_x = \int_{ \dot\bigcup \Xcal^g} b(\ux)dP_x = \sum_{g=1}^G \int_{\Xcal^g} b(\ux) dP_x,\label{eqn:A4_lemma_tmp1}  \\ 
   \sum_{g=1}^G \int_{\Xcal^g} b(\ux) dP_x 
   = \sum_{g=1}^G m_g \int_{\Xcal^g}  {b}(\ux) dP_x^g = \sum_{g=1}^G m_g \mu_g. \label{eqn:A4_lemma_tmp2}  
\ee  
Combining (\ref{eqn:A4_lemma_tmp1}) and (\ref{eqn:A4_lemma_tmp2}), we have $\mu =  \sum_{g=1}^G m_g \mu_g$.
Consider $\Esf[\upsilon(\ux)]$; 
\beqa
\Esf[\upsilon(\ux)] = \Esf\big[\sum_{g=1}^G \mu_g \mathds{1}\{\ux \in \Xcal^g\}\big]
=\sum_{g=1}^G \mu_g \Esf\big[\mathds{1}\{\ux \in \Xcal^g\}\big] 
= \sum_{g=1}^G m_g \mu_g = \mu.
\eeqa 
\end{proof}
 
\begin{lemma} \label{lemma:V_in_gen_entropy} 
For $\alpha \in [0, \infty)$, it holds that 
$ I_{\alpha}(b, \Xcal, P_x) =\sum_{g=1}^G  m_g f_{\alpha}\Big(\frac{\mu_g}{\mu}\Big).$
\end{lemma}
\begin{proof}
Since $\Xcal = \dot\bigcup \Xcal^g$, for each $\ux \in \Xcal$, there is a unique $g$ such that $\ux \in \Xcal^g$. 
Therefore 
\be
\upsilon(\ux) = \sum_{i=1}^{G} \mu_i \mathds{1} \{\ux \in \Xcal^i \} = \mu_g \qquad \mbox{for } \ux \in \Xcal^g. \label{eqn:upsilon(x)}
\ee 
Consider $I_{\alpha}(\upsilon, \Xcal, P_x)$;  
\beqa
I_{\alpha}(\upsilon, \Xcal, P_x) 
 =  \int_{\Xcal} f_{\alpha}\Big(\frac{\upsilon(\ux)}{\Esf[\upsilon(\ux)]}\Big)dP_x   
 =  \int_{\Xcal} f_{\alpha}\Big(\frac{\upsilon(\ux)}{\mu}\Big)dP_x  
 =   \sum_{g=1}^G \int_{\Xcal^g} f_{\alpha}\Big(\frac{\mu_g}{\mu}\Big)dP_x  
 =   \sum_{g=1}^G m_g f_{\alpha}\Big(\frac{\mu_g}{\mu}\Big).
\eeqa 
The second equality holds by   Lemma \ref{lemma:Eupsilon(x)}  and the third equality holds by   (\ref{eqn:upsilon(x)})
\end{proof}

\begin{itemize}
\item[i)]  $\alpha=1$: 
Using $\Xcal = \dot\bigcup \Xcal^g$, we have 
\be 
I_{\alpha}(b, \Xcal, P_x)  =    \int_{\dot\bigcup \Xcal^g} \frac{b(\ux)}{\mu} \ln \frac{b(\ux)}{\mu} dP_x  
                           =  \sum_{g=1}^G \int_{  \Xcal^g} \frac{b(\ux)}{\mu} \ln \frac{b(\ux)}{\mu} dP_x  
                           =  \sum_{g=1}^G m_g \int_{  \Xcal^g}  \frac{{b}(\ux)}{\mu} \ln \frac{{b}(\ux)}{\mu} dP_x^g.   
                             \label{eqn:add_dec_alpha_1}                       
\ee 

Since $\frac{{b}(\ux)}{\mu} \ln \frac{{b}(\ux)}{\mu}  
= \frac{\mu_g}{\mu}  \frac{{b}(\ux)}{\mu_g} \Big( \ln \frac{{b}(\ux)}{\mu_g} + \ln \frac{\mu_g}{\mu}\Big)$, 
we have
\be
\int_{\Xcal^g}  \frac{{b}(\ux)}{\mu} \ln \frac{{b}(\ux)}{\mu} dP_x^g
 &=& \frac{\mu_g}{\mu} \Big(\int_{\Xcal^g} \frac{{b}(\ux)}{\mu_g}    \ln \frac{{b}(\ux)}{\mu_g} dP_x^g  
       +   \ln \frac{\mu_g}{\mu} \int_{\Xcal^g} \frac{{b}(\ux)}{\mu_g}dP_x^g \Big)\nonumber\\
&=& \frac{\mu_g}{\mu} \Big( I_{\alpha}(b|_{\Xcal^g}, \Xcal_g, P_x^g) 
    +  \ln \frac{\mu_g}{\mu} \Big) \quad \big(\mbox{because }   \mu_g = \int_{\Xcal_g} {b}(\ux) dP_x^g \big)\qquad ~~~ \label{eqn:sub_alpha_1}
\ee 

By (\ref{eqn:sub_alpha_1}), equation (\ref{eqn:add_dec_alpha_1}) becomes 
\beqa
I_{\alpha}(b, \Xcal, P_x) 
=   \sum_{g=1}^G m_g  \frac{\mu_g}{\mu}\Big( I_{\alpha}(b|_{\Xcal^g}, \Xcal_g, P_x^g) 
      +  \ln \frac{\mu_g}{\mu}  \Big)
 =  \sum_{g=1}^G m_g  \frac{\mu_g}{\mu}I_{\alpha}(b|_{\Xcal^g}, \Xcal_g, P_x^g) 
   + \sum_{g=1}^G m_g  \frac{\mu_g}{\mu}\ln \frac{\mu_g}{\mu}. 
\eeqa 
By Lemma \ref{lemma:V_in_gen_entropy}, we have 
$I_{1}(\upsilon, \Xcal, P_x) = \sum_{g=1}^G m_g f_{1}\Big(\frac{\mu_g}{\mu}\Big)    
                             = \sum_{g=1}^G m_g \frac{\mu_g}{\mu}\ln \frac{\mu_g}{\mu}.$  
We have proved that the extended generalized entropy has the additive decomposable property when  $\alpha=1$. 

\item[ii)]   $\alpha=0$: This case can be shown similarly as in the case $\alpha =1$.  
 
\item[iii)]   $\alpha \neq 0,1 $ (and $0 < \alpha < \infty$):  
\be 
 I_{\alpha}(b, \Xcal, P_x) 
 =  \frac{1}{\alpha (\alpha-1)}  \sum_{g=1}^G \int_{\Xcal^g} \Big[\Big(\frac{b(\ux)}{\mu}\Big)^{\alpha} -1\Big]dP_x 
 =  \frac{1}{\alpha (\alpha-1)}\sum_{g=1}^G m_g\int_{\Xcal^g}  \Big[\Big(\frac{b(\ux)}{\mu}\Big)^{\alpha} -1\Big]  dP_x^g.  
   \label{eqn:add_decomp_alpha}
\ee 
Since $\Big(\frac{b(\ux)}{\mu}\Big)^{\alpha} -1  = \Big(\frac{\mu_g}{\mu}\Big)^{\alpha}\Big[\Big(\frac{b(\ux)}{\mu_g}\Big)^{\alpha} -1\Big] 
    +   \Big(\frac{\mu_g}{\mu}\Big)^{\alpha}-1, $ 
it holds that  
\beqa
\int_{\Xcal^g} \Big[\Big(\frac{b(\ux)}{\mu_g}\Big)^{\alpha} -1\Big] dP_x^g
&=&  \int_{\Xcal^g} \Big(\frac{\mu_g}{\mu}\Big)^{\alpha}\Big[\Big(\frac{b(\ux)}{\mu}\Big)^{\alpha} -1\Big]dP_x^g 
    +  \int_{\Xcal^g} \Big[ \Big(\frac{\mu_g}{\mu}\Big)^{\alpha}-1\Big]dP_x^g\\
&=&  \Big(\frac{\mu_g}{\mu}\Big)^{\alpha} \alpha (\alpha-1) I_{\alpha}(b|_{\Xcal^g}, \Xcal^g, P_x^g) + \Big(\frac{\mu_g}{\mu}\Big)^{\alpha}-1  
  \qquad (\mbox{because } \mu_g = \int_{\Xcal^g} b(\ux) dP_x^g). 
\eeqa
By (\ref{eqn:add_decomp_alpha}), we have  
\beqa
I_{\alpha}(b, \Xcal, P_x)
&=&    \sum_{g=1}^G m_g\Big(\frac{\mu_g}{\mu}\Big)^{\alpha} I_{\alpha}(b|_{\Xcal^g}, \Xcal^g, P_x^g)
    ~+ ~\frac{1}{\alpha(\alpha-1)} \sum_{g=1}^G m_g \Big[\Big(\frac{\mu_g}{\mu}\Big)^{\alpha}-1\Big] \\
&=&  \sum_{g=1}^G m_g\Big(\frac{\mu_g}{\mu}\Big)^{\alpha} I_{\alpha}(b|_{\Xcal^g}, \Xcal^g, P_x^g)
    ~+ \sum_{g=1}^G m_g f_{\alpha}\Big(\frac{\mu_g}{\mu}\Big) \\
&=&  \sum_{g=1}^G m_g\Big(\frac{\mu_g}{\mu}\Big)^{\alpha} I_{\alpha}(b|_{\Xcal^g}, \Xcal^g, P_x^g)
    ~+ I_{\alpha}(b, \Xcal, P_x)  \qquad(\mbox{By Lemma \ref{lemma:V_in_gen_entropy}}).
\eeqa
\end{itemize}

\subsection{Proof for Axiom 6$'$} 
It can be easily checked that
$P_x(M_i) = P_x(M_j)$ from the assumption  
that $ \Esf[b(\ux)] = \Esf[\tilde{b}(\ux)] =\mu$.
We need a simple property of a real-valued  differentiable convex function, which is stated in the Lemma below. 
\begin{lemma}\cite{Boyd04} \label{lemma:convex}
Suppose that $f:\mathbb{R}   \rightarrow \mathbb{R}$ is a differentiable real-valued function. 
Then $f$ is convex if and only if   
\beqa
 f(\ux_2) \geq f(\ux_1) +   f'(\ux_1)^T (\ux_2-\ux_1)
\eeqa 
\end{lemma}
Note that $f_{\alpha}(x)$  in (\ref{eqn:f_alpha}) is a real-valued  convex function defined on $\mathbb{R}$ for all  $\alpha \in [0, \infty)$. 
Therefore Lemma \ref{lemma:convex} holds and moreover $f'_{\alpha}(x_2)> f'_{\alpha}(x_1)$ if $x_2 > x_1$.

We will show that $I_{\alpha}(b,\Xcal, P_x) - I_{\alpha}(\tilde{b}, \Xcal, P_x) \geq 0$.
\beqa
 \lefteqn{ I_{\alpha}(b,\Xcal, P_x) - I_{\alpha}(\tilde{b},\Xcal, P_x)} \\  
 & = & \int_{M_i \cup M_j}  f_{\alpha}\Big(\frac{b(\bx)}{\mu}\Big) dP_x - \int_{M_i \cup M_j}  f_{\alpha}\Big(\frac{\tilde{b}(\bx)}{\mu}\Big) dP_x \\
 & = & \int_{M_i} f_{\alpha}\Big(\frac{b(\bx)}{\mu}\Big) dP_x +  \int_{M_j}  f_{\alpha}\Big(\frac{b(\bx)}{\mu}\Big) dP_x -
   \Big[ \int_{M_i} f_{\alpha}\Big(\frac{ b(\bx)-\delta}{\mu}\Big) dP_x +  \int_{M_j}  f_{\alpha}\Big(\frac{b(\bx)+\delta}{\mu}\Big) dP_x\Big] \\
 & > &  \int_{M_i}  f'_{\alpha}\Big(\frac{ b(\bx)-\delta}{\mu}\Big) \delta dP_x  
     - \int_{M_j} f'_{\alpha}\Big(\frac{b(\bx)}{\mu}\Big) \delta dP_x   \qquad(\mbox{by   Lemma \ref{lemma:convex}})\\
 & = & \delta P_x(M_i) \Big[ f'_{\alpha}\Big(\frac{ b(\bx)-\delta}{\mu}\Big)- f'_{\alpha}\Big(\frac{b(\bx)}{\mu}\Big) \Big]\\
 & \geq & 0 \qquad \qquad (\mbox{because $P_x(M_i) = P_x(M_j)$ and     the convexity of $f_{\alpha}$}).
\eeqa

\end{document}